%% file: main.tex
\documentclass{article}

\PassOptionsToPackage{numbers, compress}{natbib}

\usepackage[final]{neurips_2022}

\usepackage[utf8]{inputenc} 
\usepackage[T1]{fontenc}    
\usepackage{hyperref}       
\usepackage{url}            
\usepackage{booktabs}       
\usepackage{amsfonts}       
\usepackage{nicefrac}       
\usepackage{microtype}      
\usepackage{xcolor}         
\usepackage{wrapfig}

\usepackage{graphicx}
\usepackage{subfigure}

\usepackage{amsmath}
\usepackage{amssymb}
\usepackage{mathtools}
\usepackage{amsthm}
\usepackage{bbm}

\usepackage[capitalize,noabbrev]{cleveref}

\theoremstyle{plain}
\newtheorem{theorem}{Theorem}[section]
\newtheorem{proposition}[theorem]{Proposition}
\newtheorem{lemma}[theorem]{Lemma}
\newtheorem{corollary}[theorem]{Corollary}
\theoremstyle{definition}
\newtheorem{definition}[theorem]{Definition}
\newtheorem{assumption}[theorem]{Assumption}
\theoremstyle{remark}

\usepackage[textsize=tiny]{todonotes}
\setuptodonotes{color=green!30}

\usepackage[font=small,labelfont=bf]{caption}
\usepackage{wrapfig}
\usepackage{float}
\floatstyle{plaintop}
\newfloat{example}{thp}{lop}
\floatname{example}{Example}

\usepackage{enumitem}
\setitemize{leftmargin=1em,itemsep=1ex, topsep=0pt, parsep=0em,label=\raisebox{0.25ex}{\tiny$\bullet$}}
\setlist[enumerate]{leftmargin=*, label= {\arabic*.}, itemsep=0em, topsep=0pt}

\input{macros.tex}

\newcommand{\DD}{\mathds{D}}
\newcommand{\OO}{\mathds{O}}
\newcommand{\GG}{\mathds{G}}

\def\DP{{\textsf{DP}}\xspace}

\def\EO{{\textsf{EO}}\xspace}
\def\EOp{{\textsf{EOp}}\xspace}

\def\D{{\mathcal D}}
\def\S{{\mathcal S}}
\def\R{{\mathcal T}}
\def\H{{\mathcal H}}

\def\TP{{1,1}}
\def\FP{{0,1}}
\def\TN{{0,0}}
\def\FN{{1,0}}

\def\var{{\mathrm{Var}}}
\def\cov{{\mathrm{Cov}}}

\def\disp{{\Delta^{\star}}}
\def\dispdp{{\Delta^{\star}_{\DP}}}

\renewcommand{\cov}{\operatorname*{Cov}}
\renewcommand{\var}{\operatorname*{Var}}

\newcommand{\rr}[1]{#1}
\newcommand{\rev}[1]{#1}

\title{Fairness Transferability \\ Subject to Bounded Distribution Shift}

\makeatletter
\def\@fnsymbol#1{\ensuremath{\ifcase#1\or \dagger\or *\or \ddagger\or
   \mathsection\or \mathparagraph\or \|\or **\or \dagger\dagger
   \or \ddagger\ddagger \else\@ctrerr\fi}}
\makeatother

\author{Yatong Chen\thanks{These authors contributed equally to this work.}\ , Reilly Raab\footnotemark[1]\ , Jialu Wang, Yang Liu\thanks{Corresponding author: \href{mailto:me@yangliu@ucsc.edu}{yangliu@ucsc.edu}}
  \\
  \\
  \normalsize{University of California, Santa Cruz}
  \\
    {\small\texttt{\{ychen592, reilly, faldict, yangliu\}@ucsc.edu}}\\
}

\begin{document}

\maketitle

  \begin{abstract}
Given an algorithmic predictor that is ``fair'' on some \emph{source}
distribution, will it still be fair on an unknown \emph{target} distribution
that differs from the source within some bound? In this paper, we study the
\emph{transferability of statistical group fairness} for machine learning
predictors (\ie, classifiers or regressors) subject to bounded distribution
shift\rr{. Such shifts may be introduced by initial training data uncertainties,} user adaptation to a deployed \rr{predictor}, dynamic environments, or \rr{the use of pre-trained models in new settings}. Herein, we develop a bound \rr{that} characteriz\rr{es} such
transferability\rr{,} flagging potentially inappropriate deployments of machine
learning for socially consequential tasks. We first develop a framework for
bounding violations of statistical fairness subject to distribution shift,
formulating a generic upper bound for transferred fairness violation\rr{s} as our
primary result.  We then develop bounds for specific worked examples, \rr{focusing on}
two commonly used fairness definitions (\ie, demographic parity and equalized
odds) \rr{and} two classes of distribution shift (\ie, covariate shift and label
shift). Finally, we compare our theoretical bounds to deterministic models of
distribution shift and against real-world data, finding that we are able to
estimate fairness violation bounds in practice, even when simplifying
assumptions are only approximately satisfied.
\end{abstract}

\input{Section/Introduction}

\input{Section/Formulation}
\input{Section/GeneralBound}
\input{Section/CovariateShift}
\input{Section/TargetShift}
\input{Section/Simulation}
\input{Section/Discussion}

\vspace{-0.1in}
\paragraph{Acknowledgement} This work is supported by the National Science Foundation (NSF) under grants IIS-2143895, IIS-2040800 (FAI program in collaboration with Amazon), and CCF-2023495.
\bibliography{references}
\bibliographystyle{plainnat}

\newpage
\onecolumn
\appendix
\input{Section/appendix}

\end{document}

%% file: macros.tex
\usepackage[english]{babel} 
\usepackage[utf8]{inputenc} 

\usepackage{amsfonts}       
\usepackage{nicefrac}       
\usepackage{microtype}      
\usepackage{dsfont}         
\usepackage{bm}             
\usepackage{xcolor}         
\usepackage{xspace}         

\usepackage[]{amsthm}       
\usepackage[]{mathtools}    

\usepackage{physics}

\usepackage{booktabs}       
\usepackage{graphicx}       
\usepackage[]{tikz}         
\usepackage[]{pgfplots}     
\usepackage{float}          
\usepackage{wrapfig}        
\usepackage{caption}

\usepackage{url}                               
\usepackage{hyperref}                          
\usepackage{crossreftools}
\crefformat{footnote}{#2\footnotemark[#1]#3} 
\usepackage{refcount}                        

\usepackage[many]{tcolorbox}
\tcolorboxenvironment{proof}{
  blanker,
  before skip=12pt,
  after skip=12pt,
  borderline west={0.4pt}{0.4pt}{black},
  breakable,
  left=8pt,
}

\newcommand{\define}{\coloneqq}

\DeclareMathOperator*{\E}{\mathbb{E}} 

\newcommand{\Indicator}{\mathds{1}}
\DeclareMathOperator*{\argmax}{argmax}

\DeclarePairedDelimiter\ceil{\lceil}{\rceil}
\DeclarePairedDelimiter\floor{\lfloor}{\rfloor}

\newcommand{\RR}{\mathbf{R}} 

\newcommand{\vs}{\emph{vs.}\xspace}   
\newcommand{\eg}{\emph{e.g.}\xspace}  
\newcommand{\ie}{\emph{i.e.}\xspace}  
\newcommand{\etc}{\emph{etc.}\xspace} 
\newcommand{\viz}{\emph{viz.}\xspace} 

\usepackage{ulem}       


%% file: Section/Introduction.tex
\section{Introduction}

\emph{Distribution shift} is a common, real-world phenomenon that affects
machine learning deployments when the \emph{target} distribution of examples
(features and labels) \rr{ultimately} encountered by a \rr{data-driven} policy diverges from the
\emph{source} distribution it was trained for. For socially consequential
decisions guided by machine learning, such shifts in the underlying distribution
can invalidate fairness guarantees and cause harm by exacerbating social
disparities.
Unfortunately, distribution shift can be technically difficult or impossible to
model at training time \rr{(\eg, when depending on complex social dynamics or
unrealized world events)}.
Nonetheless, we still wish to
certify the robustness of fairness \rr{metrics for} a policy on possible target
distributions.

In this paper, we provide a general framework for quantifying the robustness of statistical group fairness
guarantees.
We assume that the target distribution is \rr{adversarially} drawn from a bounded
domain, thus reducing the hard problem of modelling distribution shift dynamics
to a more tractable, \rr{static} problem. \rr{With this framework, we
  can detect potentially inappropriate policy applications, prior to deployment, when
  fairness violation bounds are not sufficiently small}.

This work bridges a gap between recent literature on \emph{domain adaptation},
which has largely focused on the transferability of prediction accuracy (rather
than fairness), and \emph{algorithmic fairness}, which has typically
considered static distributions or prescribed models of
distribution shift. Our work is the first to systematically bound quantifiable
violations of statistical group fairness while remaining agnostic to (1) the
mechanisms responsible for distribution shift, (2) \rr{how group-specific
distribution shifts are quantified}, and (3) the specific statistical
definition of \rr{group} fairness applied.

\begin{wrapfigure}{r}{.44\linewidth}
    \centering
    \includegraphics[width=\linewidth]{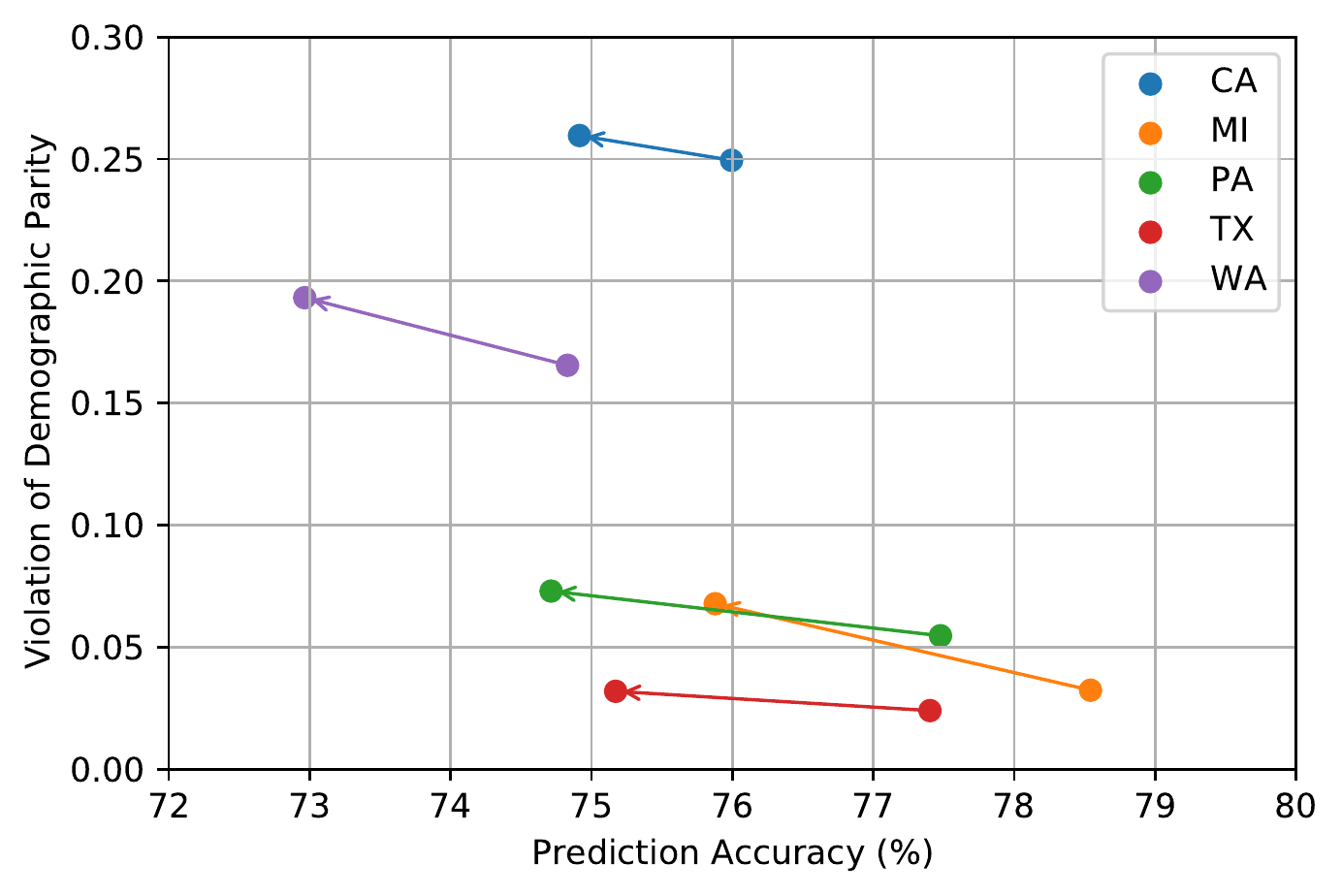}
    \caption{In \cref{sec:realcomparison}, we evaluate our bounds against historical, temporal distribution shifts in demographics and income recorded by the US Census Bureau~\cite{ding2021retiring}. The above figure depicts changes to \rr{income-prediction} accuracy and \rr{demographic parity violation} when a classifier initially trained on US state-specific \rr{demographic} data for 2014 is reused on 2018 data, thus exemplifying the negative potential effects of distribution shift.}
    \label{example:forecast-income}
\end{wrapfigure}
Our primary result is a bound on a policy's potential ``violation of statistical
group fairness''\rr{---}defined in terms of the differences in policy outcomes between
groups\rr{---}when applied to a target distribution shifted relative to the source
distribution within known constraints.
\rr{Such settings naturally arise whenever training data represents a random
sample of a target population with different statistics or a sample from dynamic
environments, when a policy is reused on a new distribution without retraining, or
whenever policy deployment itself induces a distribution shift}.
\rr{As an example of this last case, strategic individuals
seeking loans might change their features or abstain from future application (thus shifting the distribution of
examples) in response to policies trained on historical data} \cite{hardt2016strategic, ustun2019actionable, zhang2020fair}. 
\rr{Beyond policy selection, exogenous pressure such as economic trends and noise
may also drive distribution shift in this example.}

In \cref{example:forecast-income}, we show how a real-world
distribution shift in demographic and income data for US states between 2014 and
2018 may increase fairness violations while decreasing accuracy for a
hypothetical classifier trained on the 2014 distribution. In such settings, it
is useful to quantify how fairness guarantees transfer across distributions
shifted within some bound, thus allowing the deployment of unfair machine
learning policies to be avoided.

\vspace{-0.1in}

\vspace{-0.05in}
\subsection{Related Work}
\vspace{-0.1in}

Our work \rr{considers} a setting similar \rr{to} recent \rr{studies of}
\emph{domain adaptation}, which have largely focused on characterizing the
effects of distribution shift on prediction performance rather than fairness.
\rr{Our work also} builds on \rr{efforts in} \emph{algorithmic fairness},
\rr{especially \emph{dynamical} treatments of distribution shift in response to deployed machine learning policies} \cite{liu2018delayed,creager2020causal, raab2021unintended}. 
We reference specific prior work in these domains in \cref{secA:related-work}\rr{, and here discuss existing work}
 that focuses on how certain measures of fairness are affected when policies are subject to
specific distribution shift.

\textbf{Fairness subject to Distribution Shift:}
A number of recent studies have considered specific examples of
fairness transferability subject to distribution shift
\cite{schumann2019transfer, Coston2019fair, Singh2021fairness,rezaei2021robust, kang2022certifying}. 
In particular, \citet{schumann2019transfer} examine \emph{equality of
opportunity} and \emph{equalized odds} as definitions of group fairness subject
to distribution shifts quantified by an $\H$-divergence function; \citet{Coston2019fair}
consider \emph{demographic parity} subject to a \emph{covariate shift} assumption
while group identification remains unavailable to the classifier;
\citet{Singh2021fairness} focus on common group fairness definitions for binary
classifiers subject to a class of distribution shift that generalizes covariate
shift and label shift by preserving some conditional probability between
variables; and
\citet{rezaei2021robust} similarly consider common binary classification
fairness definitions \rr{such as equalized odds} subject to covariate
shift.
While we address similar settings to these works as special cases of our bound,
we propose a unifying formulation for a broader class of statistical group fairness
definitions and distribution shifts. In doing so, we recognize that particular
settings recommend themselves to more natural measures of distribution shift,
providing examples in \cref{sec:covariateDP}, \cref{sec:covariateEO}, and
\cref{sec:labelDP}). 

\rr{Another thread in existing literature is the development of robust models with the goal of guaranteeing fairness on a modelled target distribution} (e.g., \cite{an2022transferring,Roh2021sample,mandal2020ensuring, biswas2021ensuring, kang2022certifying})
, for example, by assuming covariate shift
and the availability of some unlabelled target data \cite{Coston2019fair,
  Singh2021fairness, rezaei2021robust}.
In particular, \citet{Singh2021fairness}
focus on learning stable models that will preserve prediction accuracy and
fairness, utilizing a causal graph to describe anticipated distribution
shifts. \citet{rezaei2021robust} takes a robust optimization
approach, and \citet{Coston2019fair} develops prevalence-constrained and
target-fair covariate shift method for getting the robust model. In contrast, our goal is to quantify fairness violations after an \emph{adversarial} distribution shift for \rr{\emph{any} given policy, including those not trained with robustness in mind}.

\vspace{-0.1in}
\subsection{Our Contributions}
\vspace{-0.1in}

Our primary contribution is formulating a general, \rev{worst-case} upper bound for a \rev{given} policy's
violation of statistical group fairness subject to group-dependent distribution
shifts \rev{within presupposed bounds} (i.e., \cref{eq:mainresult}).
Bounding
violations of fairness subject to distribution shift allows us to recognize and avoid potentially
inappropriate deployments of machine learning \rev{when the potential disparities of a prospective policy eclipse a given threshold within bounded distribution shifts of the training distribution.}

We first characterize the space of statistical
group fairness definitions and possible distribution shifts by appeal to
premetric functions (\cref{def:premetric}). \rr{After formulating the worst-case upper bound, we explore} common sets of
simplifying assumptions for this bound as special cases, yielding tractable
calculations for several familiar combinations of fairness definitions and
subcases of distribution shift (\cref{thm:dp-covariate}, \cref{thm:dp-label}) with
readily interpretable results. \rr{Finally,} we compare our theoretical bounds to
prescribed models of distribution shift in \cref{sec:syntheticcomparison} and to
real-world data in \cref{sec:realcomparison}. \rr{The details for reproducing our experimental results can be found at \\ {\small \texttt{\url{https://github.com/UCSC-REAL/Fairness_Transferability}}}.} 

%% file: Section/Formulation.tex

\section{Formulation}\label{sec:formulation}
\emph{The \rr{appendices} include a table of notation (\cref{secA:notation}) and all proofs (\cref{app:proofs}).}
\vspace{-0.1in}
\subsection{Algorithmic Prediction}
\vspace{-0.05in}

We consider two distributions, \(\S\) (\emph{source}) and \(\R\)
(\emph{target}), each defined as a probability distribution for
\emph{examples}, where each example defines values for three random variables:
\(X\), a \textbf{feature} (\eg, \(x\)) with arbitrary domain \(\mathcal{X}\); \(Y\), a \textbf{label} (\eg, \(y\)) with arbitrary domain \(\mathcal{Y}\); and \(G\), a \textbf{group} (\eg, \(g\) or \(h\)) with finite, countable domain \(\mathcal{G}\). The predictor's policy \(\pi\), intended for \(\S\) but used on \(\R\), defines a fourth variable for each example: \viz, \(\hat{Y}\), a \textbf{predicted label} (\eg, \(\hat{y}\)) with domain \(\hat{\mathcal{Y}}
    = \mathcal{Y}\).


Using \(\mathcal{P}(\cdot)\) to denote the space of probability distributions
over some domain, we denote the space of distributions over examples
as \( \DD \define \mathcal{P}(\mathcal{X} \times \mathcal{Y} \times \mathcal{G}) \), such that \(\S, \R \in \DD\).
It will also be useful for us to notate the space of distributions over example
\emph{outcomes} associated with a given policy as
\( \OO \define \mathcal{P}(\mathcal{X} \times \mathcal{Y} \times \hat{\mathcal{Y}}) \)
and the space of distributions over of \emph{group-specific examples} as
\(\GG \define \mathcal{P}(\mathcal{X} \times \mathcal{Y})\).


Without loss of generality, we allow the prediction policy \(\pi\) to be
stochastic, such that, for any combination \((x, g)\), the predictor effectively
samples \(\hat{Y}\) from a corresponding probability distribution \(\pi(x, g)\). 
Stochastic classifiers arise in various constrained optimization problems and \rr{proven useful} for making problems with custom losses or fairness constraints tractable \cite{cotter2019stochastic,grgichlaca2017fairness, narasimhan2018complex,wu2022metric}.

We denote the space of nondeterministic policies as
\(\Pi \define (\mathcal{X} \times \mathcal{G} \to \mathcal{P}(\hat{\mathcal{Y}}))\)
(\eg, \(\pi \in \Pi\)) and utilize the natural transformations that relate the spaces of distributions \(\DD\), policies \(\Pi\), and outcomes \(\OO\):
\begin{equation}\label{eq:pr-notation}
  \begin{aligned}
    \Pr_{\pi, \R}(\hat{Y} {=} \hat{y}, X {=} x, G {=} g) =
    \Pr_{\hat{Y} \sim \pi(x, g)}(\hat{Y} {=} \hat{y}) \cdot  \Pr_{X, G \sim \R}&(X {=} x, G {=} g)
  \end{aligned}
\end{equation}
We abuse the $\Pr$ notation for both probability
\emph{density} and probability \emph{mass} functions as appropriate.
\vspace{-0.1in}
\subsection{Statistical Group-Fairness}
\vspace{-0.05in}

We next define a broad class of \emph{disparity} functions
\(
\disp \colon \Pi \times \DD \to \RR
\)
representing how ``unfair'' a given policy is for a given distribution
(\eg, writing \(\disp(\pi, \mathcal{T})\))\rr{, noting that this notion of fairness is} limited to capturing statistical discrepancies of outcomes between groups.
\begin{definition} \label{def:premetric} We define a \emph{premetric}\footnote{\rr{Despite use on Wikipedia, this is not a standard term in the literature. In general, the axioms of a premetric as defined in \cref{def:premetric} are a subset (thus ``pre'') of those that define a metric.}} \(\Psi\) on
  the space of distributions \(p\) with respect to \(q\)
  by the properties \(\Psi(p \parallel q) \geq 0\) and \(\Psi(p \parallel p) = 0\)
  for all \(p, q\), and refer to the value of \(\Psi\) as a ``shift''.
\end{definition}
\begin{definition} \label{def:disparity} 
  We define \rr{a} \emph{statistical group disparity} \(\disp\) for
  policy \(\pi\) and distribution \(\R\) in terms of the
  symmetrized shifts between group-specific outcome distributions. We measure shifts between outcome
  distributions with a given premetric
  \(
\Psi \colon \OO^2 \to \RR
  \).
  \begin{equation} \label{eq:premetric}
    \disp(\pi, \R) \define \sum_{g, h \in \mathcal{G}}
    \Psi \Big(
    \Pr_{\pi, \R}(X, Y, \hat{Y} \mid G {=} g) ~\big\|~
    \Pr_{\pi, \R}(X, Y, \hat{Y} \mid G {=} h)
    \Big)
  \end{equation}
\end{definition}
In \cref{def:disparity},
\(\Psi\) quantifies the \rr{specific} statistical differences in outcomes between groups that
are "unfair", where a value of 0 implies perfect fairness. In this work,
we assume that \(\Psi\) is the same for all \(g,h\) and that \(\disp\) is insensitive to relative group size \(\Pr(G)\).
\vspace{-0.1in}
\paragraph{Examples}
Familiar applications of \cref{def:disparity} include \emph{demographic parity}
(\DP) and \emph{equalized odds} (\EO). A policy satisfying \DP\rr{,
  in expectation, assigns
  a given binary classification \(y \in \{0, 1\}\) to the same fraction of examples in each group}. We may measure the
violation of \DP as
\begin{equation} \label{eq:def-demographic-parity}
\begin{multlined}
      \disp_\DP(\pi, \R)
      \define \sum_{g, h \in \mathcal{G}} \Big|
      \Pr_{\pi, \R}( \hat{Y} {=} 1 \mid G {=} g) - \Pr_{\pi, \R}( \hat{Y} {=} 1 \mid G {=} h)
      \Big|
\end{multlined}
\end{equation}
The associated \emph{premetric} \(\Psi_\DP\) \rr{for} \(p, q \in \OO\) is
$
  \Psi_\DP(p \parallel q) =
  \Big|
    \Pr_p(\hat{Y} {=} 1) - \Pr_q(\hat{Y} {=} 1)
  \Big|.
$

To satisfy \EO, for binary \(\mathcal{Y} = \{0, 1\}\), \(\pi\) must maintain
group-invariant true positive and false positive classification rates.
We may measure the violation of \(\EO\) as
\begin{align}\label{eq:def-equalized-odds}
  \disp_\EO(\pi, \R)
  \define \sum_{g, h \in \mathcal{G}}
  \sum_{y \in \mathcal{Y}} \Big|
    \Pr_{\pi, \R}(\hat{Y} {=} 1 \mid G {=} g, Y {=} y) -
    \Pr_{\pi, \R}(\hat{Y} {=} 1 \mid G {=} h, Y {=} y)
  \Big|
\end{align}
The associated premetric is
$
  \Psi_\EO(p \parallel q) = \sum_{y \in \mathcal{Y}} \Big|
  \Pr_p(\hat{Y} {=} 1 \mid Y {=} y) - \Pr_q(\hat{Y} {=} 1 \mid Y {=} y)
  \Big|.
$
Note that the restriction of \EO to the \((Y = 1)\) case is known as \emph{Equal
Opportunity} (\EOp).

We remark that \cref{def:disparity} provides a unifying representation for a wide array of
statistical group ``unfairness'' definitions and may be used with inequality constraints.
That is, we may recover many working definitions of fairness that
effectively specify a maximum value of \rr{disparity}:
\begin{definition}
  A policy \(\pi\) is \textbf{\(\boldsymbol{\epsilon}\)-fair} with respect to \(\disp\)
  on distribution \(\R\) iff \(\disp(\pi, \R) \leq \epsilon\).
\end{definition}
\vspace{-0.1in}
\subsection{Vector-Bounded Distribution Shift}
Suppose, after developing policy \(\pi\) for distribution \(\S\), we realize
some new distribution \(\R\) on which the policy is actually operating. This
realization may be the consequence of sampling errors during the learning
process, strategic feedback to our policy, random processes, or the reuse of our
policy on a new distribution for which retraining is impractical.
Our goal is to bound \(\disp(\pi, \R)\) given knowledge of \(\disp(\pi, \S)\)
and some notion of how much \(\R\) \rr{possibly} differs from \(S\).
\begin{definition} \label{def:divergence}
  \(K(p \parallel q)\) is a \emph{divergence} if and only if for all \(p\) and \(q\), \(K(p \parallel q) \geq 0\) and \(K(p \parallel q) = 0 \iff q = p\).
\end{definition}
\begin{definition} \label{def:vector-divergence}
  Define the group-vectorized shift \(\mathbf{D}\), as \(\S\) mutates into \(\R\), as
  \begin{equation}
    \mathbf{D}(\R \parallel \S) \define \sum_g \mathbf{e}_g D_{g}\big(
    \Pr_{\R}(X, Y \mid G {=} g)
    \parallel
    \Pr_{\S}(X, Y \mid G {=} g)
    \big)
\end{equation}

where \(\mathbf{e}_{g}\) represents a unit vector indexed by \(g\), and each
\(D_{g} \colon \GG^{2} \to \RR\) is a divergence \rr{(\cref{def:divergence})}. Note that each \(D_g\) also defines a premetric (but not necessarily a divergence) on \(\DD\).
\end{definition}
\begin{assumption} \label{asm:budget}
Let there exist some vector \(\mathbf{B} \succeq 0\) bounding \(\mathbf{D}(\R \parallel \S) \preceq \mathbf{B}\), where \(\preceq\)
and \(\succeq\) denote element-wise inequalities.
\end{assumption}
In \cref{asm:budget}, \(\mathbf{B}\) limits the possible distribution shift as
\(\S\) mutates into \(\R\), without requiring us to specify a model for how
distributions evolve. When modelling distribution shift requires complex dynamics (\eg, when agents learn and respond to classifier policy), we reduce a potentially difficult dynamical problem to a more tractable, adversarial problem to achieve a bound.
\begin{lemma} \label{lem:v-of-B0-is-R}
  For all \(\pi\), \(\disp\), and \(\mathbf{D}\), when \(\mathbf{B} = 0\),
\(\disp(\pi, \S) = \disp(\pi, \R)\).
\end{lemma}
\cref{lem:v-of-B0-is-R} indicates that, for a fixed policy \(\pi\), a change
in disparity requires a measurable shift in distributions \rr{from $\S$ to $\R$}, confirming intuition.
\vspace{-0.1in}
\paragraph{Restricted Distribution Shift}
Common assumptions that restrict the set of  distribution shifts include \emph{covariate shift} and \emph{label shift}.
For covariate shift, the distribution of \emph{labels} conditioned on \emph{features} is preserved across distributions for all groups, while for label shift, the distributions of \emph{features} conditioned on \emph{labels} is preserved across distributions for all groups.
\begin{align} \label{eq:def-covariate-shift}
\text{\textbf{Covariate shift} implies }&\Pr_{\R}(Y \mid X, G) = \Pr_{\S}(Y \mid X, G) \\
\label{eq:def-label-shift}
\text{\textbf{Label shift} implies }& \Pr_{\R}(X \mid Y, G) = \Pr_{\S}(X \mid Y, G)
\end{align}
In \cref{sec:covariate-shift}, we explore a deterministic model of a population's response to classification \rr{as an example of} covariate shift. We do the same in \cref{sec:labelDP} for label shift.

%% file: Section/GeneralBound.tex
\vspace{-0.1in}
\section{General Bounds}\label{sec:bounds}
\vspace{-0.1in}
We first define a primary bound in \cref{def:v} before considering simplifying special cases.

Given an element-wise bound \(\mathbf{B}\) on the vector-valued shift
\(\mathbf{D}(\R \parallel \S)\) (\cref{asm:budget}) we may bound the disparity
\(\disp\) of policy \(\pi\) on any realizable target distribution \(\R\) by its supremum value.
\begin{definition} \label{def:v}
  Define the supremum value \(v\) for \(\disp\) subject to
\(\mathbf{D}(\R \parallel \S) \preceq \mathbf{B}\) as
\begin{align} \label{eq:v}
v(\disp, \mathbf{D}, \pi, \S, \mathbf{B})
  &\define
  \sup_{ \mathbf{D}(\R \parallel \S) \preceq \mathbf{B}} \disp(\pi, \R)\\
 \label{eq:mainresult}
  \mathbf{D}(\R \parallel \S) \preceq \mathbf{B}
  &\implies
  \disp(\pi, \R)
  \leq
  v(\disp, \mathbf{D}, \pi, \S, \mathbf{B})
\end{align}
\end{definition}


In general, our strategy is to exploit the mathematical structure of the setting
encoded by \(\disp\) (\ie, \(\Psi\)) and \(\mathbf{D}\) to obtain an
upper bound for \(v\) defined in \cref{eq:v}. We first explore general cases of
simplifying assumptions before presenting worked special examples for frequently
encountered settings. Finally, we compare the resulting theoretical bounds to
numerical results and simulations.
\vspace{-0.1in}
\subsection{Lipshitz Conditions}
\label{sec:lipshitz}
The value of \(v\) defines a scalar field in \rr{\(\mathbf{B}\) and therefore} a conservative vector field \(\mathbf{F} = \nabla_{\mathbf{B}} v\). 

\begin{wrapfigure}{r}{.4\linewidth}
   \centering \includegraphics[width=\linewidth]{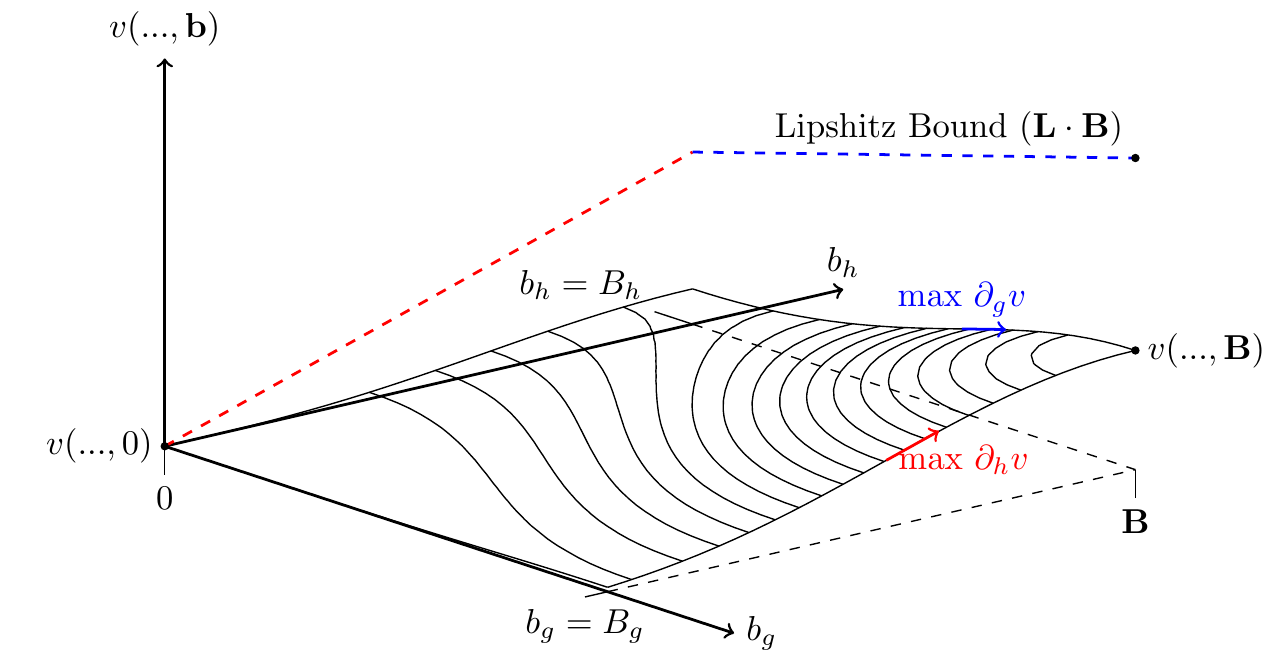}
   \caption{A Lipshitz bound for all curves \rr{parameterized by distribution shift bound \(\mathbf{b}\)} in the \((0, \mathbf{B})\) \(\DD\)-hyperrectangle \rev{on the surface \(v\). In the figure, for groups \(i \in \{g, h\}\), \(\max \partial_i v = L_i\), and the colored dotted lines corresponds to \(L_i  b_i\), which, when summed, equal \(\mathbf{L} \cdot \mathbf{B}\). }}
 \label{fig:lipshitz}
\end{wrapfigure}
For any curve in \(\DD\) from \(\S\) to \(\R\), bounds of
the form \(\mathbf{F} \preceq \mathbf{L}\) for some constant \(\mathbf{L}\) along the curve imply a Lipshitz bound on \(\disp\). We visualize a bound in \cref{fig:lipshitz} for all possible curves in the region \(\mathbf{D}(\R \parallel \S) \preceq \mathbf{B}\).
\begin{theorem}[Lipshitz Upper Bound] \label{thm:lipshitz} If there exists an \(\mathbf{L}\) such that
\(
   \nabla_{\mathbf{b}} v(\disp, \mathbf{D}, \pi, \S, \mathbf{b}) \preceq \mathbf{L}
\),
everywhere along some curve \rr{as \(\mathbf{b}\) varies} from \(0\) to \(\mathbf{B}\), then
\begin{equation}
  \begin{aligned} \label{eq:lipshitz-upper}
  \disp(\pi, \R) \leq \disp(\pi, \S) + \mathbf{L} \cdot \mathbf{B}
\end{aligned}
\end{equation}
\end{theorem}
\rev{Succinctly, if we are guaranteed that disparity can never increase faster than a certain rate in some measure of distribution shift, then, given a maximum distribution shift, this rate bounds the maximum possible disparity.}
The utility of \cref{thm:lipshitz} arises when a Lipshitz condition \(\mathbf{L}\) is
known, but direct computation of \(v\)
is difficult. We provide an example of a Lipshitz bound in \cref{sec:labelDP}.


\subsection{Subadditivity Conditions}
\label{sec:subadditive}
\begin{definition} \label{def:w}
  Define $w$ as the maximum \emph{increase} in disparity
subject to \(\mathbf{D}(\R \parallel \S) \preceq \mathbf{B}\), \ie,
\(w(\disp, \mathbf{D}, \pi, \S, \mathbf{B}) \define v(\disp, \mathbf{D}, \pi, \S, \mathbf{B}) - \disp(\pi,\S)\).
\end{definition}

\begin{theorem} \label{thm:subadditive}
  Suppose, in the region \(\mathbf{D}(\R \parallel \S) \preceq \mathbf{B}\), that
  \(w\) is subadditive in its last argument.
That is,
\(
    w(..., \mathbf{a}) + w(..., \mathbf{c})
    \geq w(..., \mathbf{a} + \mathbf{c})
    \) for \(\mathbf{a}, \mathbf{c} \succeq 0\) and \(\mathbf{a} + \mathbf{c} \preceq \mathbf{B}\).
If \(w\) is also locally differentiable, then a first-order approximation of \(w(..., \mathbf{b})\) evaluated at
\(0\), \ie,
\begin{equation}
  \mathbf{L} =
  \nabla_{\mathbf{b}} w(..., \mathbf{b}) \big|_{\mathbf{b} = 0} =
  \nabla_{\mathbf{b}} v(..., \mathbf{b}) \big|_{\mathbf{b} = 0}
 \end{equation}
 provides an upper bound for \(v(..., \mathbf{B})\), \ie,
 \begin{equation}
   v(\disp, \mathbf{D}, \pi, \S, \mathbf{B})
   \leq
   \disp(\pi, \S) + \mathbf{L} \cdot \mathbf{B}
\end{equation}
\end{theorem}
\cref{thm:subadditive} notes that ``diminishing returns'' in the change of \(\disp\) as the
difference of \(\R\) with respect to \(\S\) is increased implies a bound on \(\disp\) in terms of its local sensitivity to \(\mathbf{D}\) at \(\S\) (\ie, using a first-order Taylor approximation). Note that, if \(w\) is concave in the bounded region, it is also subadditive in the bounded region, but the converse is not true, nor does the converse imply Lipshitzness.


\vspace{-0.1in}
\subsection{Geometric Structure} \label{sec:metricbounds}

It may happen that \(\Psi \colon \OO^2 \to \RR \) and each
\(D_{g} \colon \GG^2 \to \RR \) share structure that permits a geometric
interpretation of distribution shift. While the utility of this observation
depends on the specific properties of \(\Psi\) and \(\mathbf{D}\), we
demonstrate a worked example building on \cref{sec:covariateEO} in
\cref{app:geometric}, in which we allow ourselves to select a suitable
\(\mathbf{D}\) for ease of interpretation. We proceed to consider worked examples that adopt common  assumptions
limiting the form of distribution shift and apply common definitions of statistical
group fairness.

%% file: Section/CovariateShift.tex
\section{Covariate Shift}
\label{sec:covariate-shift}
\vspace{-0.1in}
We now present our fairness transferability results subject to covariate
shift for both demongraphic parity (\Cref{sec:covariateDP}) and equalized opportunity (\Cref{sec:covariateEO}) as fairness criteria.
\subsection{Demographic Parity} \label{sec:covariateDP}

The simplest way to work with \cref{eq:v} is to bound the supremum \(v\). We first
consider demographic parity (\cref{eq:def-demographic-parity}) for
\(\mathcal{Y} = \{0,1\}\) and \(\mathcal{G} = \{g, h\}\), subject to covariate
shift (\cref{eq:def-covariate-shift}). We find that the form of \(\dispdp\)
subject to covariate shift recommends itself to a natural choice of vector divergence, \(\mathbf{D}\). \rr{First,} define a \emph{re-weighting coefficient} \(
 \omega_g(\R, \S, x) \define \frac{\Pr_{\R}(X {=} x \mid G {=} g)}{\Pr_{\S}(X {=} x \mid  G {=} g)}
\)\rr{.}
\begin{theorem}\label{thm:dp-covariate}
  For demographic parity between two groups under covariate shift (denoting, for
each \(g\), \(\beta_{g} \define \Pr_{\pi,\S}(\hat{Y} {=} 1 \mid G {=} g)\)), 
  \begin{equation}
    \disp_{\DP}(\pi, \R) \leq \disp_{\DP}(\pi, \S) + \sum_g
    \big(\beta_{g}(1 - \beta_{g}) \cdot \var_{\S}[\omega_g(\R, \S, x)]\big)^{\nicefrac{1}{2}}
  \end{equation}
\vspace{-0.1in}
\end{theorem}

We notice
that $\var_{\S}[\omega_g(\R, \S, x)]$ recommends itself as a suitable divergence
\(D_{g}\) from \(\S\) to \(\R\). Using basis vectors \(\mathbf{e}_g\), for this example, we could
define
\(
    \mathbf{D}(\mathcal{T} \parallel \S) = \sum_{g} \mathbf{e}_g
    \var_{\S}[\omega_g(\mathcal{T}, \S, x)]
\). 
When \(\var_{\S}[\omega_g(\R, \S, x)] \leq B_g\), it follows
   $
    \disp_{\DP}(\pi, \R) \leq \disp_{\DP}(\pi, \S) + \sum_g
    \big(\beta_{g}(1 - \beta_{g}) \cdot B_g\big)^{\nicefrac{1}{2}}
$.
Comparing the inequality in
\cref{thm:dp-covariate}
and the consequent of \cref{eq:lipshitz-upper}, we can interpret
\(\Pr_{\pi, \S}(\hat{Y} {=} 1)\) in \cref{thm:dp-covariate} as an upper bound for
the \emph{average} value of \(\nabla_{\mathbf{b}} v(\disp_\DP, \mathbf{D}, \pi, \R, \mathbf{b})\) along
any curve from \(\S\) to \(\R\). \rr{Interpreting} this result, the closer \({\rm Pr}(\hat{Y} {=} 1)\) is to 0.5 for
any group, the more potentially sensitive the fairness of the policy is to
distribution shifts for that group. We can further generalize the results to multi-class and multi-group setting:

\begin{corollary}\label{cor:dp-covariate-multi}
\cref{thm:dp-covariate} may be generalized to multiple classes
\(\mathcal{Y} = \{1, 2, ..., m\}\)
 and multiple groups \(\mathcal{G}\in\{1, 2, ..., n\}\), where \(\beta_{g,y} =\Pr(\hat{Y}{=}y \mid G{=}g)\) and assuming \(\var_{\S}[\omega_g(\R, \S, x)] \leq B_g\):
\begin{align}
  \dispdp(\pi, \R) &\define \sum_{y \in \mathcal{Y}} \sum_{g,h \in \mathcal{G}} \Big|
  \Pr_{\pi, \R}(\hat{Y} {=} y \mid G {=} g) -
  \Pr_{\pi, \R}(\hat{Y} {=} y \mid G {=} h)
  \Big| \\
  \disp_{\DP}(\pi, \R) &\leq \disp_{\DP}(\pi, \S) + \sum_{y} \sum_{g}
  \big(\beta_{g,y}(1 - \beta_{g,y}) \cdot B_g \big)^{\nicefrac{1}{2}}
\end{align}
\end{corollary}
We remark that in general, binary classification bounds may frequently be
generalized to multi-class bounds by redefining fairness violations as a sum
of binary-class fairness violations (\ie, same-class \vs different-class
labels) and summing the bounds on each.

\subsection{Equal Opportunity} \label{sec:covariateEO}

Consider an example using the \((Y {=} 1)\)-conditioned case
of Equalized Odds---termed \emph{Equal Opportunity} (\EOp). Denoting, for each group \(g\), the true positive rate
\(\beta^{+}_{g} \define \Pr_{\pi, \R}( \hat{Y} {=} 1 \mid Y {=} 1, G {=} g )\)
as an implicit function of \(\pi\) and \(\R\), we define disparity for \EOp as
\(
  \disp_\EOp(\pi, \R)
  \define \sum_{g, h \in \mathcal{G}}
 |
    \beta^{+}_{g} -
    \beta^{+}_{h}
 |
\).

We may bound the realized value of \(\disp_{\EOp}(\pi, \R)\) by bounding \(\beta^{+}_{g}\) for each group:
\begin{theorem} \label{thm:cov-EO-superbound}
 Subject to covariate shift and any given \(\mathbf{D}, \mathbf{B}\), assume extremal values for \(\beta_{g}^{+}\), \ie,
\begin{equation}
\label{eq:cov-EO-requisite}
\forall g, ~~
\big(D_{g}(\R \parallel \S) < B_{g}\big) \implies \big(
l_g \leq \beta^{+}_{g}(\pi, \R) \leq u_g\big)
\end{equation}
it follows that
\begin{equation} \label{eq:cov-eop-bound}
v(\disp_{\EOp}, \mathbf{D}, \pi, \S, \mathbf{B}) \leq \max_{\substack{{x_g \in \{l_g, u_g\}} \\ {x_h \in \{l_h, u_h\}}}} \sum_{g, h}\Big|
  x_g - x_h
\Big|
\end{equation}
\end{theorem}
\begin{corollary} \label{cor:max-cov-EO}
\rr{The disparity measurement} \(\disp_\EOp\) cannot exceed \(\frac{|\mathcal{G}|^2}{4}\).
\end{corollary}
In \cref{app:geometric}, we bound the extremal values of \(\beta_g^+\) by geometrically interpreting this quantity as an inner product on an appropriate vector space, utilizing the
freedom to select an appropriate \(\mathbf{D}\). 

%% file: Section/TargetShift.tex
\section{Label Shift} \label{sec:labelDP}
\vspace{-0.1in}
Under label shift (\(\Pr_\S(X|Y) = \Pr_\R(X|Y)\)), violations of \EO and \EOp are invariant, because
the independence of \(\hat{Y}\) and \(Y\) given \(X\)
implies
$\Pr_{\pi,\R}(\hat{Y}|Y) = \Pr_{\pi,\S}(\hat{Y}|Y)$. We therefore focus on the violation of
demographic parity (\DP) (\cref{eq:def-demographic-parity}) subject to the label shift condition, treating a binary
classification task over two groups for simplicity.

In this setting, we choose to measure group-specific distribution shifts \rr{from}
\(\S\) \rr{to} \(\R\) by the change in proportion of \rr{ground-truth} positive labels, which we
refer to as the group \emph{qualification rate}
\(Q_{g}(\R) \define  \Pr_\R(Y = 1 \mid G = g)\):
\begin{align} \label{eq:defDlabelDP}
  D_{g}(\R \parallel \S) \define \Big| Q_{g}(\mathcal{S}) - Q_{g}(\mathcal{\R}) \Big| \leq B_{g}
\end{align}
\begin{theorem} \label{thm:yangsthm}
A Lipshitz condition bounds \(\nabla_\mathbf{b} v(\disp_\DP, \mathbf{D}, \pi, \S, \mathbf{b})\) when
\begin{equation}
    D_{g}(\R \parallel \S) \define \Big| Q_{g}(\mathcal{S}) - Q_{g}(\mathcal{\R}) \Big| \leq B_{g}
\end{equation}
Specifically,
\begin{equation}
  \pdv{}{b_{g}} v(\disp_\DP, \mathbf{D}, \pi, \S, \mathbf{b}) \leq
  (|\mathcal{G}| - 1)
  \Big| \beta^{+}_{g} - \beta^{-}_{g} \Big|
\end{equation}
for true positive rates \(\beta^{+}_{g}\) and false positive rates \(\beta^{-}_{g}\):
\begin{equation}
    \beta^{+}_{g}
\define \Pr_{\pi}( \hat{Y} {=} 1 \mid Y {=} 1, G {=} g ); \quad \beta^{-}_{g} \define \Pr_{\pi}( \hat{Y} {=} 1 \mid Y {=} 0, G {=} g )
\end{equation} 
\end{theorem}
Because \(\beta^{+}_g\) and \(\beta^{-}_g\) are invariant under label shift
given a constant policy \(\pi\), we elide their explicit dependence on the underlying distribution.

\begin{theorem}\label{thm:dp-label}
For \DP under the bounded label-shift assumption $\forall g,  | Q_{g}(\S) - Q_{g}(\R) | \leq B_{g}$,
\begin{equation}
\begin{aligned} \label{eq:dplabelbound}
 \dispdp(\pi, \R)  &\leq  \dispdp(\pi, \S) + (|\mathcal{G}| - 1)
\sum_g B_{g} \ \Big|\beta^{+}_g - \beta^{-}_g \Big|
\end{aligned}
\end{equation}
\end{theorem}
\vspace{-0.15in}
Intuitively, the change in \(\disp_\DP\) subject to label shift \rr{depends on} \(|\beta^{+}_{g} - \beta^{-}_{g}|\), the marginal change in acceptance rates as agents change their qualifications \(Y\).
\rr{We measure the distribution shift as agents change their qualifications} by \(|Q_g(\S) - Q_g(\R)|\). When
$\beta^{+}_{g}$ is close to $\beta^{-}_{g}$, the policy looks like a random
classifier, and a label shift has limited effect on statistical group disparity.
When $|\beta^{+}_{g} - \beta^{-}_{g}|$ is large, indicating high classifier
accuracy, the effect \rr{on supremal disparity} is larger. Our bound thus exposes a direct trade-off between accuracy and fairness
transferability guarantees.

%% file: Section/Simulation.tex
\vspace{-0.1in}
\section{Comparisons to Synthetic Distribution Shifts (Demographic Parity)} \label{sec:syntheticcomparison}

To further interpret our results, in this section, we consider specific and popular agent models to characterize distribution shift and instantiate our bounds \rr{for particular forms of \(\mathbf{D}\), \(\mathbf{B}\), and \(\disp\)}. 

\subsection{Covariate Shift via Strategic Response} 
\label{sec:strategic-response} Let us consider a specific example of covariate
shift (\cref{eq:def-covariate-shift}) caused by a deterministic,
group-independent model of \emph{strategic response} in which agents react to a binary classification policy \(\pi\) characterized by
group-specific feature thresholds:
\begin{equation}
  \hat{Y} \sim \pi(x, g) = \begin{cases}
  1 &\text{ with probability } 1 \text{ if } x \geq \tau_{g} \\
  0 &\text{ with probability } 1 \text{ otherwise }
\end{cases}
\end{equation}


For simplicity, we assume the feature domain \(\mathcal{X} = [0, 1]\).
In response to threshold \(\tau_g\), agents in each group \(g\) may modify their feature \(x\) to \(x'\) by incurring a cost $c_g(x, x')\geq 0$. Similar to
\cite{hardt2016strategic}, we define the \emph{utility} \(u_g\) for agents in group \(g\) to be
\begin{align} \label{eq:strategic-utility}
  u_g(x, x') &\define \beta_g(x') - \beta_{g}(x) - c_g(x, x'); \quad \beta_{g}(x) \define \Pr(\hat{Y} {=} 1 \mid X {=} x, G {=} g),   \forall g.
\end{align}
Contrary to the standard strategic classification setting, we do not assume that
feature updates represent false reports, but that such updates may correspond to actual changes underlying the true qualification \(Y\)
of each agent. This assumption has been made in a recent line of research in
incentivizing improvement from human agents subject to such classification
\cite{chen2021linear}.

Next, we assume all agents are rational utility maximizers (\cref{eq:strategic-utility}).
For a given threshold $\tau_g$ and manipulation budget $m_{g}$, the best response
of an agent with original feature $x$ is
\begin{equation}
  x' = \argmax_{z}\ u_g(x, z), \quad \text{ such that }\ c_g(x, z) \leq m_{g}
\end{equation}
To make the problem tractable, we make additional assumptions about the agents' best responses. 
\begin{assumption}
    An agent's original feature $x$ is sampled as $X\sim \mathcal{U}_{[0,1]}$\footnote{where \(\mathcal{U}\) represents the uniform distribution.}.
\end{assumption}
\begin{assumption}
\label{assumption:cost-function}
The cost function $c_g(x, x')$ is  monotone in $|x - x'|$ as
\(
  c_g(x, x') =  |x' - x|
\).
\end{assumption}
Under \Cref{assumption:cost-function}, only those agents with features
$x \in [\tau_g - m_{g}, \tau_g)$ will \emph{attempt} to change their feature. We
also assume that feature updates are non-deterministic, such that agents with
features closer to the decision boundary \(\tau_{g}\) have a greater
\emph{chance} of updating their feature and each updated feature \(x'\) is
sampled from a uniform distribution depending on \(\tau_{g}\), \(m_{g}\), and \(x\):

\begin{assumption} 
\label{assumption:nondeterministic-feature-update}
For agents who \emph{attempt} to update their features, the probability of
a successful feature update is
\(
    \Pr(X \neq X') = 1 - \frac{|x - \tau_g|}{m_{g}}
\).
\end{assumption}

\begin{assumption}
\label{assumption:new-feature-uniform-distribution}
  An agent's updated feature \(x'\), given original feature \(x\), manipulation
  budget \(m_{g}\), and classification boundary \(\tau_{g}\), is sampled as
  \(
    X' \sim \mathcal{U}_{[\tau_g, \tau_g + m_{g} - x]}
  \).
\end{assumption}

With the above setting, we can specify the reweighting coefficient $\omega_g(x)$ for our setting (\cref{eq:strategic-omega} in \cref{sec:proof-strategic-response} and get the following bound for the strategic response setting\footnote{See \Cref{fig:sc-distribution-shift} in Appendix \ref{secA:additional-figures} for a demonstration of  \Cref{thm:dp-covariate}.}: 
\begin{proposition}
  \label{proposition:bound-strategic-response}
  For our assumed setting of strategic response involving \DP for two groups \(\{g, h\}\),
\cref{thm:dp-covariate} implies
\vspace{-0.1in}
\begin{equation}
    \dispdp(\pi, \R) 
    \leq \dispdp(\pi, \S) + \tau_g(1-\tau_g)\frac{2}{3}m_{g}  +
    \tau_h(1 - \tau_h)\frac{2}{3}m_{h}
\end{equation}
\end{proposition}
The above result shows that two factors lead to a smaller difference between the source and target fairness violations: a less stochastic classifier (when the threshold $\tau_g$ is far away from $0.5$) and a smaller manipulation budget $m_g$ (diminishing agents' ability to adapt their feature). In this case, $B_g =\frac{2}{3} \tau_g (1 - \tau_g)$.
These factors lead to less potential manipulation and result in a tighter upper bound for the fairness violation on \(\R\).
\vspace{-0.1in}
\begin{figure}
    \subfigure[A stereoscopic (cross-eye view) comparison between the bound of \cref{sec:covariateDP} (gradated) and simulated results for the model of \cref{sec:strategic-response} (blue) in response to a \DP-fair classifier with different initial group-independent acceptance rates. \label{fig:DP-covariate-experimental-small}]{\includegraphics[width=0.48\linewidth]{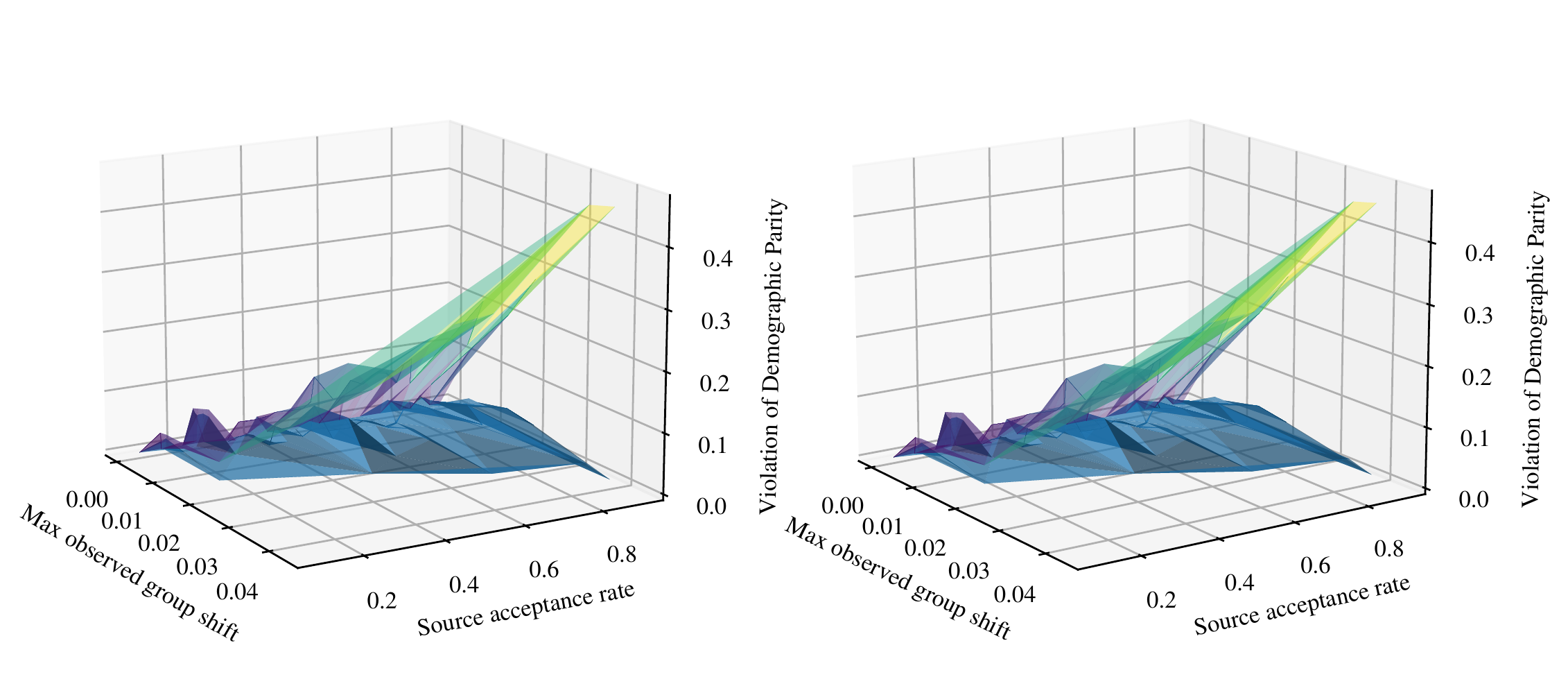}}
\hfill
\subfigure[\small{A policy satisfying \DP is subject to distribution shift prescribed by replicator dynamics (\cref{sec:replicator}). Realized disparity increases (blue) are compared to the theoretical bound (\cref{thm:dp-label}, gradated), which is tight when group have dissimilar qualification rates.}\label{fig:dp-replicator}  ]{\includegraphics[width=0.48\linewidth]{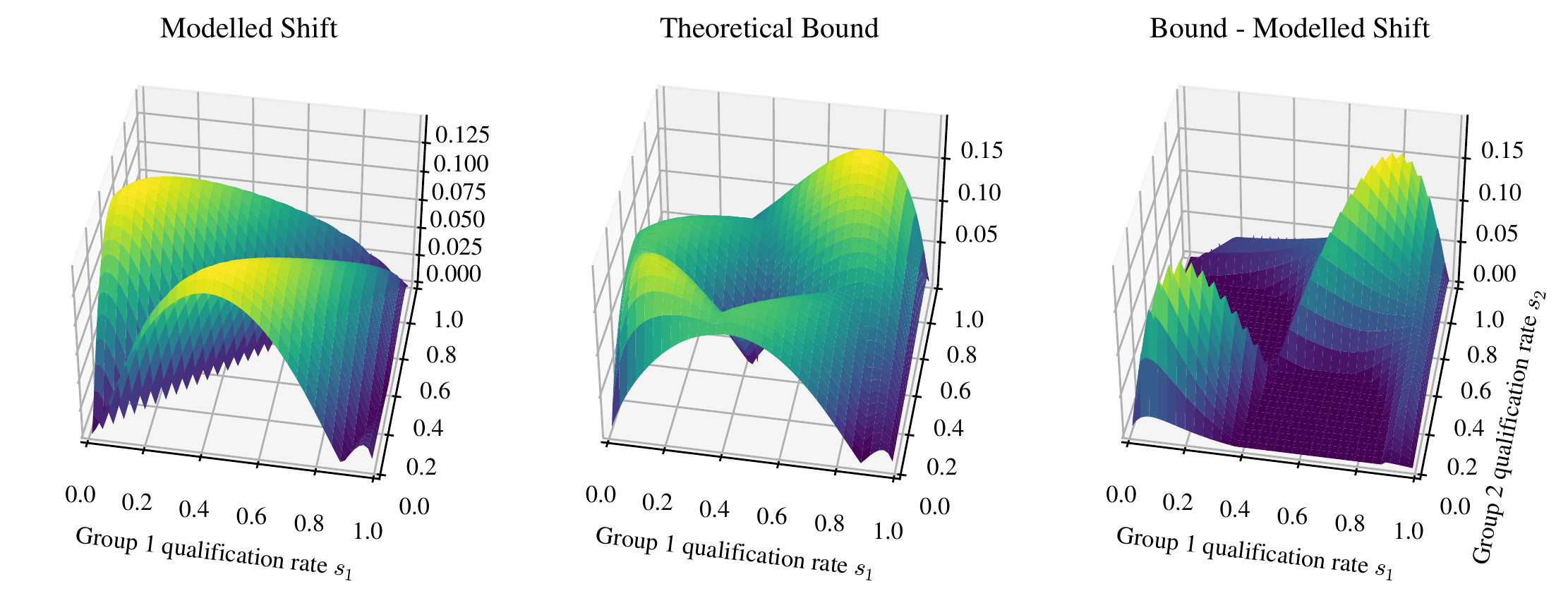}}

\caption{Comparisons to synthetic distribution. Larger versions are provided in \cref{sec:empirical-evals}.}
\end{figure}


\subsection{Label Shift via Replicator Dynamics} \label{sec:replicator}

We now evaluate our theoretical bound for demographic
parity subject to label shift (\cref{thm:dp-label}) on the replicator dynamics model of
\citet{raab2021unintended}. \rev{Briefly, replicator dynamics assumes that the proportion of agents in a population choosing one strategy over another grows in proportion to the ratio of average utilities realized by the two strategies.
} The cited model additionally assumes \(\mathcal{X} = \RR\),
\(\mathcal{Y} = \{\FP\}\), and a monotonicity condition for \(\S\) given by
\(
    \dv{}{x} \frac{\Pr_{\S}(X {=} x \mid Y {=} 1)}{\Pr_{\S}(X {=} x \mid Y {=} 0)} > 0
\).

Label shift under the discrete-time (\(t\)) replicator dynamics may be expressed
in terms of group \emph{qualification rates}
\(Q_{g} \define \Pr_{t}(Y {=} 1 \mid G {=} g)\) and agent utilities
(\ie, group- and feature-independent values \(U_{y,\hat{y}}\)) such
that, in each group, the popularity and average utility associated with a label determines
its frequency at the next time step \(t{+}1\).

Denote the \emph{fractions} of group-conditioned, feature-independent outcomes with the expression
\(\rho^{y,\hat{y}}_{g} \define \Pr_{t}(\hat{Y} {=} \hat{y}, Y {=} y \mid G {=} g)\)
and abbreviate the fraction-weighted utility as
\(u^{y,\hat{y}}_{g}(t) \define U_{y,\hat{y}} \cdot \rho^{y,\hat{y}}_{g}\). We may then
represent the replicator dynamics as
\begin{align} \label{eq:replicator-dynamics}
  Q_{g}[t + 1] &= \frac{u^{\TP}_{g}(t) + u^{\FN}_{g}(t) } {u^{\TP}_{g}(t) + u^{\FN}_{g}(t) +  u^{\TN}_{g}(t) + u^{\FP}_{g}(t) }
\end{align}
To apply \cref{thm:dp-label}, we also observe that
$
  | \beta^{+}_{g} - \beta^{-}_{g} | = \frac{| \rho^{\TP}_{g} - \rho^{\FP}_{g} |}{\rho^{\TP}_{g} + \rho^{\FP}_{g}},
$
where \(\beta^{+}_{g}\) and \(\beta^{-}_{g}\) represent the true positive rate
and false positive rate for group \(g\), respectively, and we use the change in
qualification rate as our measurement of label shift, \ie,
\(B_{g} = |Q_{g}[t + 1] - Q_{g}[t]|\).
When demographic parity is perfectly
satisfied, we note that the acceptance rate
(\(\rho^{\TP}_{g} + \rho^{\FP}_{g}\)) is group-independent.

\begin{theorem}
  \label{thm:label-replicator}
  For \DP subject to label replicator dynamics,
  \begin{equation} \label{eq:dp-replicator-bound}
  \begin{aligned}
\dispdp(\pi, \R) \leq \dispdp(\pi, \S) + \sum_g \Big|Q_{g}[t + 1] - Q_{g}[t]\Big| \frac{| \rho^{\TP}_{g} - \rho^{\FP}_{g} |}{\rho^{\TP}_{g} + \rho^{\FP}_{g}}
  \end{aligned}
\end{equation}
\end{theorem}
\vspace{-0.1in}
In \cref{fig:dp-replicator}, we graphically represent all possible states of an initially
fair system (thus determining \(\beta\) and \(\rho\) as a result of
the monotonicity condition) by the tuple of qualification rates for each group. With
the dynamics prescribed by \cref{eq:replicator-dynamics}, we depict the \emph{rate of change} of disparity given a fixed, locally \DP-fair policy, and compare this to the theoretical bound when \(B_{g} = |Q_{g}[t + 1] - Q_{g}[t]|\).

Interpreting our results, we note that the bound lacks information about the relative directions of the change in acceptance rates for each group, and thus over-approximates possible fairness violations when group acceptance rates shift the same direction. When group acceptance rates move in opposing directions, however, the bound gives excellent agreement with the modelled replicator dynamics.
\vspace{-0.1in}
\section{Comparisons to Real-World Distribution Shifts} \label{sec:realcomparison}
\vspace{-0.1 in}
We now compare our special-case theoretical bounds \rev{(\ie, label/covariate shift)} to real-world distribution
shifts and hypothetical classifiers. We use American Community Survey (ACS) data
provided by the US Census Bureau \cite{IPUMS}. We adopt the sampling and
pre-processing approaches following the \texttt{Folktables} package provided by
\citet{ding2021retiring}\footnote{This package is available at
\url{https://github.com/zykls/folktables}.} to obtain 1,599,229 data points. The
data is partitioned by (1) all fifty US states and (2) years from 2014 to
2018. We use 10 features covering the demographic information used in the UCI
Adult dataset \cite{Blake1998UCIRO}, including age, occupation, education, \etc,
as \(X\) for our model, select sex as binary protected group, i.e.,
\(G \in \{g=\texttt{female}, h=\texttt{male}\}\). We set the label \(Y\) to whether an individual's annual income is greater than \$50K. 

\rev{To apply our label-shift or covariate-shift bounds, we first need to verify whether the two datasets satisfy either of these assumptions. }
We adopted a conditional independence test \cite{NEURIPS2020_e2d52448}, \rev{which takes data from source and target domains as input and returns a divergence score for each covariate and label variable, reflecting to what extent the variable is shifted between distributions.} We find that the likelihood that the covariates shift across US states is approximately two orders of magnitude higher than for labels.
 More specifically, there are 4 covariates, including class of worker (probabilistic divergence score of 2.67e-2), hours worker per week (3.56e-2), sex (3.56e-2) and race (2.55e-1), that are more likely to be shifted than the label variable (1.29e-4). For temporal shifts within states, \rev{we find that the label variable is more likely to be shifted (0.1) than all the other covariates (which are below 0.01), approximately two orders of magnitude in favor of label shift over covariate shift. We therefore compare the disparities of hypothetical policies on these distributions to bounds generated from the corresponding, approximately satisfied assumptions.}

On this data, we train a set of group-dependent, linear threshold
classifiers $\Pr_{\pi(x, g)}(\hat{Y} {=} 1) = \Indicator[\sigma(w \cdot x) > \tau_g]$,
for a range of thresholds \(\tau_{g}\) and \(\tau_{h}\) for each source
distribution. Here, $\sigma(\cdot)$ is the logistic function and $w$ denotes a
weight vector. We then consider two types of real-world distribution shift: (1)
\emph{geographic}, in which a model trained for one state is
evaluated on other US state in the same year, and (2) \emph{temporal}, in which a model trained for 2014 is evaluated on the same
state in 2018.


\vspace{-0.2in}
\begin{figure}[H]
    \centering
    \subfigure[CA $\longrightarrow$ IL]{\includegraphics[width=0.24\linewidth]{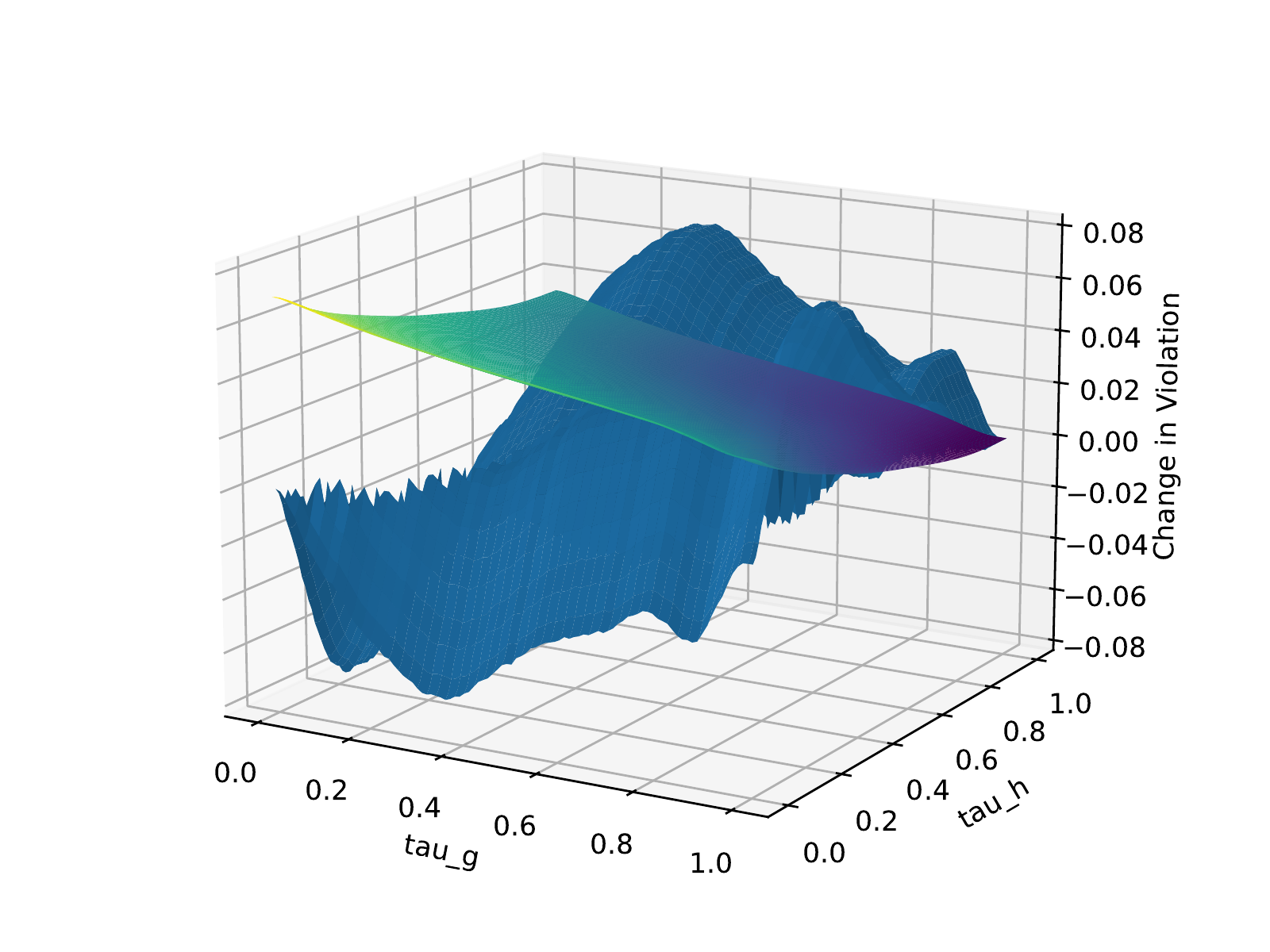}}
    \subfigure[CA $\longrightarrow$ NV]{\label{fig:simulation-state-ca-nv}\includegraphics[width=0.24\linewidth]{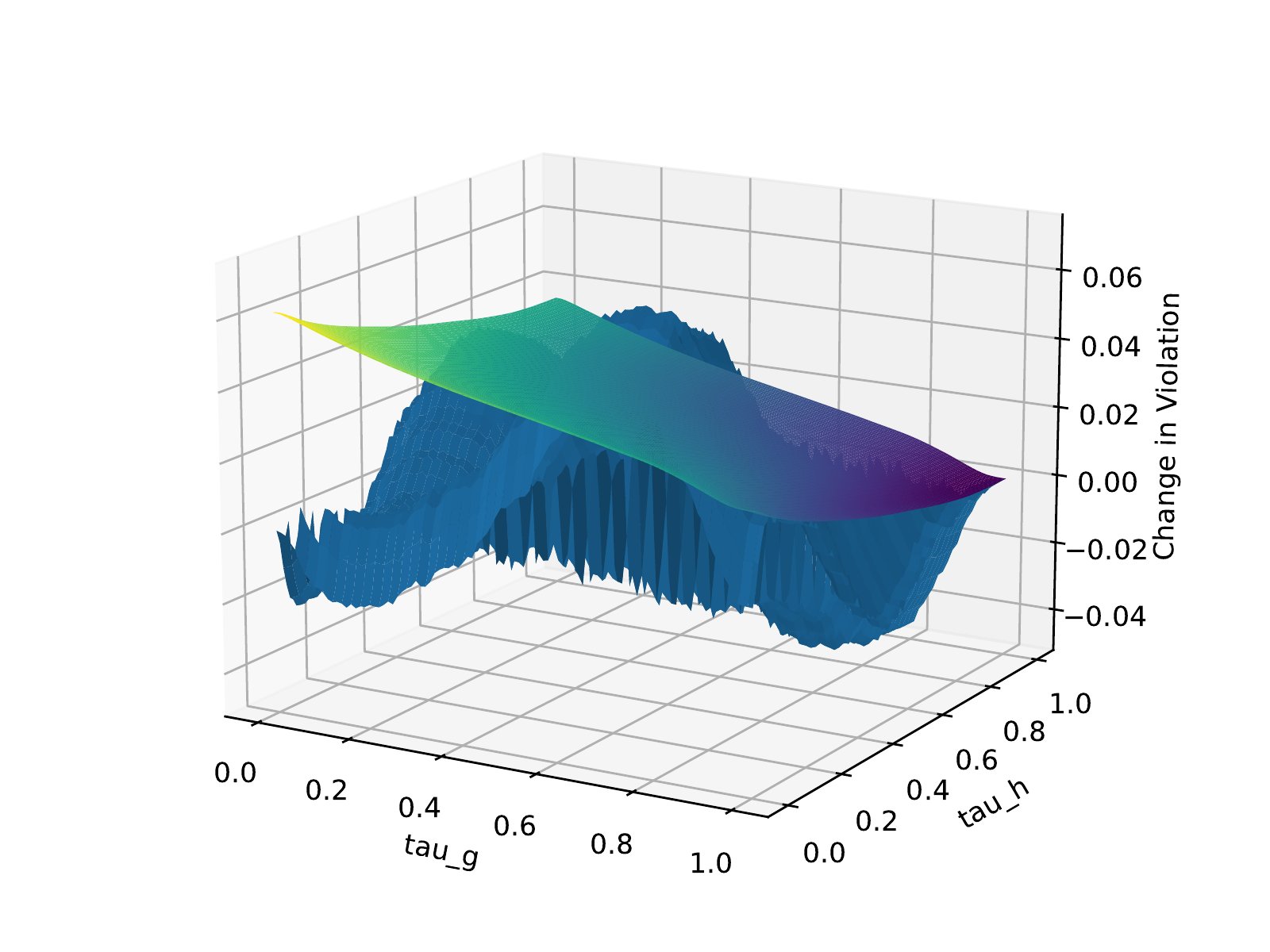}}
    \subfigure[CA: 2014 $\rightarrow$ 2018]{\label{fig:simulation-time-ca}\includegraphics[width=0.24\linewidth]{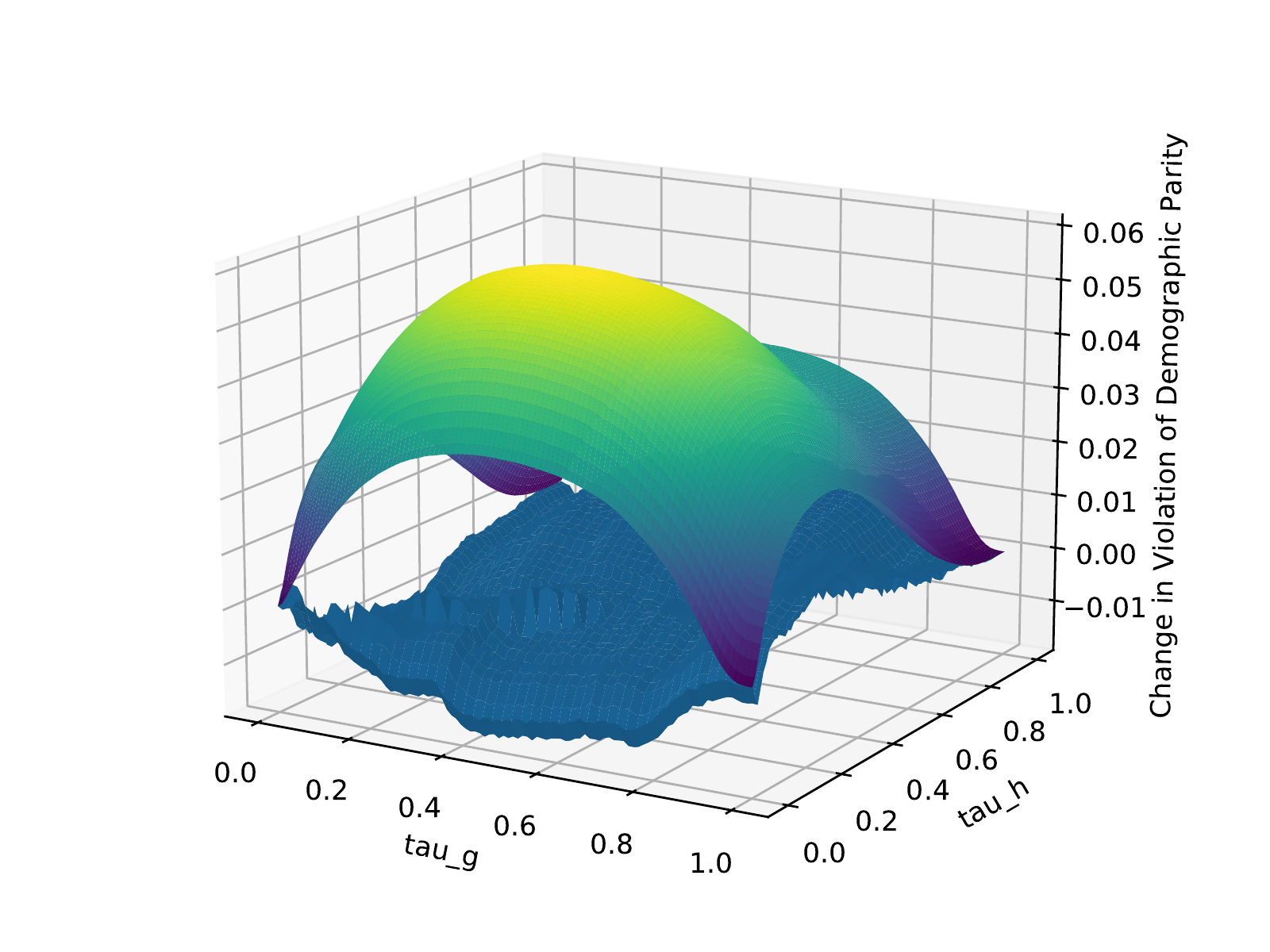}}
    \subfigure[TX: 2014 $\rightarrow$ 2018]{\label{fig:simulation-time-tx}\includegraphics[width=0.24\linewidth]{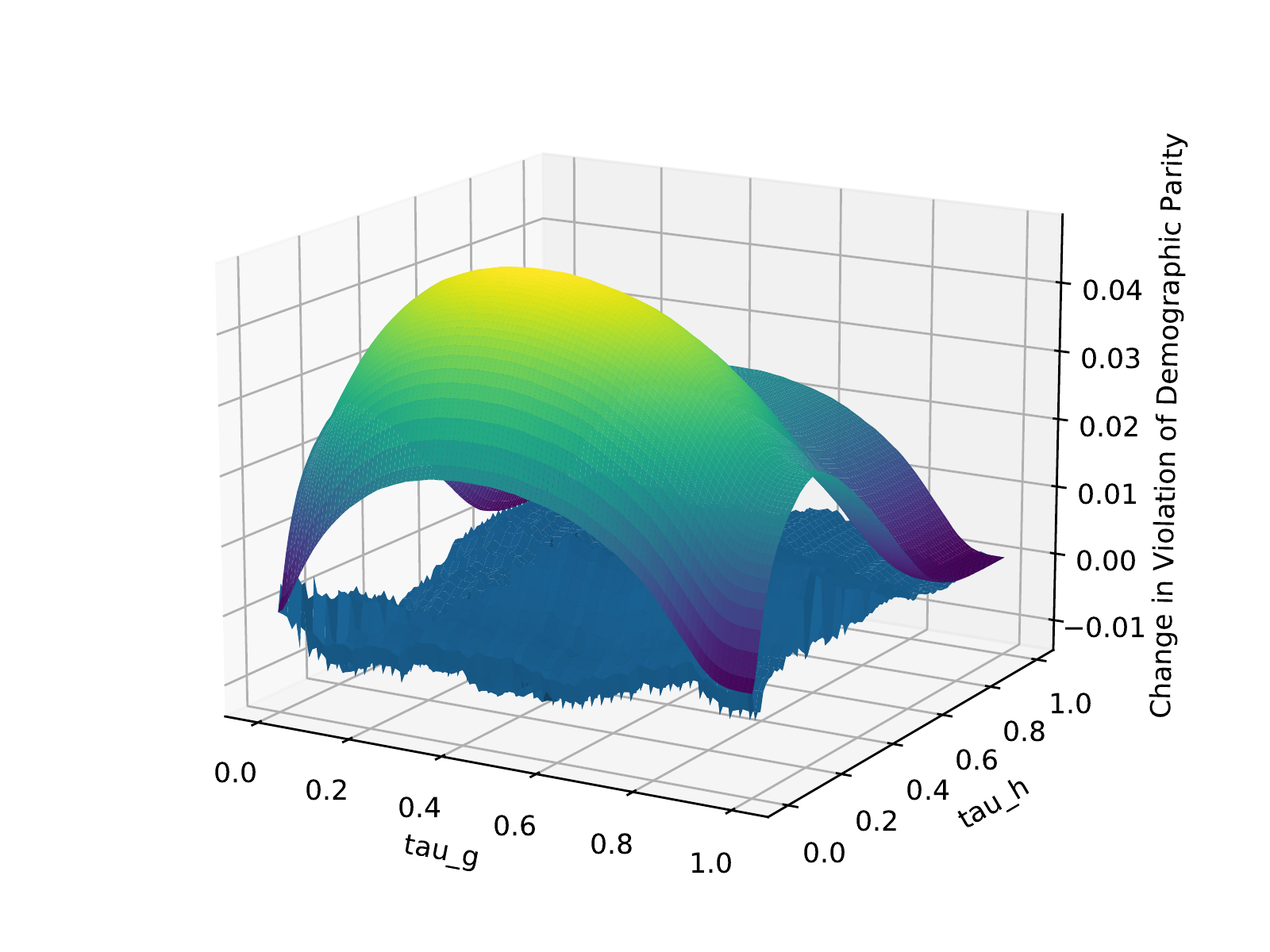}}
    \caption{Simulated change in \DP violation \rev{(blue mesh)} subject to geographic and temporal distribution shifts \vs direct application of bounds for approximately satisfied assumptions (respectively, \cref{thm:dp-covariate} and \cref{thm:dp-label}) \rev{(gradated mesh)}. The $x$-axis and $y$-axis of both figures represent the policy thresholds $\tau_g$ and $\tau_h$.}
    \label{fig:combined_simulation}
    \vspace{-0.1in}
\end{figure}
We graphically compare the theoretical bounds of
\cref{thm:dp-covariate} and \cref{thm:dp-label} for the increased violation of
\DP subject to covariate and label shift, respectively, to the simulated
violations for our model and data in \cref{fig:combined_simulation}. \rev{We provide additional examples and an evaluation of bounds for \EO subject to covariate shift (noting that label shift preserves \EO in theory) in \cref{app:real-world-data}. Despite the fact that geographic or
temporal distribution shifts only approximately satisfy the assumptions of
covariate or label shift, these comparisons demonstrate that our theoretical bounds are not vacuous, approximately bounding the change of fairness violation across real-world domain shifts. For geographic shifts, the covariate shift \EO bounds (\cref{app:real-world-data}) correctly overestimate disparity and tighten near accurate policies, while our \DP bounds are useful only for a subset of policy thresholds (Figure \ref{fig:simulation-state-ca-nv}). {add specific pointer. e.g, 4.a, that one is 4.b} For temporal shift, the label shift bound for \DP correctly overestimates the real change of DP violations but still remains at the same order of magnitude (Figure \ref{fig:simulation-time-ca} and \ref{fig:simulation-time-tx}).}




%% file: Section/Discussion.tex
\vspace{-0.1in}
\section{Conclusion and Discussion}
\vspace{-0.1in}
In this paper, we have developed a unifying framework for bounding the violation
of statistical group fairness guarantees when the underlying distribution
shifts within presupposed bounds. We hope that this work can
generate meaningful discussion regarding the viability of fairness guarantees
subject to distribution shift, the bounds of adversarial attacks against
algorithmic fairness, and evaluations of robustness with respect to algorithmic
fairness.
We believe that, just as published empirical measurements are of limited use without
reported uncertainties, fairness guarantees must be accompanied by bounds on their robustness to
distribution shift.

Future work remains to apply our framework for to problem of fairness transferability in settings with more complicated distribution shift dynamics. For example, compound distribution shifts \cite{schrouff2022maintaining}, which compose covariate shifts and label shifts, cannot be treated by composing the theoretical bounds developed herein without additional information regarding intermediate distributions. Another potential future direction is to develop reasonable bounds on
anticipated distribution shift from models of human behavior and exogenous
pressures.

%% file: Section/appendix.tex
\section{Notation}
\label{secA:notation}
\begin{table}[!ht]
    \centering
    \begin{tabular}{c l}
        \toprule
            Symbol & Usage \\
        \midrule
            $X$ & A random variable representing an example's \emph{features}. \\
            $\mathcal{X}$ & The domain of features $X$. \\
            $Y$ & A random variable representing an example's \emph{ground truth label}. \\
            $\mathcal{Y}$ & The domain of labels $Y$. \\ 
            $\hat{Y}$ & A random variable representing the \emph{predicted label} for an example. \\
            $\hat{\mathcal{Y}}$ & The domain of predicted labels $\hat{Y}$ (distinguished semantically from \(\mathcal{Y}\)). \\
            $G$ & A random variable representing an example's \emph{group} membership. \\
            $\mathcal{G}$ & The domain for group membership $G$. \\
            $\pi$ & A learned (non-deterministic) policy for predicting \(\hat{Y}\) from \(X\) and \(G\). \\
            $\Pr$ & A sample probability (density) according to a referenced distribution. \\
            $\mathcal{P}$ & The space of probability distributions over a given domain. \\
            $\DD$ & The space of distributions of \emph{examples} over \(\mathcal{X} \times \mathcal{Y} \times \mathcal{G}\). \\
            $\OO$ & The space of distributions of \emph{outcomes} over \(\mathcal{X} \times \mathcal{Y} \times \hat{\mathcal{Y}}\). \\
            $\GG$ & The space of distributions of \emph{group-conditioned examples} \(\mathcal{X} \times \mathcal{Y}\). \\
            $\S$ & The \emph{source} distribution in \(\DD\). \\
            $\R$ & The \emph{target} distribution in \(\DD\) to which \(\pi\) is now applied. \\
            $\mathbf{D}$ & A vectorized (by group) premetric for measuring shifts in \(\DD\). \\

            $\mathbf{B, a, b, c}$ & A vector of element-wise bounds for \(\mathbf{D}\). \\
            $\mathbf{e}_g$ & A group-specific basis vector. \\
            $\disp$ & A disparity function, measuring ``unfairness''. \\
            \(\Psi\) & A premetric function (see \cref{def:premetric}) for measuring shifts in \(\OO\). \\
            $v$ & Supremal disparity within bounded distribution shift. \\
            $\DP$ & Abbreviation for Demographic Parity. \\
            $\EO$ & Abbreviation for Equalized Odds. \\
            $\EOp$ & Abbreviation for Equal Opportunity. \\
        \bottomrule
    \end{tabular}
    \caption{Primary Notation}
    \label{tab:notation_table}
  \end{table}

\section{Extended Discussion of Related Work} \label{secA:related-work}

\textbf{Domain Adaptation:}  Prior work has considered the conditions under
which a classifier trained on a source distribution will perform well on a given
target distribution, for example, by deriving bounds on the number of training
examples from the target distribution needed to bound prediction error
\cite{ben-david2010domain, mansour2009domain}, or in conjunction
with the dynamic response of a population to classification
\cite{liu2021model}. We \rr{are interested in} a similar setting and concern, but address the
transferability of \emph{fairness guarantees}, rather than accuracy. In
considering covariate shift and label shift as special cases in this paper, our
work may be paired with studies that address the transferability of prediction
accuracy under such assumptions \cite{shimodaira2000improving,
sugiyama2008direct, zhang2013domain}.

\textbf{Algorithmic Fairness:} Many formulations of fairness have been
proposed for the analysis of machine learning policies. When it is appropriate
to ignore the specific social and dynamical context of a deployed policy, the statistical regularity of policy outcomes may be considered across individual
examples \cite{dwork2012fairness} and across groups \cite{zemel2013learning,
feldman2015certifying, corbett2017algorithmic,hardt2016equality, chouldechova2017fair}.
In our paper, we focus on such statistical definitions of fairness between
groups, and develop bounds for demographic parity \cite{chouldechova2017fair}
and equalized odds \cite{hardt2016equality} as specific examples.

\textbf{Dynamic Modeling:} When the dynamical context of a deployed policy must be accounted for, such as when the policy influences control over the future
trajectories of a distribution of features and labels, we benefit from modelling how populations respond to classification. \rev{Among this line of work, \cite{liu2018delayed} initiate the discussion of the long-term effect of imposing static fairness constraints on a dynamic social system, highlighting the importance of measurement and temporal modeling in the evaluation of fairness criteria.}
However, developing such models
remains a challenging problem
\cite{creager2020causal, Singh2021fairness,rezaei2021robust, d2020fairness,
zhang2020fair,wen2019fairness, liu2019disparate, coate1993will,
hu2018short,mouzannar2019fair, raab2021unintended}.
\rev{In particular, \cite{creager2020causal} discuss causal directed acyclic graphs (DAGs) as a unifying framework on fairness in dynamical systems.}
In this work, rather than relying precise models of distribution shift to
quantify the transferability of fairness guarantees in dynamical contexts, we
assume a bound on the difference between source and target distributions. We
thus develop bounds on realized statistical group disparity while remaining
agnostic to the specific dynamics of the system.

\section{Additional Figures} \label{secA:additional-figures}

 \begin{figure}[H]
   \centering
   \includegraphics[width=.5\linewidth]{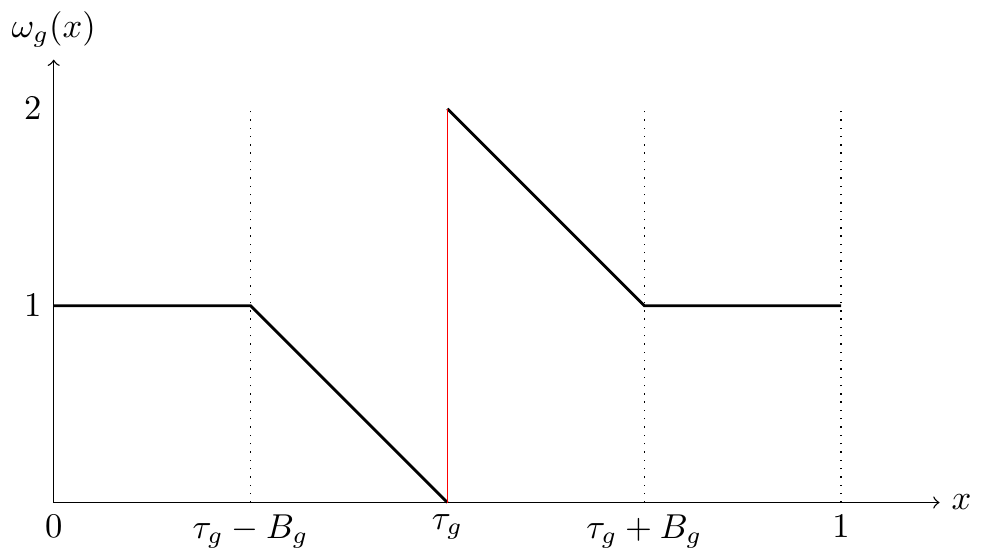}
   \caption{Distribution of the reweighting coefficient $w_g(x)$ for the setting of Covariate shift via Strategic Response.}
 \label{fig:sc-distribution-shift}
\end{figure}

\section{A Geometric Interpretation} \label{app:geometric}
In this extension of \cref{sec:covariateEO}, we fulfill the promise of
\cref{sec:metricbounds} and consider a case in which shared structure of between
\(\Psi \colon \OO^2 \to \RR \) and each \(D_{g} \colon \GG^2 \to \RR \) permits
a geometric interpretation of distribution shift for Equal Opportunity \(\EOp\), building on \cref{thm:cov-EO-superbound}. We continue to defer rigorous proof to \cref{app:proofs}.

We first recall the definition of the true positive rate of policy \(\pi\), for each group, on distribution \(\mathcal{T}\).
\begin{equation}
  \beta^{+}_{g} \define \Pr_{\pi, \R}( \hat{Y} {=} 1 \mid Y {=} 1, G {=} g )
\end{equation}
The true positive rate may be expressed as a ratio of inner products
defined over the space of square-integrable \(L^{2}\) functions on
\(\mathcal{X}\).\footnote{This precludes distributions with non-zero
probability mass concentrated at singular points.}
\begin{align}
\label{eq:ratio-inner-products}
  \beta^{+}_{g}[\R]
  &= \frac{\Pr_{\R}(\hat{Y}{=}1, Y{=}1 \mid G{=}g)}{\Pr_{\R}(Y{=}1 \mid G{=}g)}
  = \frac{\big\langle \mathsf{r}_{g}[\R], \mathsf{t}_{g} \big\rangle_g}{\big\langle \mathsf{r}_{g}[\R], \mathsf{1} \big\rangle_g}
  \\
  \langle a, b \rangle_g &\define \int_{\mathcal{X}} \label{eq:inner-product}
  a(x) b(x) \mathsf{s}_g(x) \dd{x}
\end{align}
where we use the shorthands
\begin{align}
    \mathsf{r}_{g}[\R](x) &\define \Pr_{\R}(X {=} x \mid G {=} g)
\\\mathsf{s}_{g}(x) &\define \Pr_{\S} (Y {=} 1 \mid X {=} x, G {=} g)
\\ \mathsf{1}(x) &\define 1
\\ \mathsf{t}_{g}(x) &\define \Pr_{\pi}(\hat{Y} {=} 1 \mid Y {=} 1, X=x, G=g)
\end{align}
and assume that \(\mathsf{s}_g(x) > 0\) for all \(x\) and \(g\).

We observe that the only degree of freedom in $\beta^{+}_{g}$ 
as \(\mathcal{\R}\) varies subject to covariate shift is \(\mathsf{r}_g\): by the covariate
assumption,
\(\mathsf{s}_g\) is fixed; \(\mathsf{t}\) meanwhile remains independent of
\(\R\) for fixed policy \(\pi\), since \(\pi\) is independent of \(Y\)
conditioned on \(X\) and \(G\).

\paragraph{Selection of \(\mathbf{D}\)}
We now select each \(D_{g}\) to be the standard metric for the inner product
defined by \cref{eq:inner-product}, where, for each group, distributions in
\(\GG\) are mapped to the corresponding vector \(\mathsf{r}_{g}\):
\begin{equation}
\begin{aligned}
  & D_{g}\big(
    \Pr_{\R}(X, Y \mid G {=} g)
    \parallel
    \Pr_{\S}(X, Y \mid G {=} g)
    \big) \\
  &\qquad\qquad\qquad \define
    \sqrt{
    \langle \mathsf{r}_{g}[\S], \mathsf{r}_{g}[\R] \rangle_{g}
    + \langle \mathsf{r}_{g}[\R], \mathsf{r}_{g}[\R] \rangle_{g}
    - 2 \langle \mathsf{r}_{g}[\S], \mathsf{r}_{g}[\R] \rangle_{g}
  }
\end{aligned}
\end{equation}
In this geometric picture, \(\mathbf{D}(\R \parallel \S) \preceq \mathbf{B}\) implies that all
possible values for \(\mathsf{r}_{g}[\R]\) lie within a ball of radius \(B_g\)
centered at \(\mathsf{r}_{g}[\S]\). By the normalization condition of a
probablity (density) function,
denoting \(\mathsf{s}_{g}^{-1}(x) \define (\mathsf{s}_{g}(x))^{-1}\),
the vector \(\mathsf{r}_{g}[\R]\) must also lie on the hyperplane
\begin{align}
\label{eq:normalized-hyperplane}
    \int_{\mathcal{X}} \mathsf{r}_g[\R] \dd{x} =
    \langle \mathsf{r}_g[\R], \mathsf{s}_g^{-1} \rangle_g = 1
\end{align}
Recalling \cref{eq:ratio-inner-products}, the group-specific true positive rate \(\beta^{+}_{g}[\R]\) for policy \(\pi\)
is given by a ratio of the projected distances of \(\mathsf{r}_{g}\) along the \(\mathsf{t}_g\) and
\(\mathsf{1}\) vectors.  Let us therefore denote the projection of \(\mathsf{r}_{g}[\R]\) onto the
\((\mathsf{1}, \mathsf{t}_{g})\)-plane as 
\(\mathsf{r}_{g}^{\perp}[\R]\). We may then consider
the possible values of \(\mathsf{r}_{g}^{\perp}[\R]\) as projections from the intersection of the \(\mathsf{r}_{g}[\S]\)-centered
hypersphere of radius \(B_{g}\) and the hyperplane of normalized distributions (\cref{eq:normalized-hyperplane}).
Using \(\angle(\cdot, \cdot)\) to denote the angle between vectors and
denoting
\(\phi'_g \define \angle(\mathsf{r}_g, \mathsf{t}_g)\),
\(\theta'_g \define \angle(\mathsf{r}_g, \mathsf{1})\),
\(\phi_g \define \angle(\mathsf{r}_g^{\perp}, \mathsf{t}_g)\), and
\(\theta_g \define \angle(\mathsf{r}_g^{\perp}, \mathsf{1})\), we appeal to the geometric relationship \(\langle a, b \rangle = \cos(\angle(a, b)) \|a\|\|b\|\) to write
\begin{align} \label{eq:tpr-cos-cos}
\beta^{+}_{g}
\frac{\|\mathsf{1}\|}{\|\mathsf{t}_g\|}
  =
  \frac{\cos{\phi'_g}}
       {\cos{\theta'_g}}
    =
  \frac{\cos{\phi_g}}
       {\cos{\theta_g}}
\end{align}
From
these observations, we need only bound the ratio between \(\cos(\phi_g)\) and
\(\cos(\theta_g)\) to bound \(\beta^{+}_{g}\). Relating these angles in the \((\mathsf{1}, \mathsf{t}_{g})\)-plane by
\(\phi_g = \xi_g - \theta_g\) where \(\xi_g \define \angle(\mathsf{t}_g, \mathsf{1})\),
we arrive at the following theorem:
\begin{theorem} \label{thm:geometric-beta}
The true positive rate \(\beta^{+}_{g}\) is bounded over the domain of covariate
shift \(\DD_{\text{cov}}[\mathbf{B}]\), which we define by the bound
\(\mathbf{D}(\R \parallel \S) \preceq \mathbf{B}\), and the invariance of
\(\Pr(Y {=} 1 \mid X {=} x, G {=} g)\) for all \(x, g\), as
\begin{equation}\label{eq:bound-TPR-cos}
  \frac{\cos(\phi_g^u)}{\cos(\xi_g - \phi_g^u)} \leq
\frac{\|\mathsf{1}\|}{\|\mathsf{t}_g\|}
  \beta^{+}_{g}(\pi, \R) \leq \frac{\cos (\phi_g^l)}{\cos(\xi_g - \phi_g^l)}
\end{equation}
with upper (\(\phi_g^u\)) and lower (\(\phi_g^l\)) bounds for \(\phi_g\) represented as
\begin{equation}
\phi_g^l \define \min_{\R \in \DD_\text{cov}[\mathbf{B}]} \phi_g; \quad
\phi_g^u \define \max_{\R \in \DD_\text{cov}[\mathbf{B}]} \phi_g
\end{equation}
\end{theorem}
\vspace{-0.1in}
We obtain a final bound on \(\disp_{\EOp}\) by substituting
\cref{eq:bound-TPR-cos} into \cref{eq:cov-eop-bound}.
We visualize the geometric bound on \(\beta^{+}_{g}\)
(\cref{thm:geometric-beta}) in \cref{fig:geometric}.
In \cref{sec:comparison-to-model}, we apply this bound to real-world credit score
data assuming the model of strategic manipulation given in
\cref{sec:strategic-response}. Although the result is not an easily interpretted
formula, it provides a demonstration of geometric reasoning applied to
statistical fairness guarantees.

Finally, we note that, in addition to the constraints considered above, each
vector \(\mathsf{r}_g\) is subject to the positivity condition,
\(\forall x \in \mathcal{X}, \mathsf{r}_{g}(x) \geq 0\).
The bound developed in this section, however, does not benefit from this additional
constraint; we leave this to potential future work.

\begin{figure}[H]
  \centering
    \subfigure[]{\includegraphics[width=0.3\textwidth]{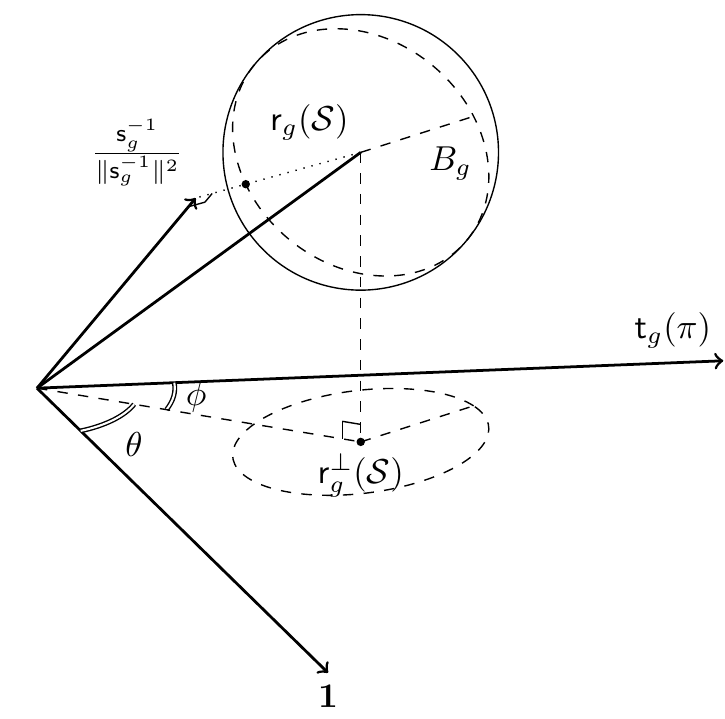}}
    \subfigure[]{\includegraphics[width=0.3\textwidth]{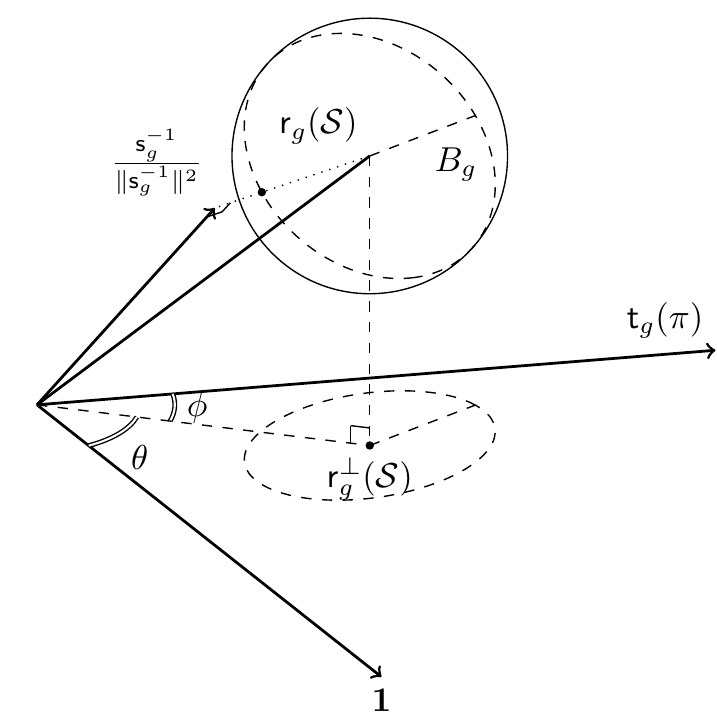}}
    \subfigure[]{\includegraphics[width=0.3\textwidth]{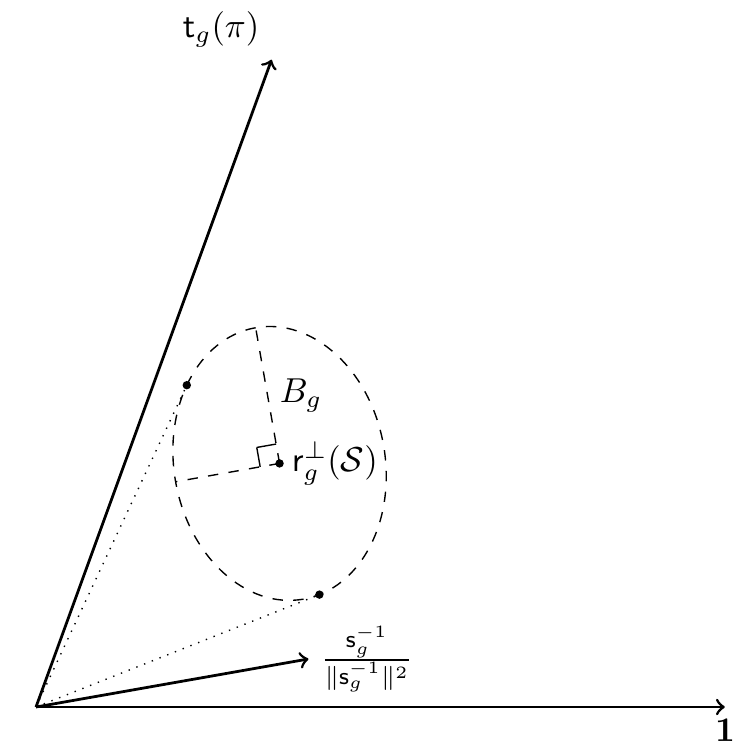}}
  \caption{A geometric bound in an infinite-dimensional vector space (\ie, a Hilbert space), represented with a stereoscopic (cross-eye) view in three dimensions (to provide intuition) and an examination of the \((\mathsf{t}_{g}, \mathsf{1})\)-plane. The extreme values of \(\beta^{+}_{g}\) correspond to
the extremal angles of \(\phi\) and \(\theta\). In this
figure, the vector displayed parallel to \(\mathsf{s}_g^{-1}\) from the origin
terminates on the hyperplane of normalized distributions.}
 \label{fig:geometric}
\end{figure}

\section{Empirical Evaluations of the Bounds}\label{sec:empirical-evals}

\subsection{Comparisons to Dynamic Models of Distribution Shift} \label{sec:comparison-to-model}

\begin{figure}[H]
  \centering
    \includegraphics[width=0.85\textwidth]{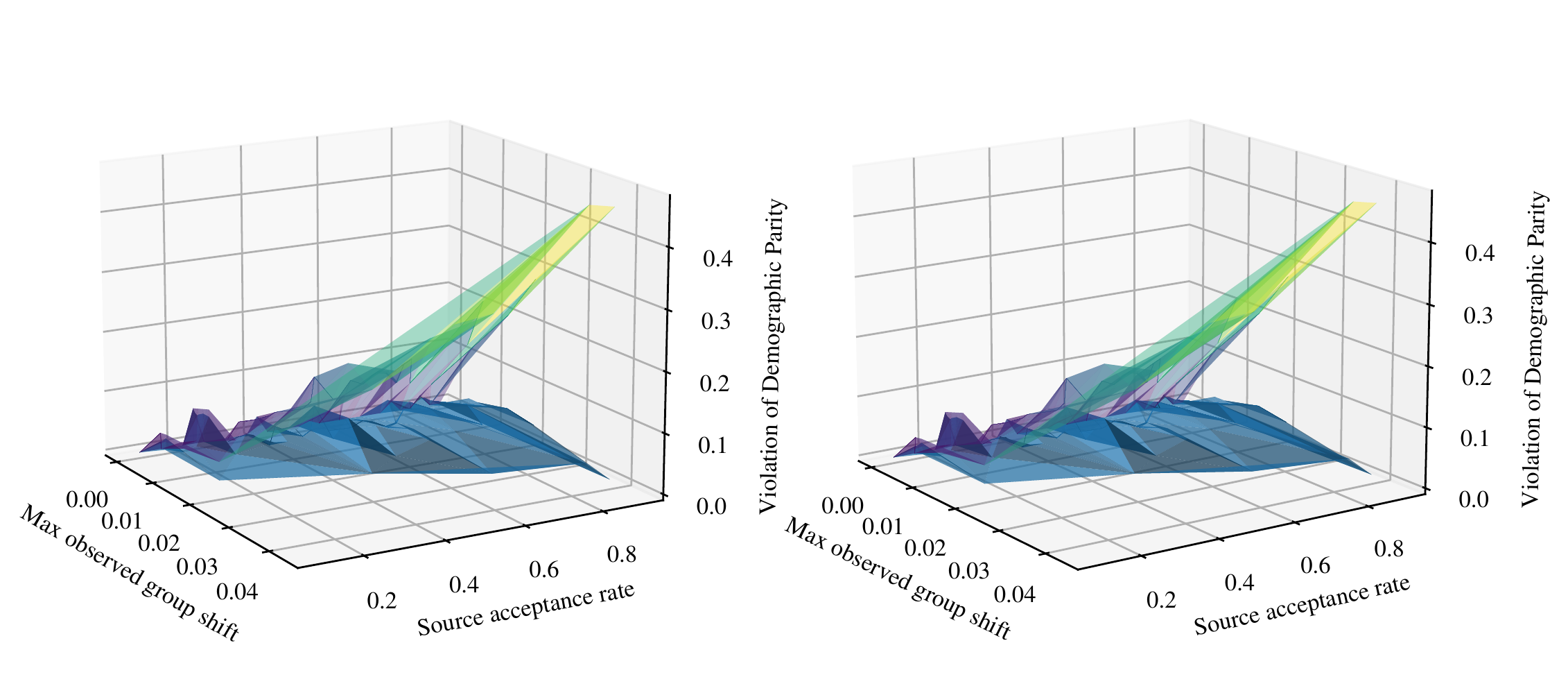}
  \caption{A stereoscopic (cross-eye view) comparison between the bound of \cref{sec:covariateDP} (gradated) and simulated results for the model of \cref{sec:strategic-response} (blue) in response to a \DP-fair classifier with different initial group-independent acceptance rates. The \(x\)-axis represents the maximum shift \(D_g\) over all groups \(g\) in response to the classifier.}
 \label{fig:DP-covariate-experimental}
\end{figure}
\vspace{-1em}

\begin{figure}[H]
  \centering
    \includegraphics[width=0.85\textwidth]{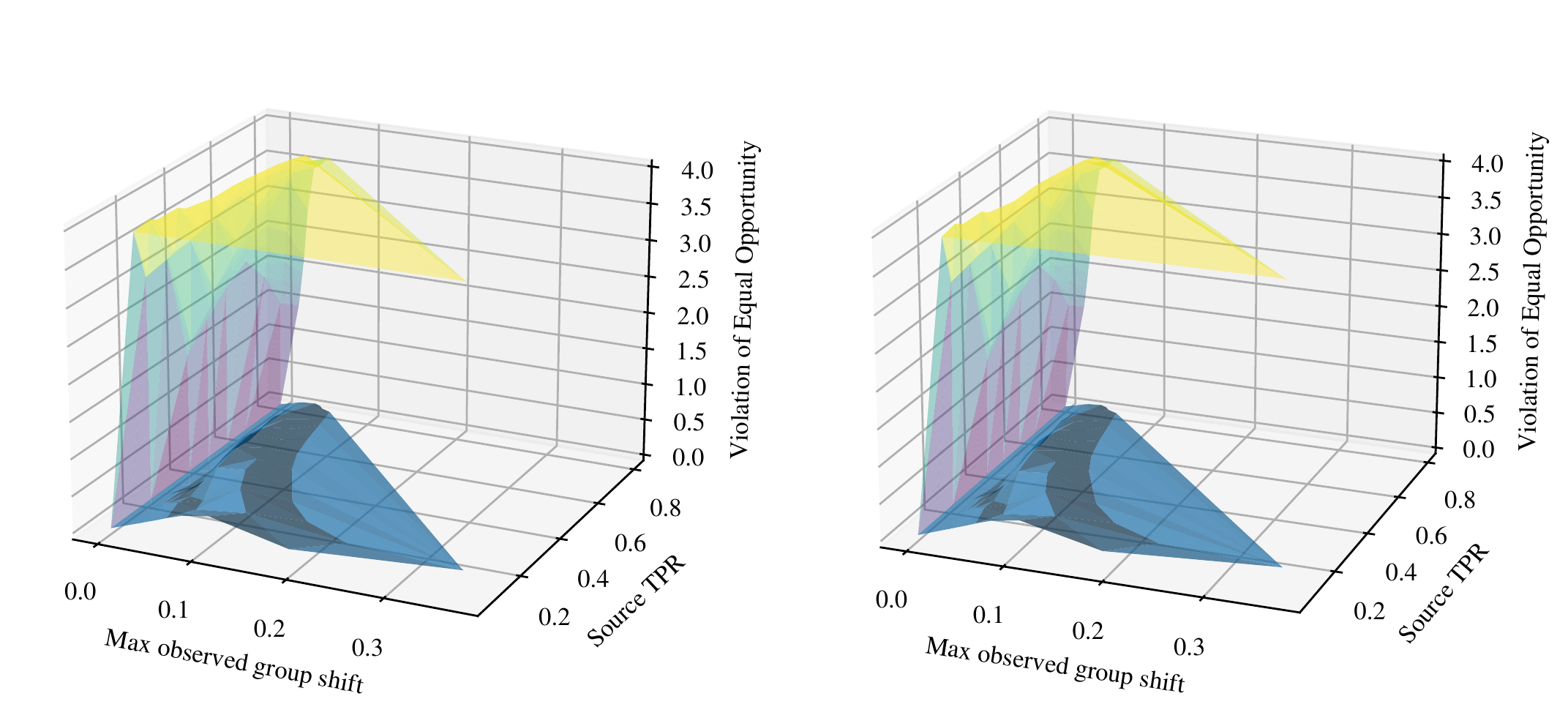}
  \caption{A stereoscopic (cross-eye view) comparison between the theoretical bound of \cref{sec:covariateEO} (gradated) and simulated results for the model of \cref{sec:strategic-response} (blue) in response to a \EOp-fair classifier with different initial group-independent true positive rates (TPR). The \(x\)-axis represents the maximum shift \(D_g\) over all groups \(g\) in response to the classifier. As \cref{cor:max-cov-EO} limits the maximum possible value of \EO violation, we include this limit as part of the bound.}
 \label{fig:geometric-experimental}
\end{figure}
\vspace{-1em}

\begin{figure}[H]
  \centering
      \includegraphics[width=0.85\textwidth]{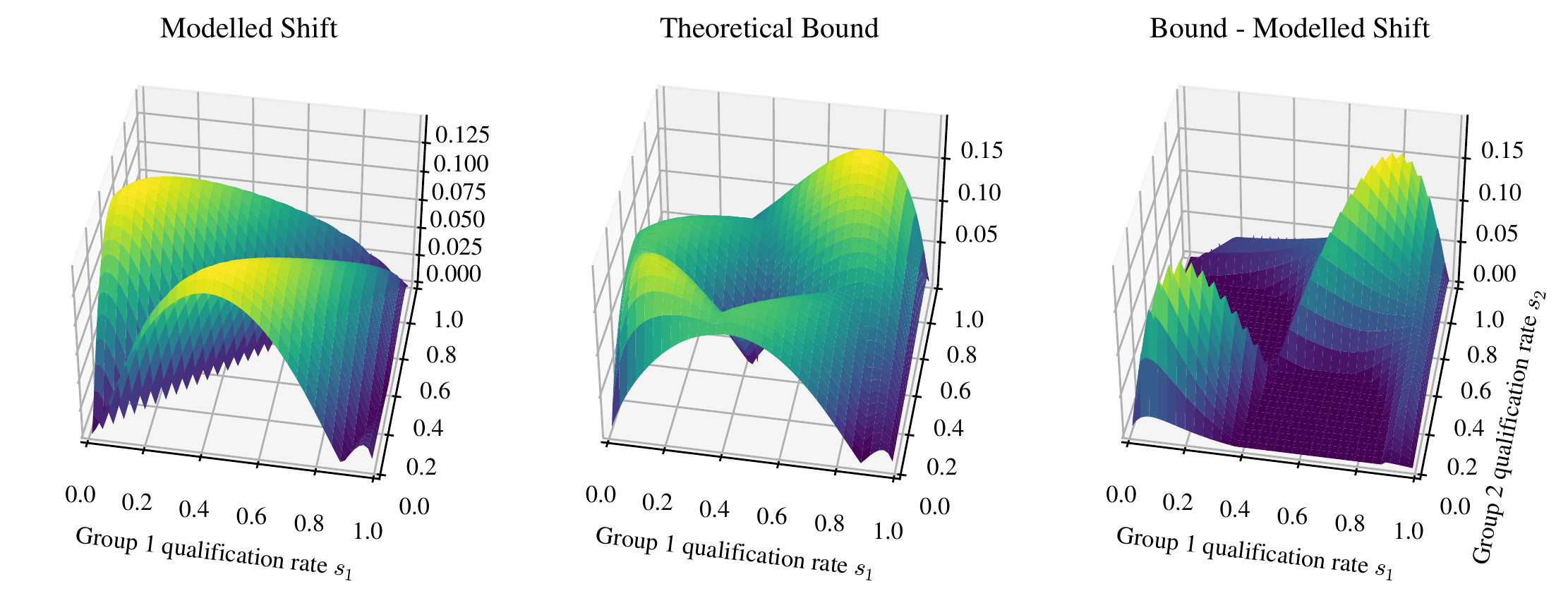}
\caption{A policy satisfying \DP is subject to distribution shift prescribed by replicator dynamics (\cref{sec:replicator}). Realized disparity increases (blue) are compared to the theoretical bound (\cref{thm:dp-label}, gradated), which is tight when group have dissimilar qualification rates.}
 \label{fig:replicator-transfer-large}
\end{figure}

\subsection{Comparisons to Real-World Data}\label{app:real-world-data}

We provide additional graphics comparing bounds on demographic parity or equal opportunity to real-world distribution shifts. \cref{fig:covariate_shift_CA_A} compares the covariate shift bound of \cref{thm:dp-covariate} to the violation of demographic parity for hypothetical policies trained on one US state and deployed in another. \cref{fig:label_shift_temporal_CA_A} compares the label shift bound of \cref{thm:dp-label}
to the violation of demographic parity for hypothetical policies trained for a US state in 2014 and deployed in 2018.
\cref{fig:covariate_EO_empirical} compares the covariate shift bound of \cref{thm:cov-EO-superbound} with \cref{thm:geometric-beta} to the violation of equal opportunity for hypothetical policies trained on one US state and deployed in another.

\begin{figure}[H]
    \centering
    \subfigure[CA $\longrightarrow$ AZ]{\includegraphics[width=0.23\linewidth]{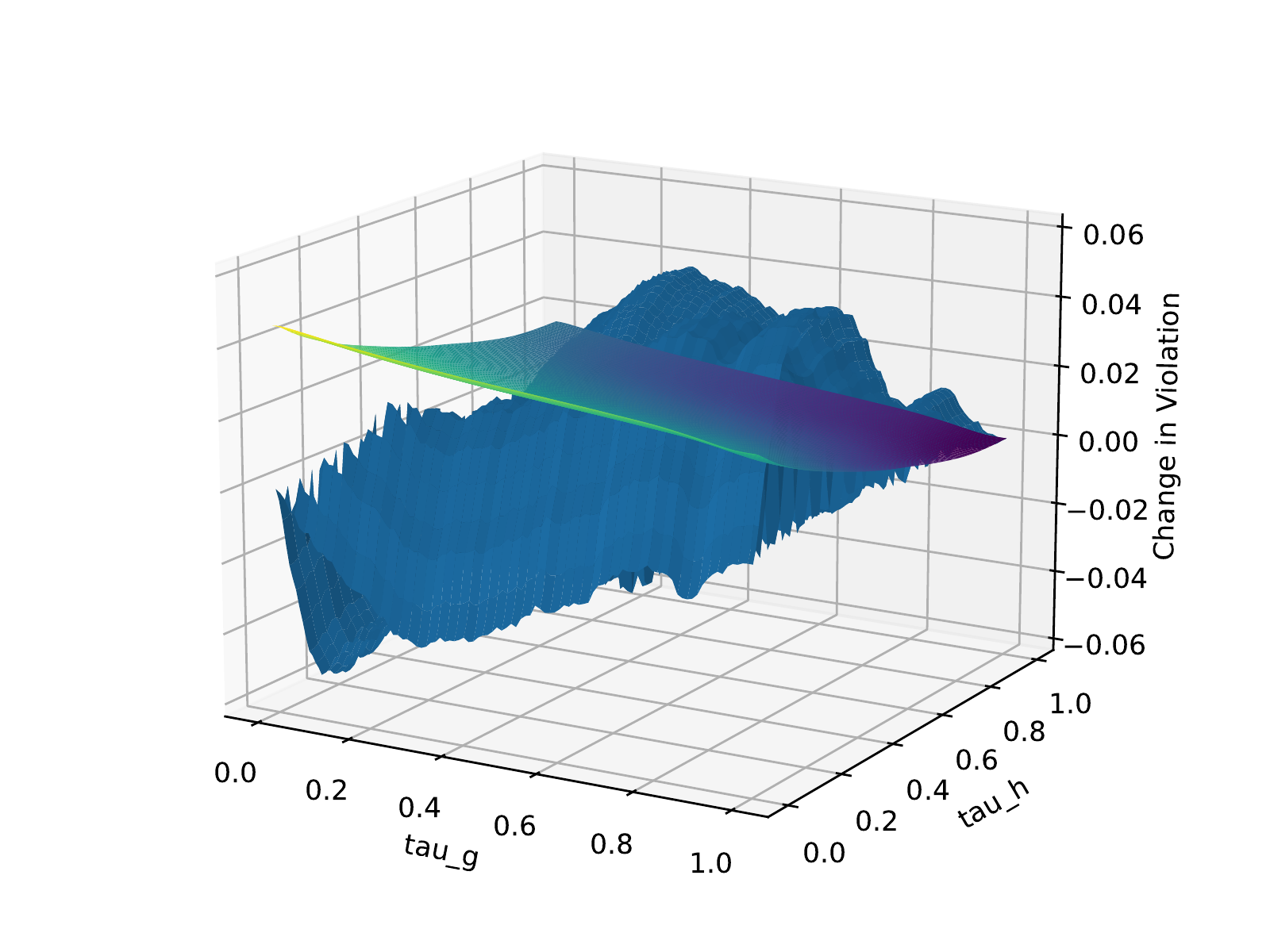}}
    \subfigure[CA $\longrightarrow$ CO]{\includegraphics[width=0.23\linewidth]{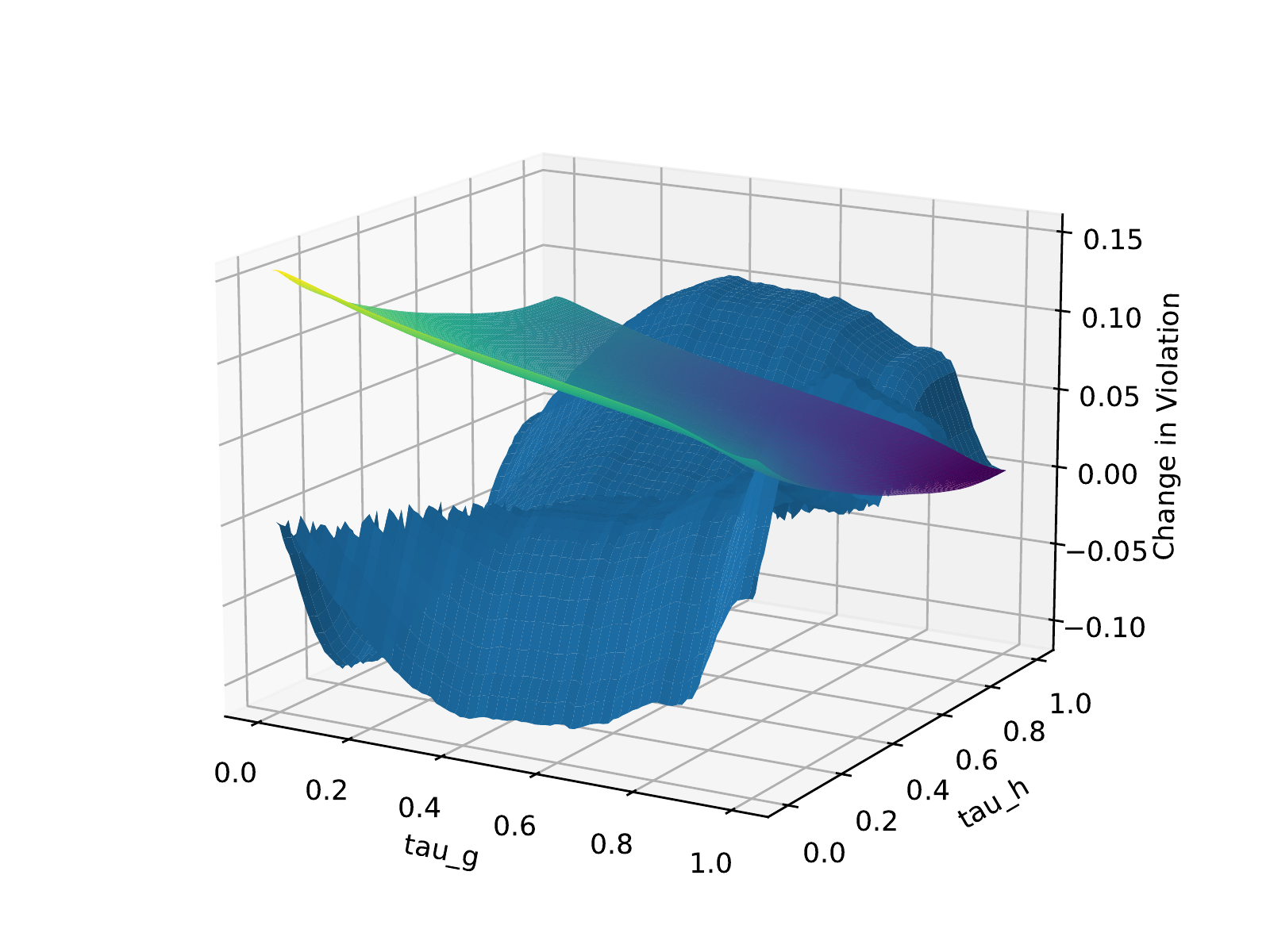}}
    \subfigure[CA $\longrightarrow$ FL]{\includegraphics[width=0.23\linewidth]{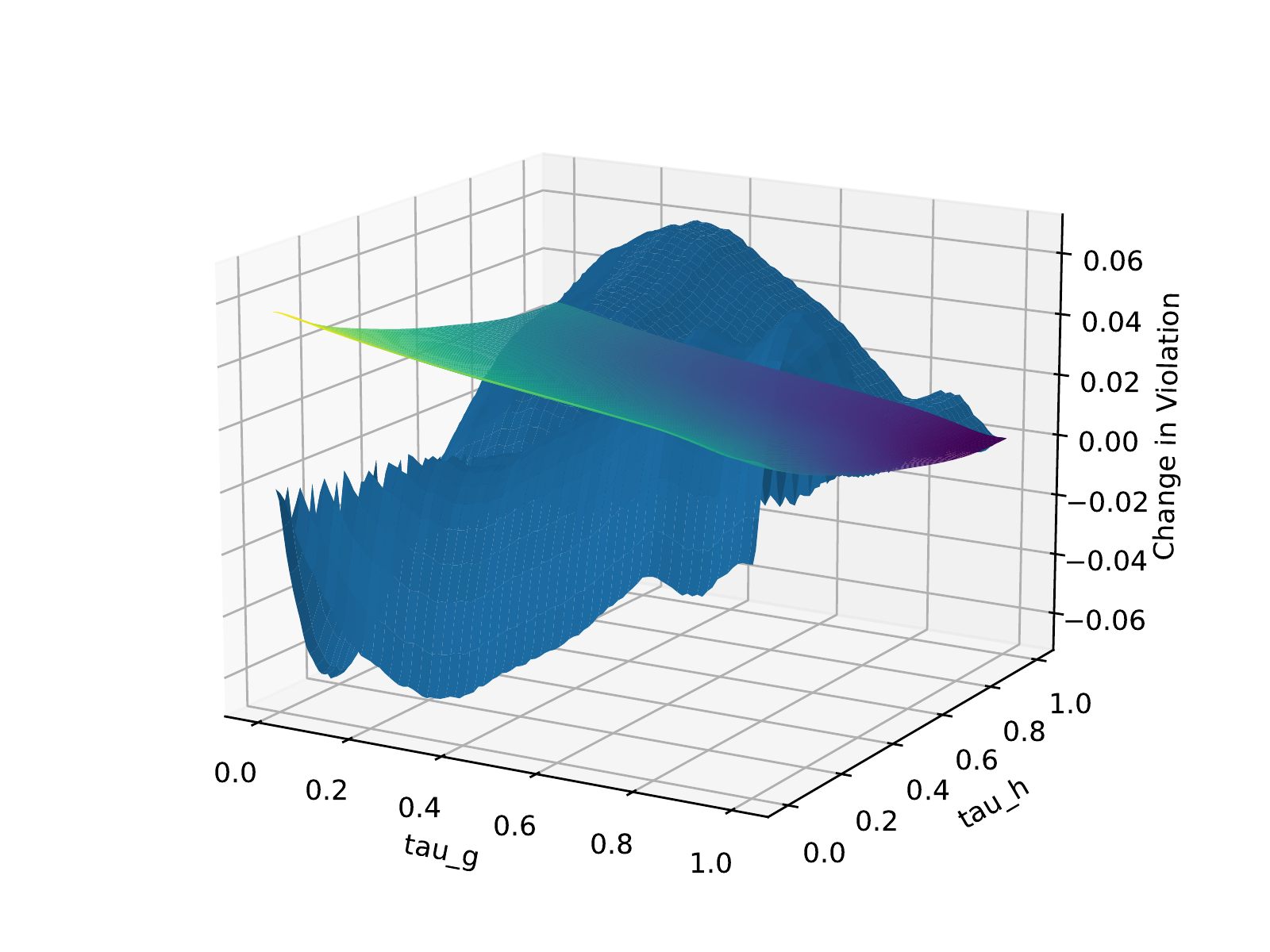}}
    \subfigure[CA $\longrightarrow$ MD]{\includegraphics[width=0.23\linewidth]{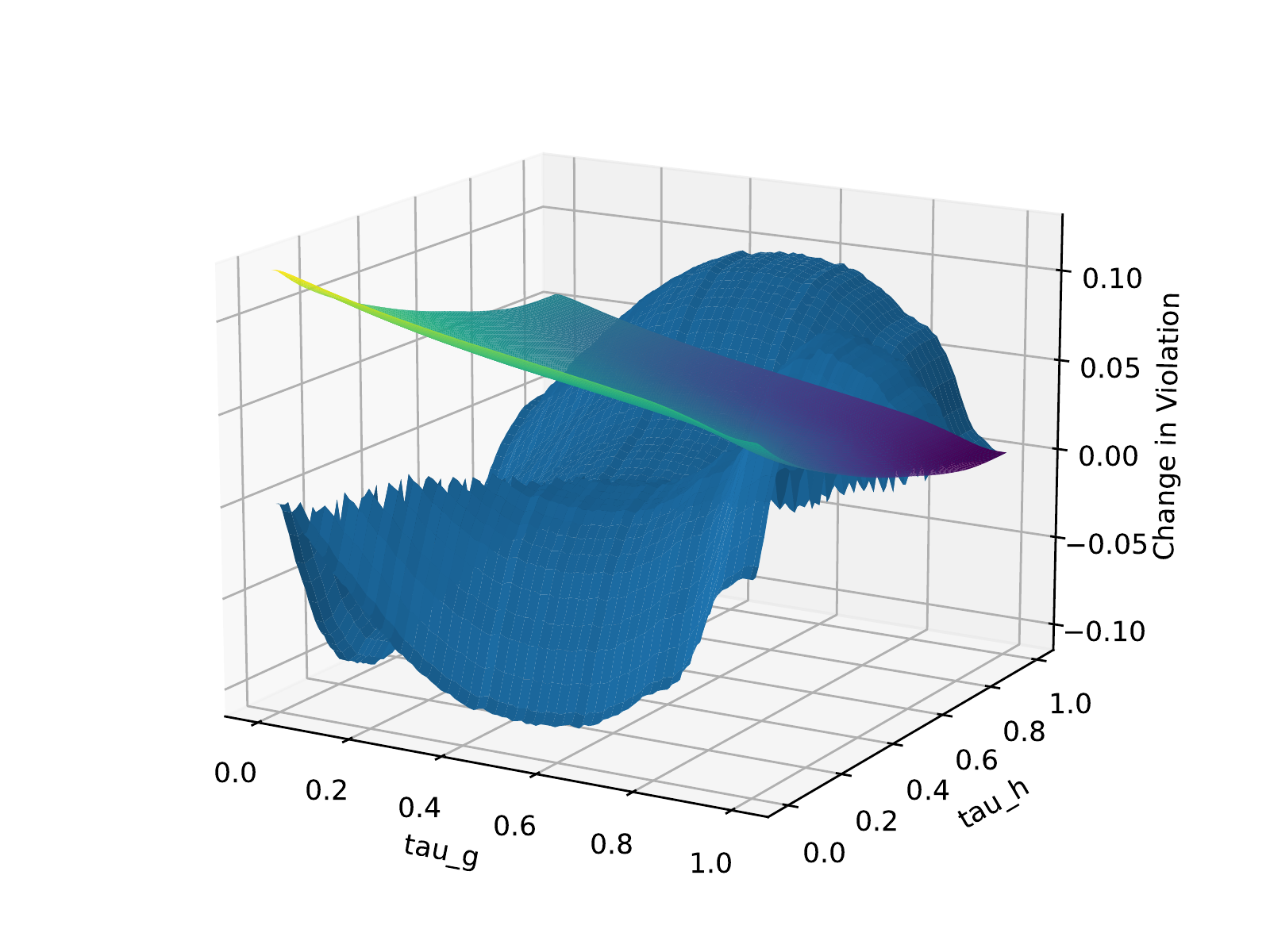}}
    \subfigure[CA $\longrightarrow$ OR]{\includegraphics[width=0.23\linewidth]{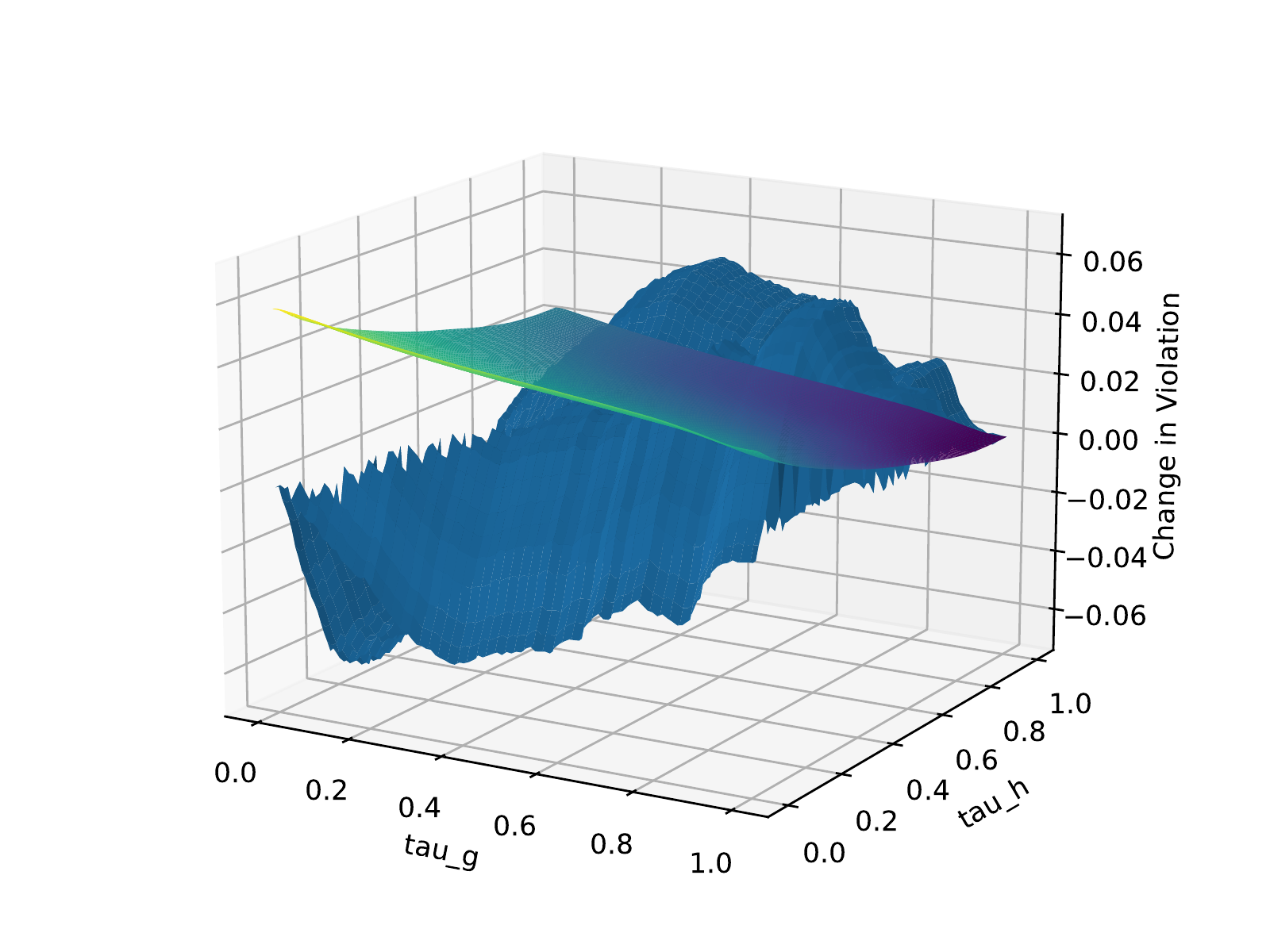}}
    \subfigure[CA $\longrightarrow$ UT]{\includegraphics[width=0.23\linewidth]{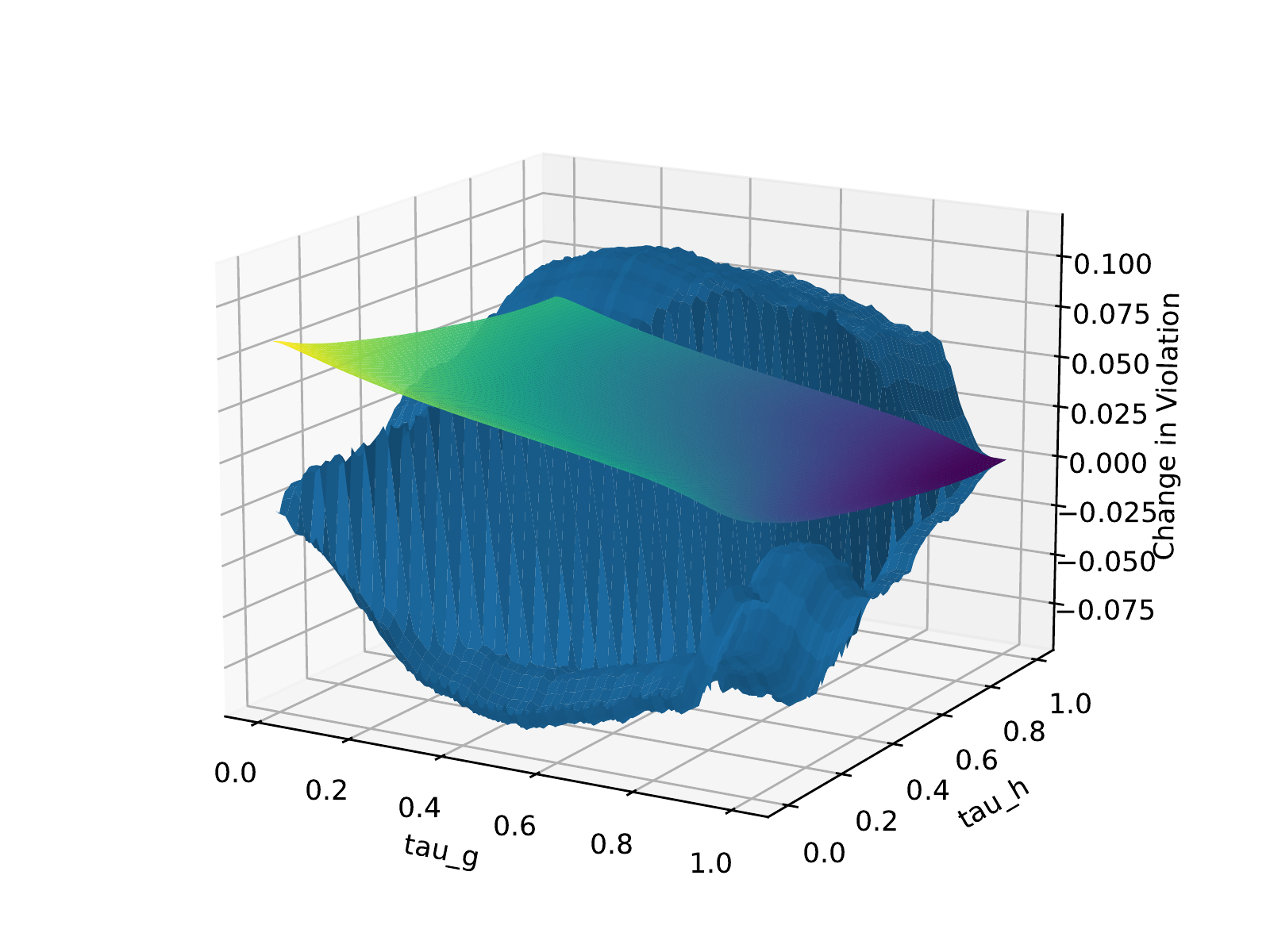}}
    \subfigure[CA $\longrightarrow$ VT]{\includegraphics[width=0.23\linewidth]{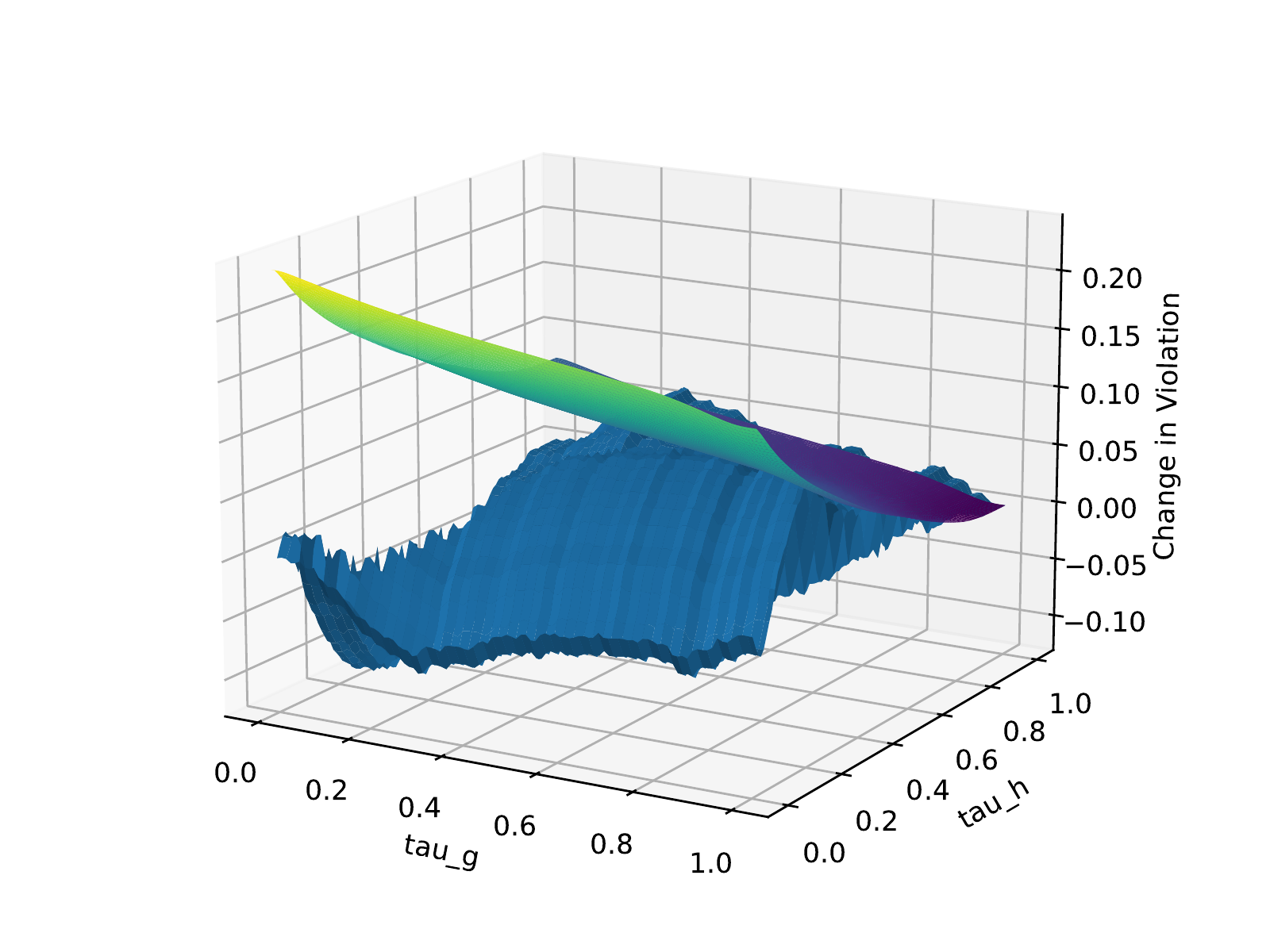}}
    \subfigure[CA $\longrightarrow$ WI]{\includegraphics[width=0.23\linewidth]{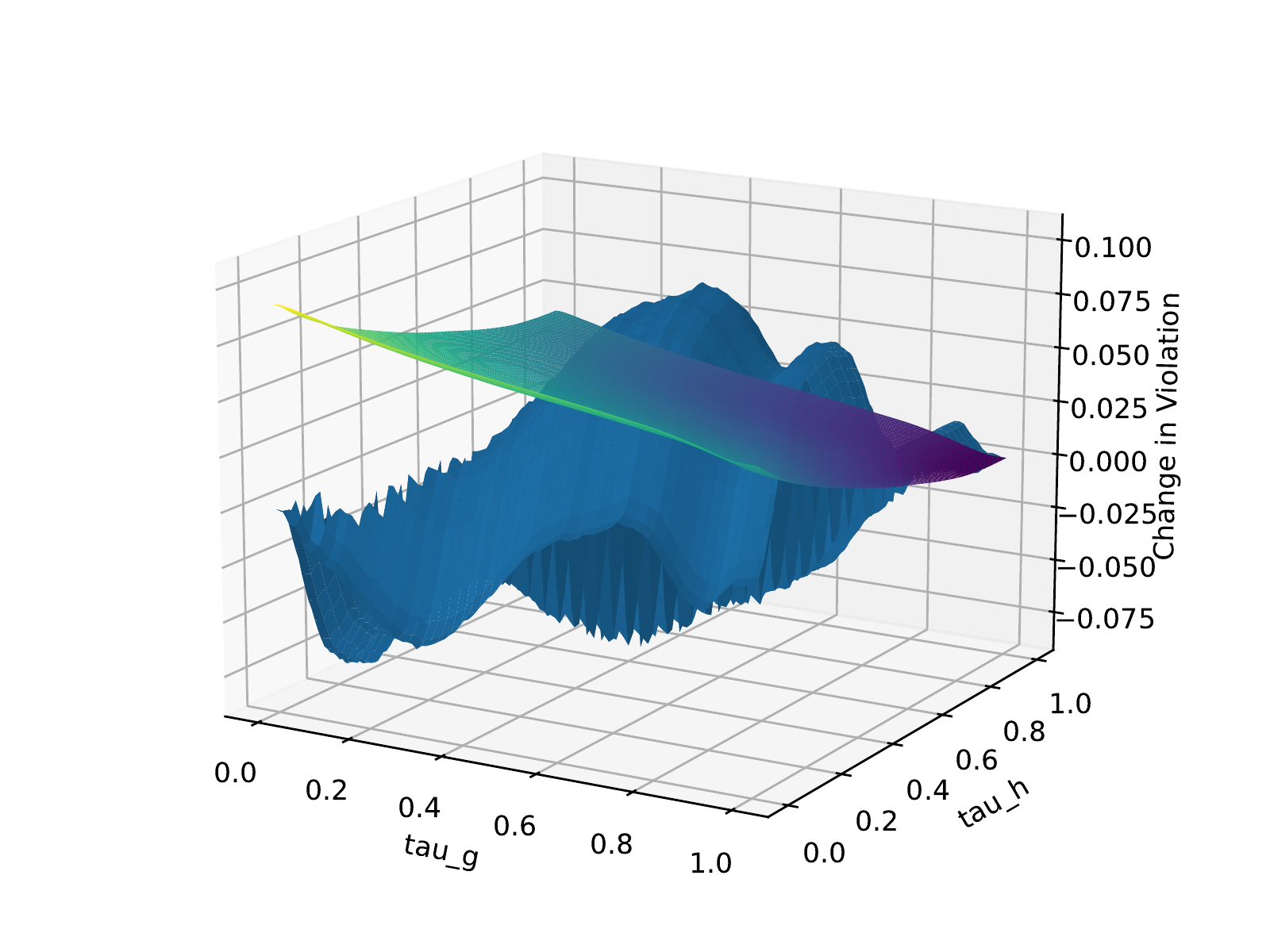}}
    \caption{Change in violation of demographic parity for hypothetical policies trained on one US state's data and reused for another state (blue) compared to covariate-shift bounds (\cref{thm:dp-covariate}, gradated). The $x$-axis and $y$-axis represent the thresholds $\tau_g$ and $\tau_h$, respectively.}
    \label{fig:covariate_shift_CA_A}
\end{figure}

\begin{figure}[H]
    \centering
    \subfigure[Arizona (AZ)]{\includegraphics[width=0.23\linewidth]{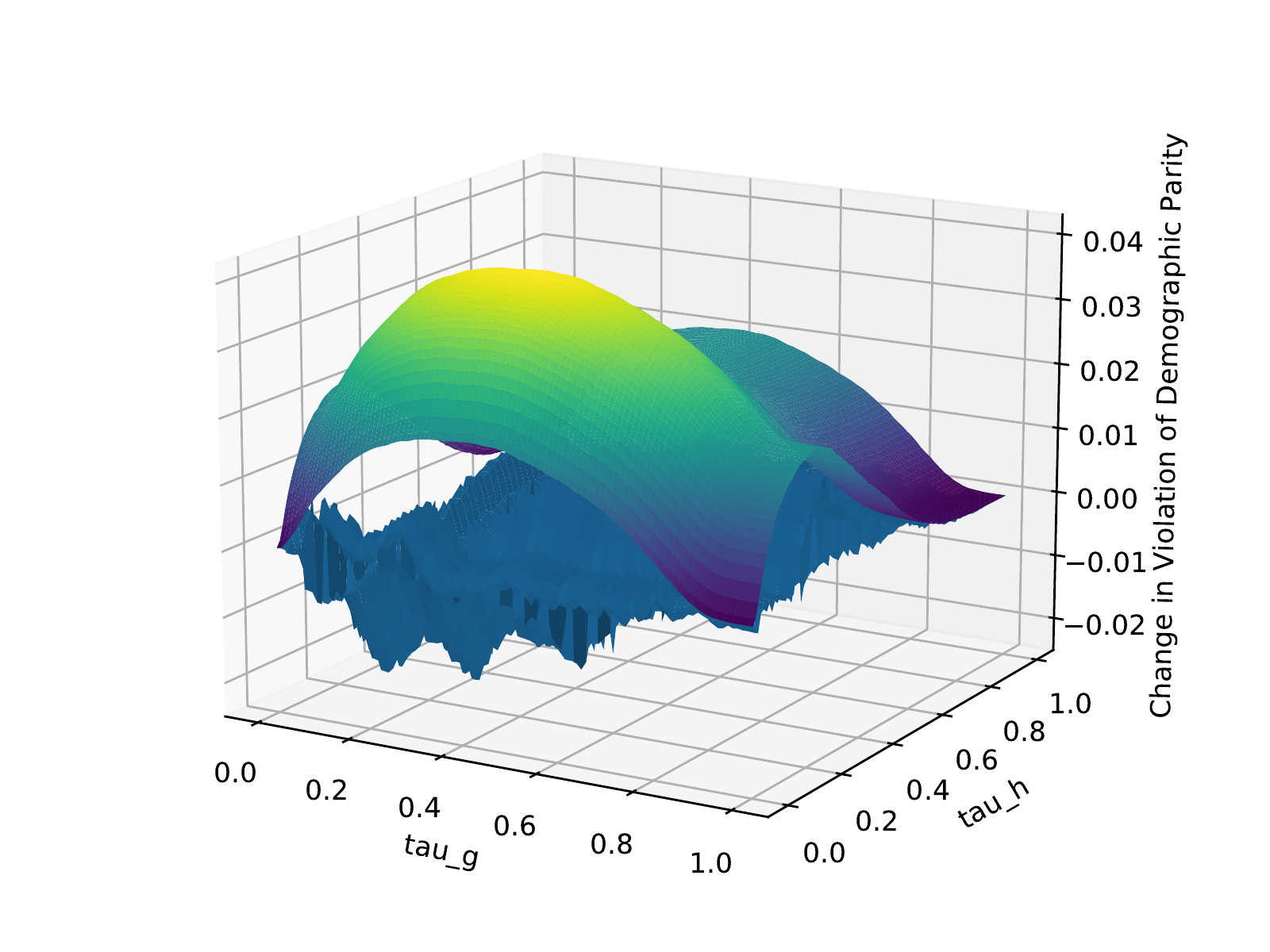}}
    \subfigure[Colorado (CO)]{\includegraphics[width=0.23\linewidth]{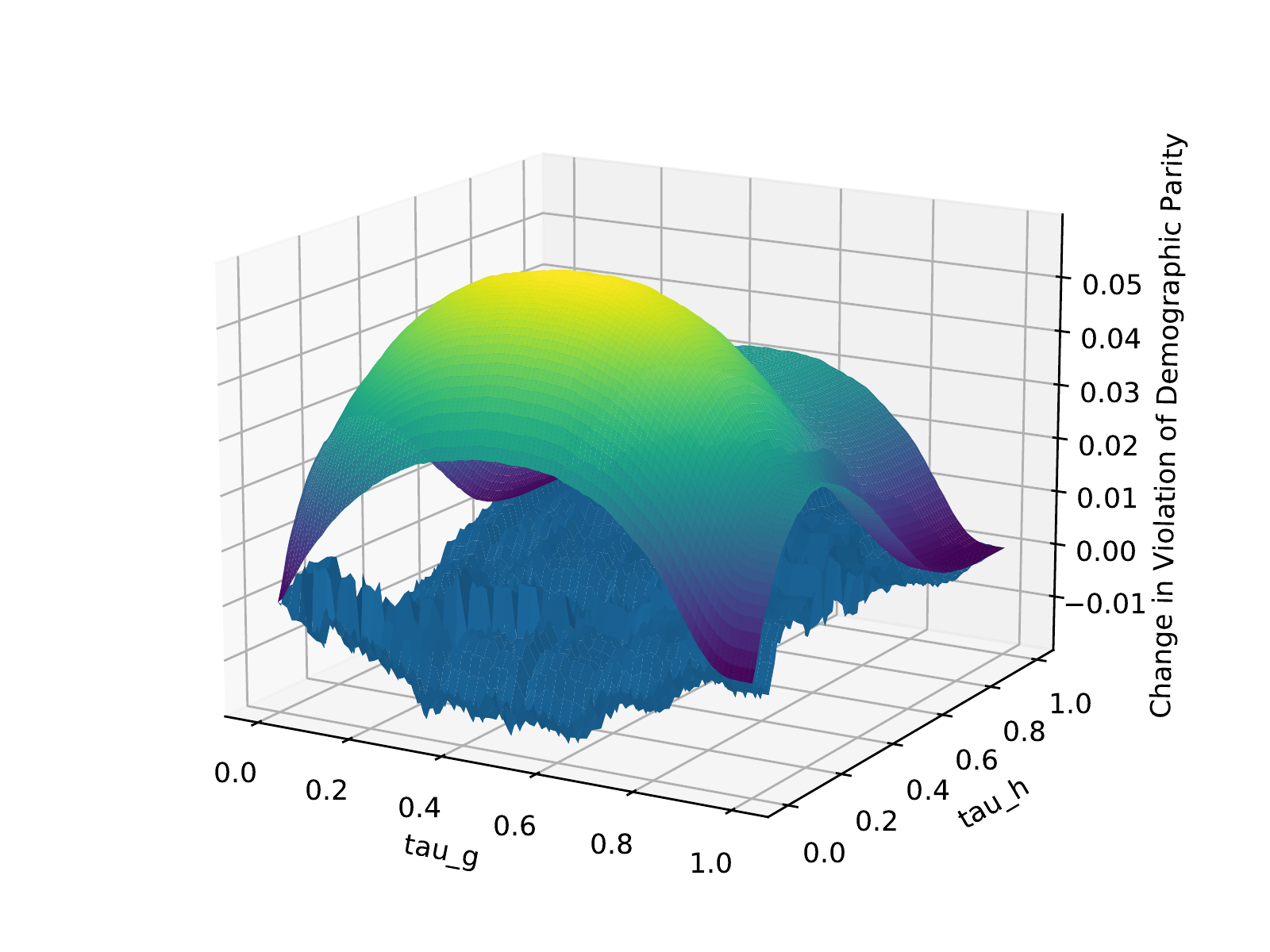}}
    \subfigure[Florida (FL)]{\includegraphics[width=0.23\linewidth]{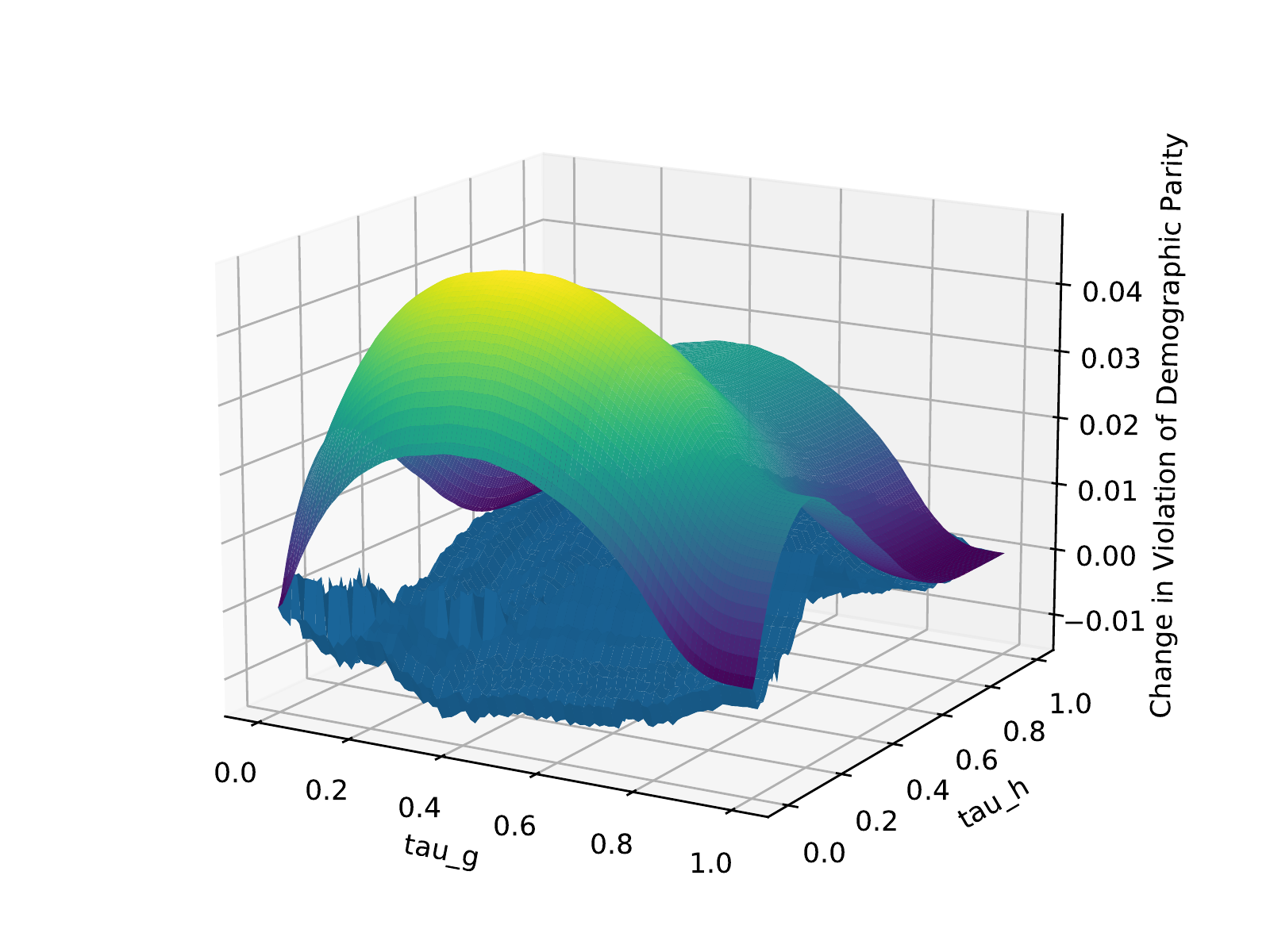}}
    \subfigure[Maryland (MD)]{\includegraphics[width=0.23\linewidth]{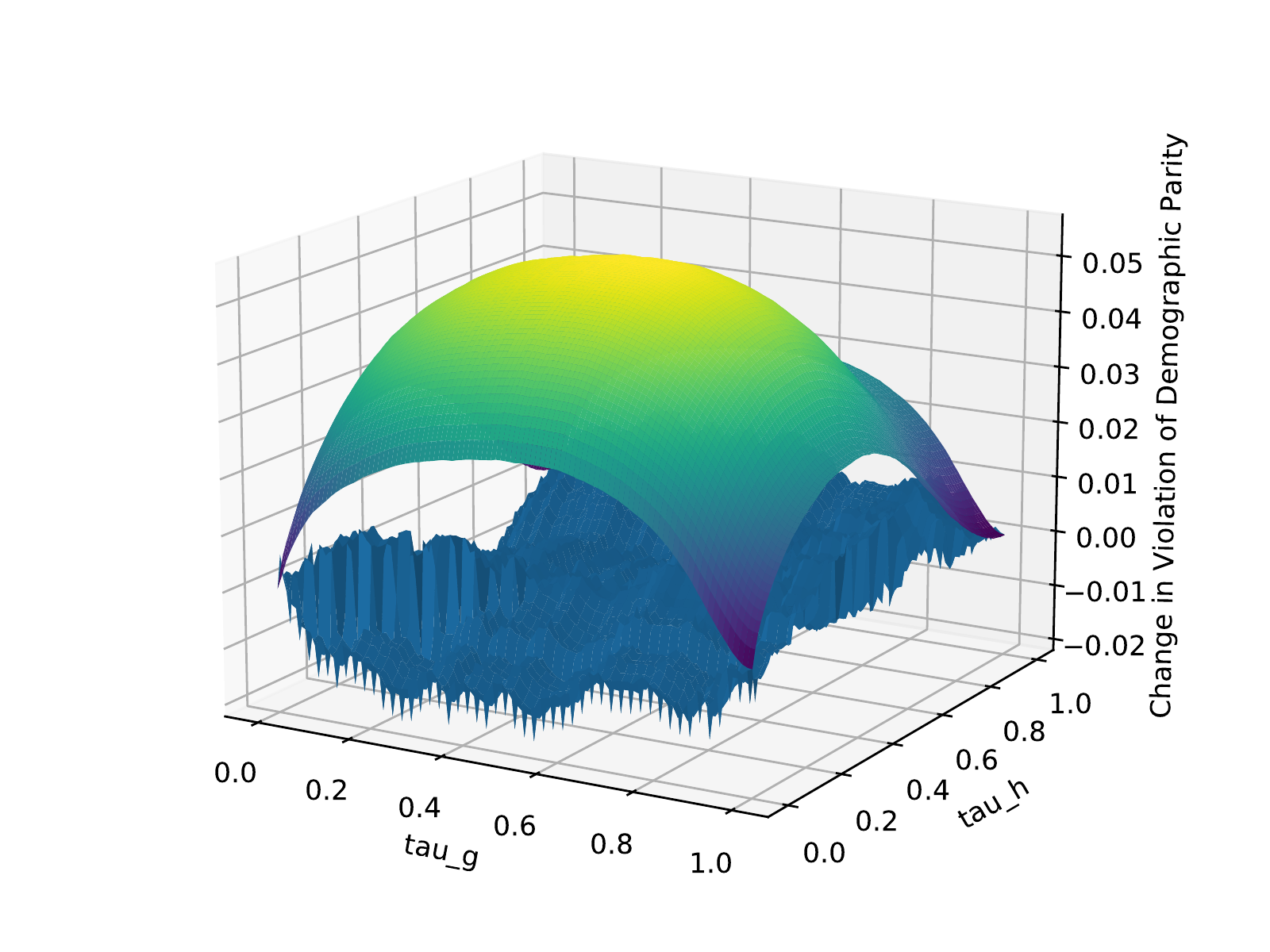}}
    \subfigure[Oregon (OR)]{\includegraphics[width=0.23\linewidth]{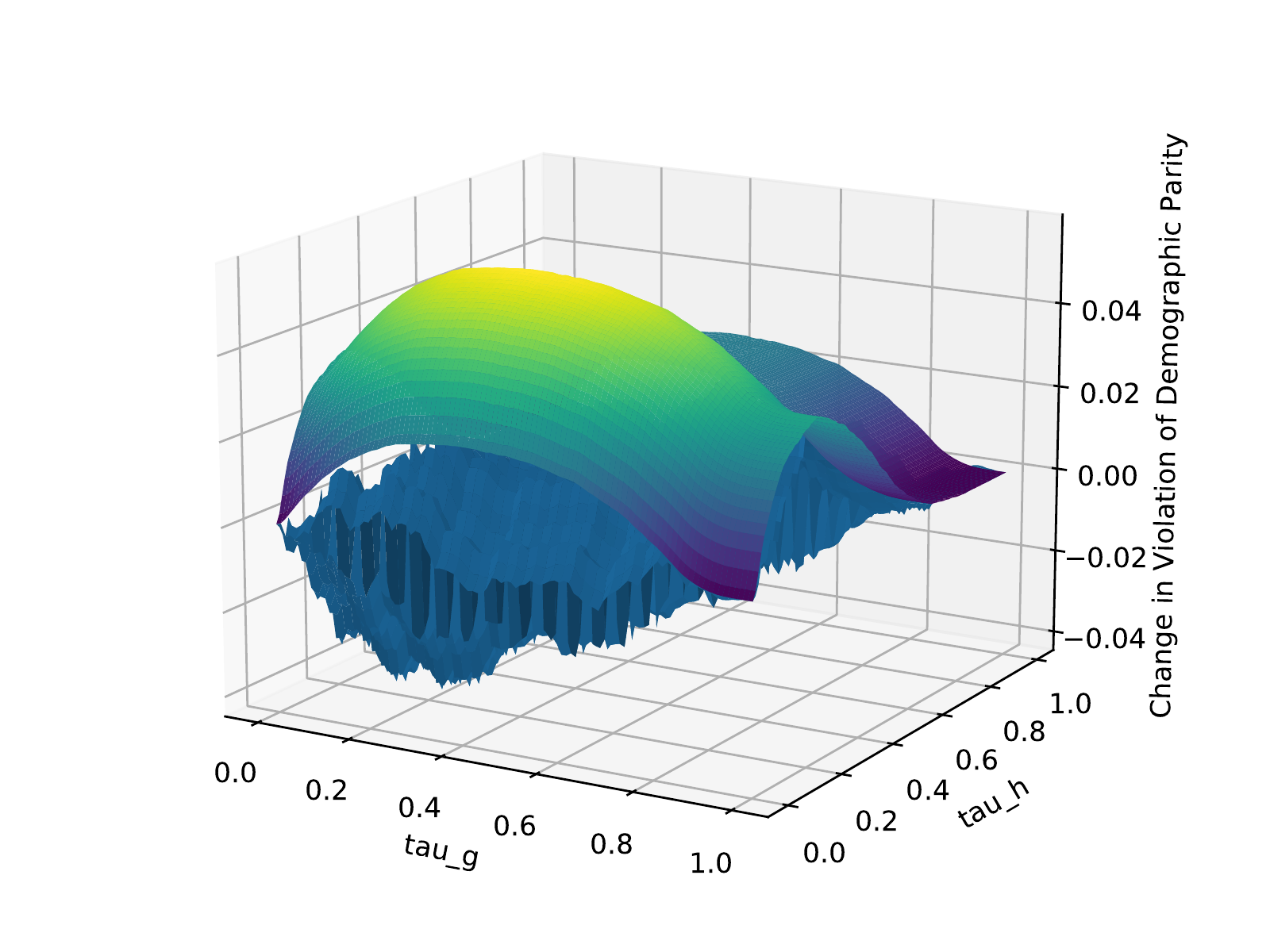}}
    \subfigure[Utah (UT)]{\includegraphics[width=0.23\linewidth]{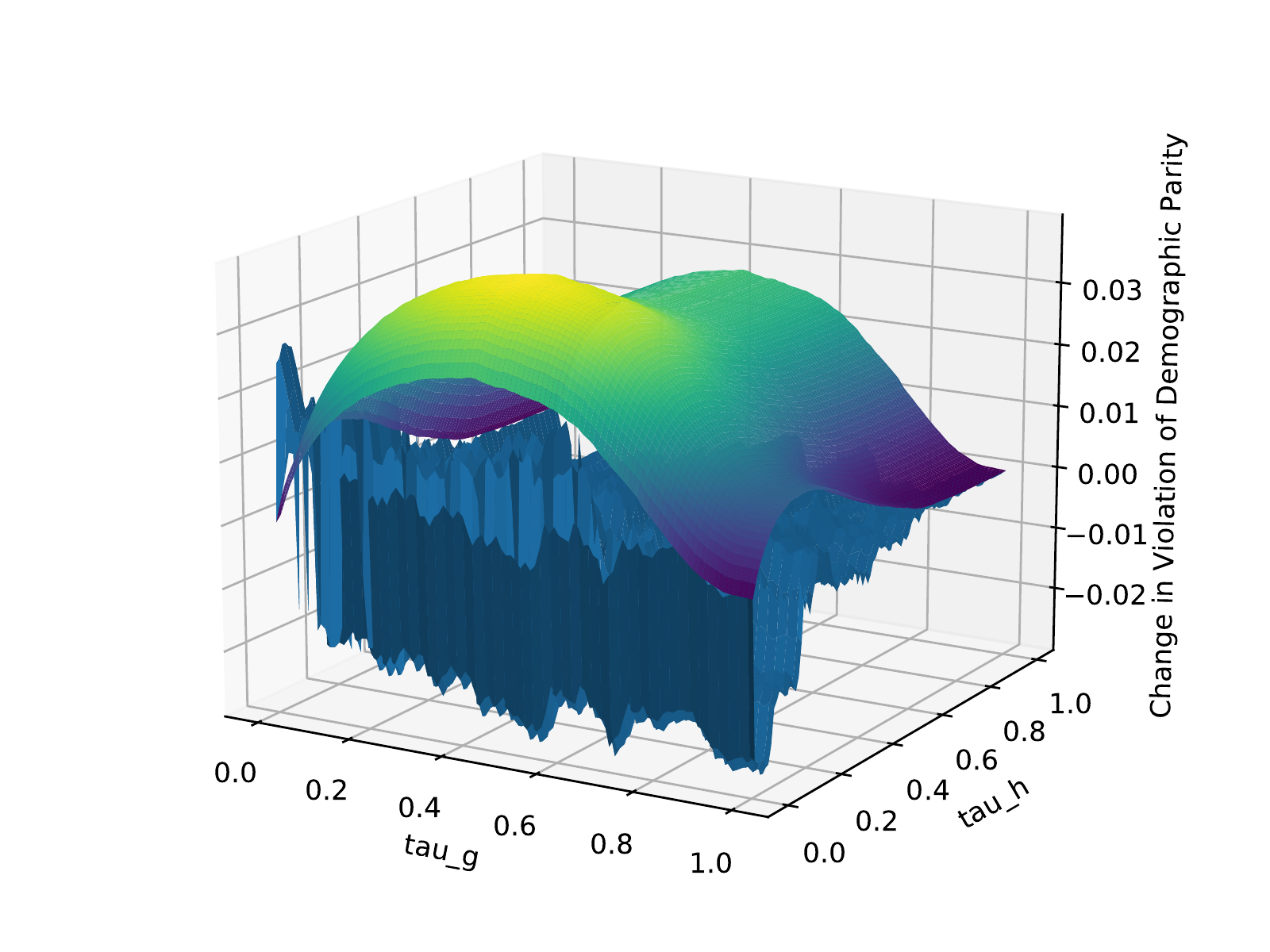}}
    \subfigure[Vermont (VT)]{\includegraphics[width=0.23\linewidth]{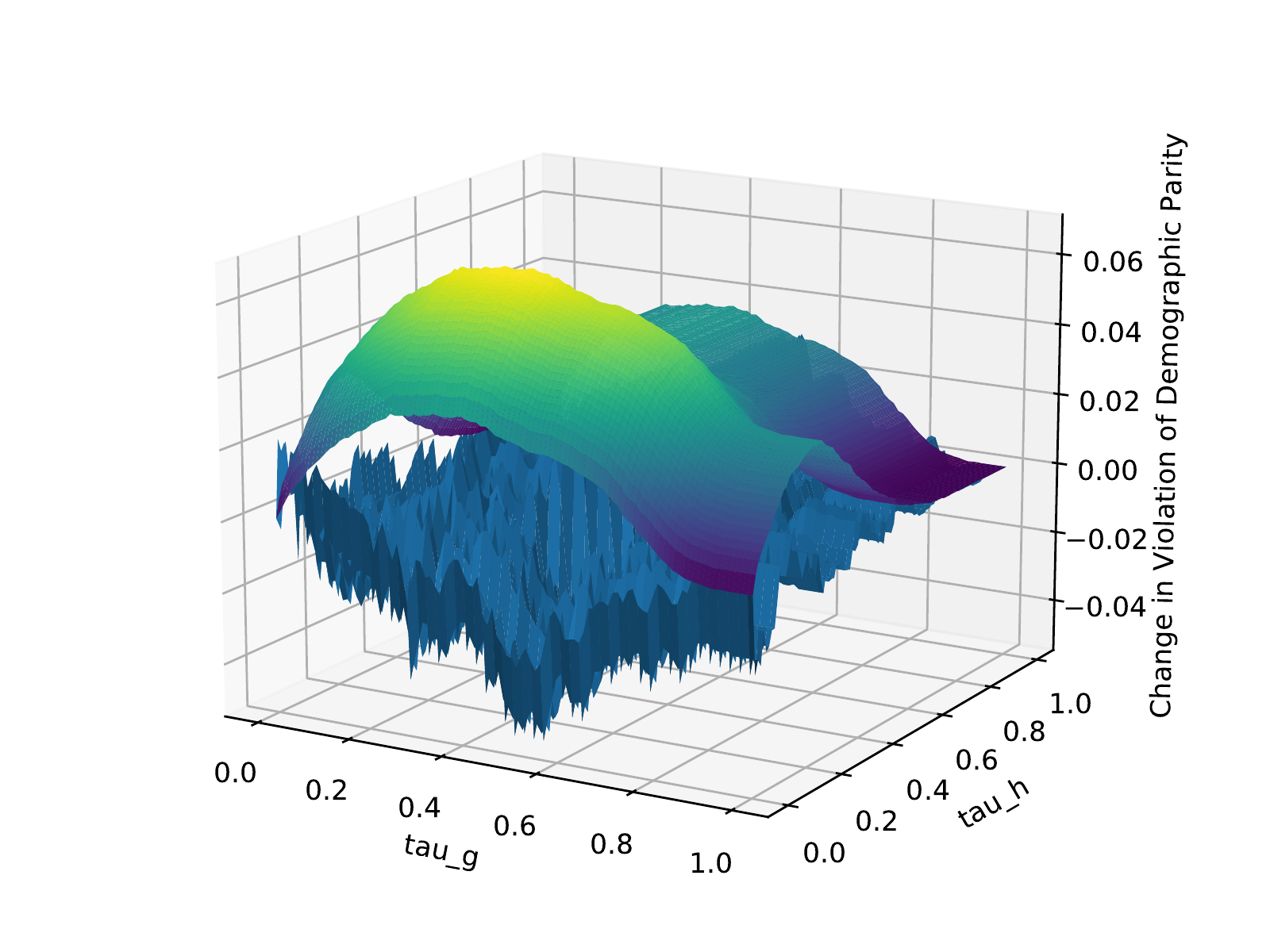}}
    \subfigure[Wisconsin (WI)]{\includegraphics[width=0.23\linewidth]{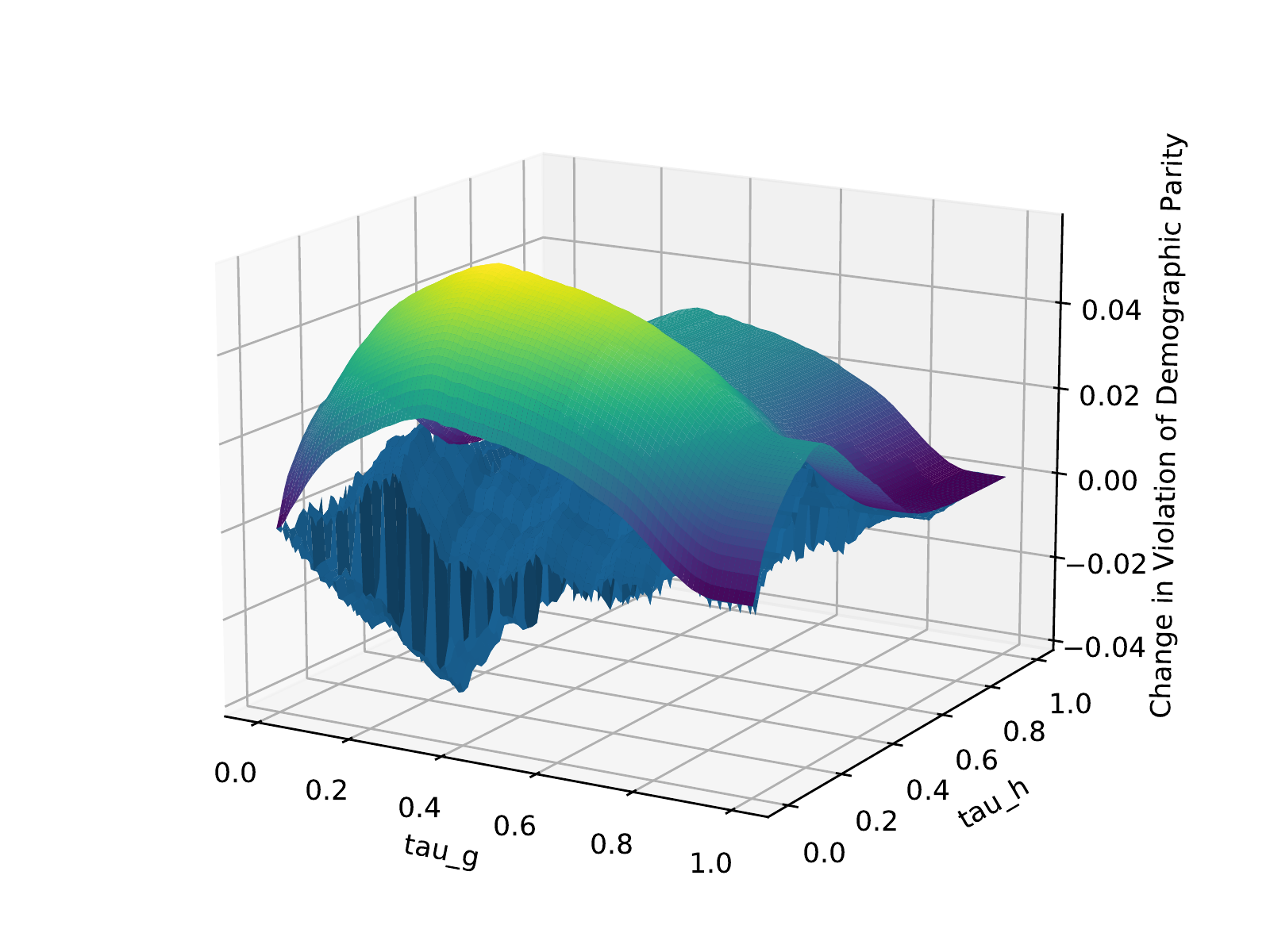}}
    \caption{Change in violation of demographic parity for hypothetical policies trained on 2014 data and reused for 2018 (blue) compared to label-shift bounds (\cref{thm:dp-label}, gradated). The $x$-axis and $y$-axis represent the thresholds $\tau_g$ and $\tau_h$, respectively.}
    \label{fig:label_shift_temporal_CA_A}
\end{figure}

\begin{figure}[H]
    \centering
    \subfigure[CA $\longrightarrow$ NV]{\includegraphics[width=0.23\linewidth]{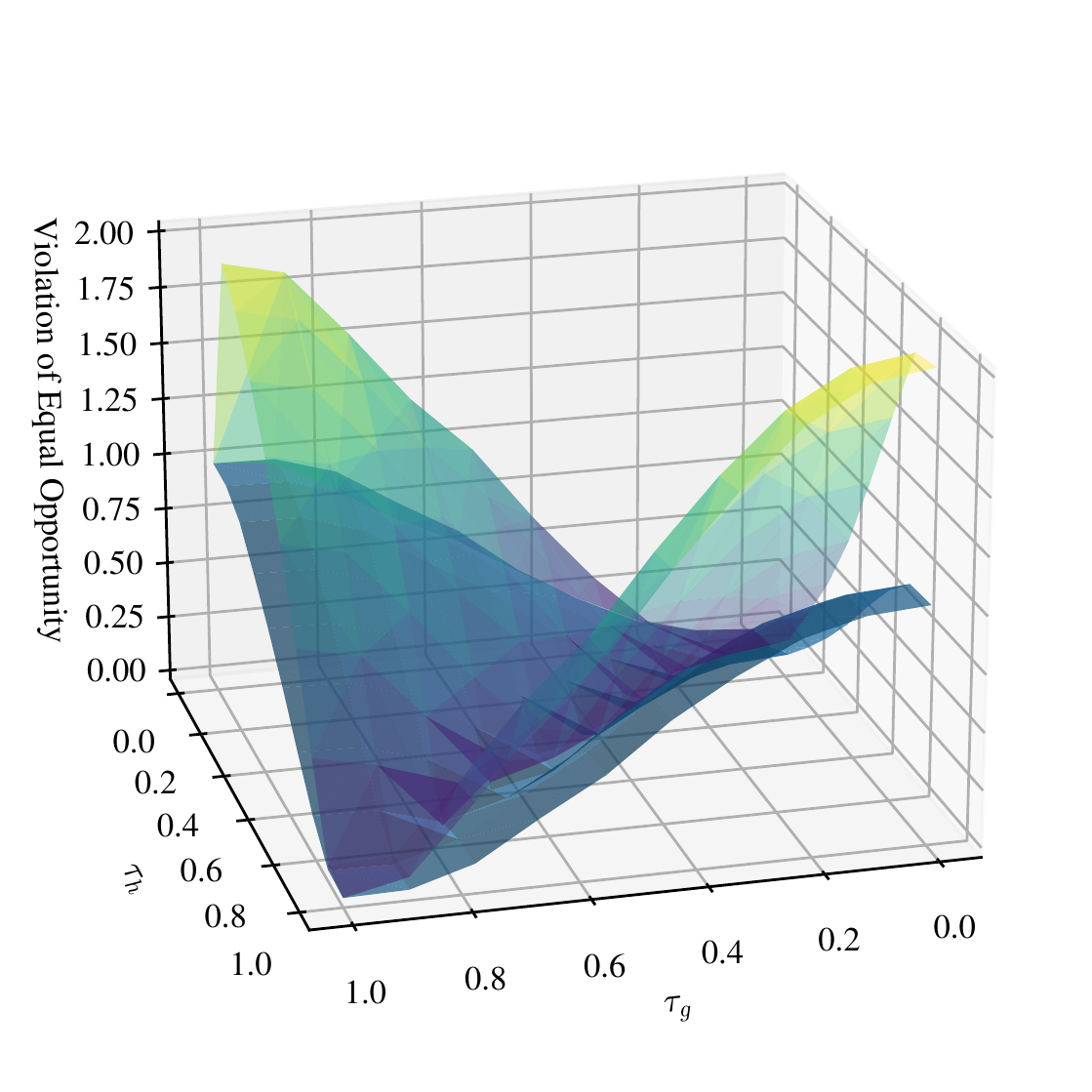}}
    \subfigure[CA $\longrightarrow$ IL]{\includegraphics[width=0.23\linewidth]{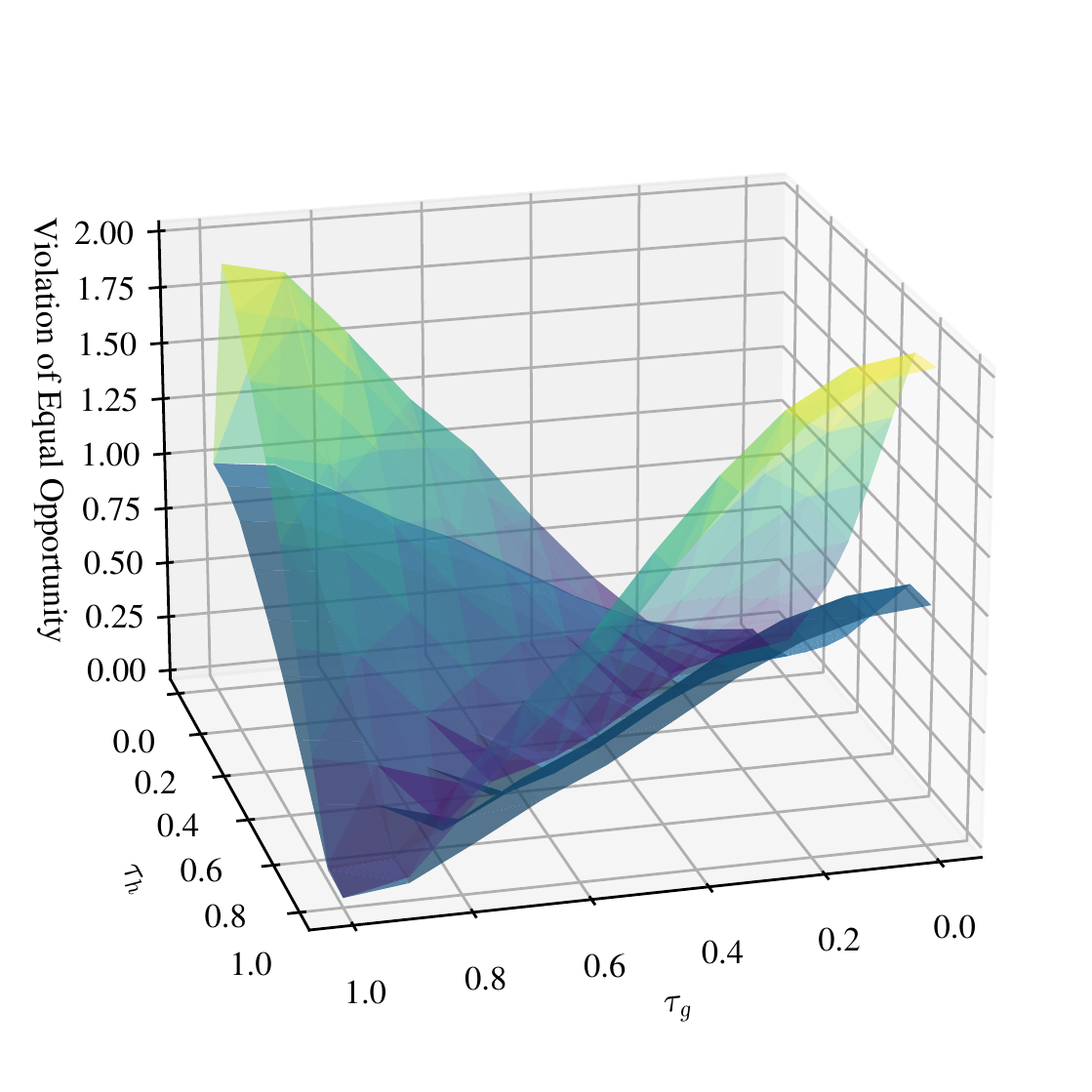}}
    \subfigure[NV $\longrightarrow$ IL]{\includegraphics[width=0.23\linewidth]{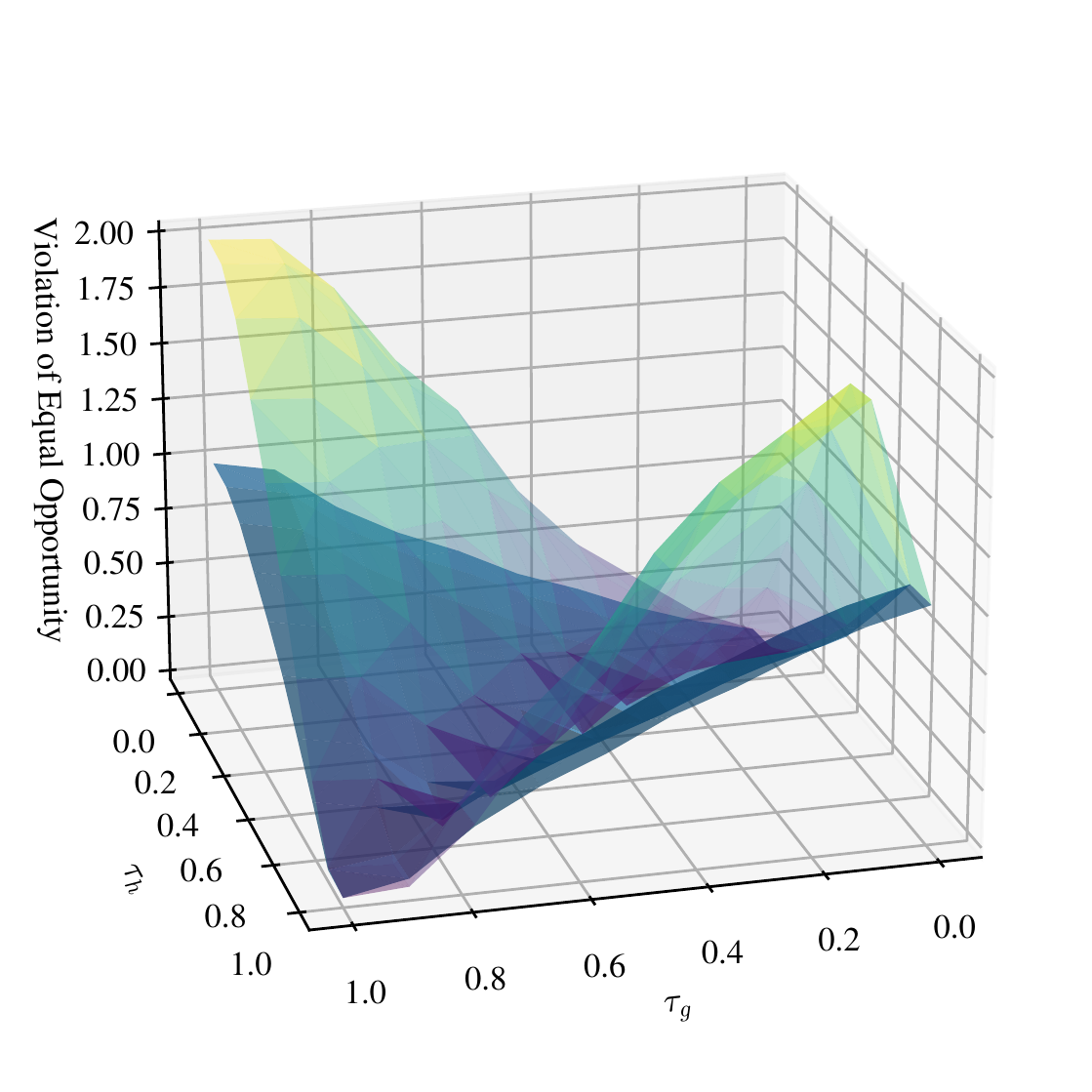}}
    \caption{Change in violation of equal opportunity for hypothetical policies trained on one US state's data and reused for another state (blue) compared to covariate-shift bounds (\cref{thm:cov-EO-superbound,,thm:geometric-beta}, gradated). The $x$-axis and $y$-axis represent the thresholds $\tau_g$ and $\tau_h$, respectively.}
\label{fig:covariate_EO_empirical} 
\end{figure}

\section{Omitted Proofs} \label{app:proofs}


\textbf{Proof of \cref{lem:v-of-B0-is-R}}:

\emph{Statement}:
  For all \(\pi\), \(\disp\), and \(\mathbf{D}\), when \(\mathbf{B} = 0\),
\(\disp(\pi, \S) = \disp(\pi, \R)\).

\begin{proof}
  By the definitions of group-vectorized shift (\cref{def:vector-divergence}) and divergence (\cref{def:divergence})
  together with the bounded distribution shift assumption (\cref{asm:budget}), we note
  \begin{equation}
        \mathbf{B} = 0 \implies \mathbf{D}(\R \parallel \S) = 0
    \end{equation}
and
\begin{equation}
  D_{g}(\R \parallel \S) = 0 \implies
\Pr_{\S}(X, Y \mid G {=} g) =
\Pr_{\R}(X, Y \mid G {=} g)
\end{equation}
Combining these implications and invoking the independence of
\(\hat{Y} \sim \pi\) and \(Y\) conditioned on \(X\) and \(G\)
(\cref{eq:pr-notation}), it follows that
  \begin{equation}
    \mathbf{B} = 0 \implies    \forall g, \quad
\Pr_{\pi, \S}(X, Y, \hat{Y} \mid G {=} g) =
\Pr_{\pi, \R}(X, Y, \hat{Y} \mid G {=} g)
\end{equation}
Consulting the definition of disparity (\cref{def:disparity}), it follows that \(\disp(\pi, \S)\) and
\(\disp(\pi, \R)\) are equal when \(\mathbf{B} =0\).
\end{proof}

\textbf{Proof of \cref{thm:lipshitz}}:

\emph{Statement}:
If there exists an \(\mathbf{L}\) such that
\(
   \nabla_{\mathbf{b}} v(\disp, \mathbf{D}, \pi, \S, \mathbf{b}) \preceq \mathbf{L}
\),
everywhere along some curve from \(0\) to \(\mathbf{B}\), then
\begin{equation}
  \disp(\pi, \R) \leq \disp(\pi, \S) + \mathbf{L} \cdot \mathbf{B}
\end{equation}

\begin{proof}
We reiterate that \(v(\disp, \mathbf{D}, \pi, \S, \mathbf{b})\) defines a
scalar field over the non-negative cone \(\mathbf{b} \in (\RR_{+} \cup 0)^{|\mathcal{G}|}\).
Treating \(v\) as a scalar potential, we may define the conservative vector field \(\mathbf{F}\):
    \begin{equation}
      \mathbf{F} = \nabla_{\mathbf{b}} v
    \end{equation}
    This formulation, in terms of a potential, ensures the path-independence of
the line integral of \(\mathbf{F}\) along any continuous curve \(C\) from \(0\) to
\(\mathbf{B}\). That is,
    \begin{equation}
    v(..., \mathbf{B}) - v(..., 0) = \int_C \mathbf{F}(\mathbf{b}) \cdot \dd{\mathbf{b}}
    \end{equation}
    Therefore, given a Lipshitz condition for \(\mathbf{F}\) along any curve \(C\)
    with endpoints \(0\) and \(\mathbf{B}\), \ie when there exists some finite
    \(\mathbf{L}\) such that
    \begin{align}
      \forall \mathbf{b} \in C,
      \quad
      \mathbf{F}(\mathbf{b}) \preceq \mathbf{L}
    \end{align}
    and therefore
    \begin{align}
      v(\disp, \mathbf{D}, \pi, \S, \mathbf{B}) &= v(..., 0) + \int_{C} \mathbf{F}(\mathbf{b}) \cdot \dd{\mathbf{b}} \\
      &\leq \disp(\pi, \S) + \mathbf{L} \cdot \mathbf{B}
    \end{align}
    By the bounded distribution shift assumption (\cref{asm:budget}), \cref{lem:v-of-B0-is-R}, and the definition of the supremum bound (\cref{def:v}), we conclude
    \begin{align}
      \disp(\pi, \R) \leq \disp(\pi, \S) + \mathbf{L} \cdot \mathbf{B}
    \end{align}
\end{proof}





\textbf{Proof of \cref{thm:subadditive}}:

\emph{Statement}:
 Suppose, in the region \(\mathbf{D}(\R \parallel \S) \preceq \mathbf{B}\), that
  \(w\) is subadditive in its last argument.
That is,
\(
    w(..., \mathbf{a}) + w(..., \mathbf{c})
    \geq w(..., \mathbf{a} + \mathbf{c})
    \) for \(\mathbf{a}, \mathbf{c} \succeq 0\) and \(\mathbf{a} + \mathbf{c} \preceq \mathbf{B}\).
Then, a local, first-order approximation of \(w(..., \mathbf{b})\) evaluated at
\(0\), \ie,
\begin{equation}
  \mathbf{L} =
  \nabla_{\mathbf{b}} w(..., \mathbf{b}) \big|_{\mathbf{b} = 0} =
  \nabla_{\mathbf{b}} v(..., \mathbf{b}) \big|_{\mathbf{b} = 0}
 \end{equation}
 provides an upper bound for \(v(..., \mathbf{B})\):
 \begin{equation}
   v(\disp, \mathbf{D}, \pi, \S, \mathbf{B})
   \leq
   \disp(\pi, \S) + \mathbf{L} \cdot \mathbf{B}
\end{equation}

\begin{proof}
  Represent
\begin{equation}
\mathbf{B} = \sum_g \mathbf{e}_g B_g
\end{equation}
Then, invoking the definition of the derivative as a Weierstrass limit from
elementary calculus, as well as \cref{lem:v-of-B0-is-R}, and by repeatedly
appealing to the assumed subadditivity condition within our domain, we find
\begin{subequations}
  \begin{align}
\mathbf{B} \cdot \mathbf{L}
    &=
\mathbf{B} \cdot \nabla_{\mathbf{b}} v(\pi, \S, \mathbf{b}) \big|_{\mathbf{b} = 0} \\
&= \sum_g B_g \frac{d}{d x} v (\pi, \S, x \mathbf{e}_g) \big|_{x = 0}\\
&= \sum_g B_g \lim_{N \to \infty} N \Big(v(\pi, \S, \frac{1}{N} \mathbf{e}_g) - v(\pi, \S, 0) \Big) \\
&= \sum_g B_g \lim_{N \to \infty} N \Big( w(\pi, \S, \frac{1}{N} \mathbf{e}_g) \Big) \\
&\geq \sum_g B_g ~ w(\pi, \S, \mathbf{e}_g) \\
&\geq \sum_g w(\pi, \S, B_g \mathbf{e}_g) \label{eq:subadditive2}\\
&\geq w(\pi, \S, \mathbf{B}) \label{eq:subadditive1}
\end{align}
\end{subequations}
Also recall (\cref{def:w})
\begin{align}
  w(\pi, \S, \mathbf{B}) &\define v(\pi, \S, \mathbf{B}) - \disp(\pi,\S)
\end{align}
Therefore, we obtain
\begin{equation}
v(\pi, \S, \mathbf{B}) \leq \disp(\pi, \S) + \mathbf{B} \cdot \mathbf{L}
\end{equation}
\end{proof}



\begin{lemma}
\label{lem:cs-positive-prediction-rate-difference}
For each group \(g\in \mathcal{G}\), under covariate shift,
\begin{equation}
  \begin{aligned}
\Pr_{\pi,\R}&\Big(\hat{Y} {=} 1 \mid G {=} g\Big) - \Pr_{\pi,\S}\Big(\hat{Y} {=} 1 \mid G {=} g \Big) = \cov_{\pi,\S}\Big[\omega_g(\R, \S, X), \Pr_{\pi(X, g)}(\hat{Y} {=} 1)\Big]
\end{aligned}
\end{equation}
\end{lemma}

\begin{proof}
First, note that $\E_{\S}[\omega_g(\R, \S, x)] = 1$, since
\begin{align*}
   \E_{\S}[\omega_g(\R, \S, x)]
   &= \int_\mathcal{X}  \omega_g(\R, \S, x)\Pr_{\S}(X = x \mid  G{=}g) \dd{x}
   \\
   &= \int_\mathcal{X} \frac{\Pr_{\R}(X {=} x \mid  G{=}g)}{\Pr_{\S}(X {=} x \mid  G=g)} \Pr_{\S}(X {=} x \mid  G{=}g) \dd{x} \\
   &= \int_\mathcal{X} {\Pr_{\R}(X {=} x \mid  G{=}g)}  \dd{x}
   = 1
\end{align*}

Then, adopting the shorthand \(\omega_g(x) = \omega_g(\R, \S, x)\), we have:
\begin{align}
   &\Pr_{\pi,\R}\Big(\hat{Y} {=} 1 \mid G {=} g\Big) - \Pr_{\pi,\S}\Big(\hat{Y} {=} 1 \mid G {=} g \Big) \\
   =&\int_{\mathcal{X}} \Pr_{\pi(x, g)}(\hat{Y} {=} 1) \Pr_{\R}(X {=} x \mid  G {=} g) \dd{x}
    - \int_{\mathcal{X}} \Pr_{\pi(x, g)}(\hat{Y} {=} 1) \Pr_{\S}(X {=} x \mid  G {=} g) \dd{x}\\
   =& \int_{\mathcal{X}} \Pr_{\pi(x, g)}(\hat{Y} {=} 1) \Big( \omega_g(x) - 1 \Big) \Pr_{\S}(X {=} x \mid  G {=} g) \dd{x} \\
   =&\  \E_{\S} \Big[\Pr_{\pi(x, g)}(\hat{Y} {=} 1)  (\omega_g(x) - 1) \mid G {=} g \Big]\\
    =& \   \E_{\S} \Big[\Pr_{\pi(x, g)}(\hat{Y} {=} 1)  (\omega_g(x) - \E_{\S}[\omega_g(x) ]) \mid G {=} g  \Big] \tag{since $\E_{\S}[\omega_g(x)] = 1$}\\
    =& \   \E_{\S} \Big[(\Pr_{\pi(x, g)}(\hat{Y} {=} 1)- \E_{\S} [\Pr_{\pi(x, g)}(\hat{Y} {=} 1)] (\omega_g(x) - \E_{\S}[\omega_g(x) ]) \mid G {=} g  \Big] \tag{$\E[f(x) - \E[f(x)]] = 0$}\\
    =& \ \cov_{\pi, \S}\Big[w_g(\R, \S, X), \Pr_{\pi(x, g)}(\hat{Y} {=} 1)\Big]
 \end{align}
\end{proof}

\begin{lemma}
\label{lemma:01rv-variance-expecation-bound}
If $X$ is a random variable and $X \in [0,1]$, then $\var(X)\leq \E[X](1-\E(X))$.
\end{lemma}
\begin{proof}
\begin{align*}
    \var[X] &= \E[(X - \E[X])^2]\\
    &= \E[X^2] - (\E[X])^2\\
    &\leq \E[X]- (\E[X])^2 \tag{$X\in[0,1]$}\\
    & = \E[X](1-\E(X))
\end{align*}
\end{proof}

\textbf{Proof of \cref{thm:dp-covariate}}:

\emph{Statement}:
For demographic parity between two groups under covariate shift (denoting, for
each \(g\), \(\beta_{g} \define \Pr_{\pi,\S}(\hat{Y} {=} 1 \mid G {=} g)\)),
  \begin{equation}
    \disp_{\DP}(\pi, \R)  \leq \disp_{\DP}(\pi, \S) + \sum_g \Big( \beta_{g}(1 - \beta_{g}) 
    B_g \Big)^{\nicefrac{1}{2}}
  \end{equation}
\begin{proof}
Again adopting the shorthand \(\omega_g(x) = \omega_g(\R, \S, x)\),
\begin{align}
& \disp(\pi, \R) \\&=
    \Big|\Pr_{\pi,\R}\Big(\hat{Y} {=} 1 \mid G {=} g\Big) - \Pr_{\pi,\R}\Big(\hat{Y} {=} 1 \mid G {=} h \Big)\Big|\\
    &=  \Big|\Pr_{\pi,\R}\Big(\hat{Y} {=} 1 \mid G {=} g\Big) - \Pr_{\pi, \S}\Big(\hat{Y} {=} 1 \mid G {=} g\Big)\\
    & \qquad + \Pr_{\pi, \S}\Big(\hat{Y} {=} 1 \mid G {=} g\Big)
- \Pr_{\pi,\S}\Big(\hat{Y} {=} 1 \mid G {=} h \Big)\\
    & \qquad +  \Pr_{\pi,\S}\Big(\hat{Y} {=} 1 \mid G {=} h \Big)
- \Pr_{\pi,\R}\Big(\hat{Y} {=} 1 \mid G {=} h \Big)\Big|\\
    &\leq \disp(\pi, \S) + \cov_{\S}\Big[\omega_g(x), \Pr_{\pi(x, g)}(\hat{Y} {=} 1)\Big] + \cov_{\S}\Big[\omega_h(x), \Pr_{\pi(x, h)}(\hat{Y} {=} 1)\Big]
    \tag{By \Cref{lem:cs-positive-prediction-rate-difference}} \\
    & \leq  \disp(\pi, \S) + \sqrt{\var_{\S}[\omega_g(x)]}\cdot\sqrt{\var_{\S}[\Pr_{\pi(x, g)}(\hat{Y} {=} 1)]} + \sqrt{\var_{\S}[\omega_h]}\cdot\sqrt{\var_{\S} [\Pr_{\pi(x, h)}(\hat{Y} {=} 1)]} \tag{$\big|\cov[a,b]\big| \leq \sqrt{\var[a]}\cdot \sqrt{\var[b]}$} \\
    & \leq \disp(\pi, \S) + \sqrt{\var_{\S}[\omega_g(x)]} \cdot \sqrt{\E_{\S}[\Pr_{\pi(x, g)}(\hat{Y} {=} 1)](1 - \E_{\S} [\Pr_{\pi(x, g)}(\hat{Y} {=} 1)])} \\
    & \qquad \qquad \quad
    + \sqrt{\var_{\S}(\omega_h(x))}\cdot \sqrt{\E_{\S} [\Pr_{\pi(x, h)}(\hat{Y} {=} 1)](1- \E_{\S} [\Pr_{\pi(x, h)}(\hat{Y} {=} 1))]} \tag{$\hat{Y}\in\{0,1\}$, and \Cref{lemma:01rv-variance-expecation-bound}}\\
    \label{eq:demongraphic-parity-covariate-shift-bound}
    & = \disp(\pi, \S) +  \sqrt{\var_{\S}[\omega_g(x)]}\cdot\sqrt{\beta_g(1-\beta_g)} + \sqrt{\var_{\S}[\omega_h(x)]}\cdot\sqrt{\beta_h(1-\beta_h)}
    \tag{$\beta_g = \Pr_{\pi, \S}(\hat{Y} {=}1 | G {=} g) = \E_{\S} [\rr{\mathbbm{1}}_{\pi(x, g)}(\hat{Y} {=} 1)]$}\\
    & = \disp(\pi, \S) + \sum_g \Big( \beta_{g}(1 - \beta_{g}) \var_{\S}[\omega_g(\R, \S, x)] \Big)^{\nicefrac{1}{2}}
    \end{align}
\end{proof}

\textbf{Proof of \cref{cor:dp-covariate-multi}}:
\emph{Statement}:
\cref{thm:dp-covariate} \cref{thm:dp-covariate} may be generalized to multiple classes
\(\mathcal{Y} = \{1, 2, ..., m\}\)
 and multiple groups \(\mathcal{G}\in\{1, 2, ..., n\}\),
 \begin{equation}
  \dispdp(\pi, \R) \define \sum_{y \in \mathcal{Y}} \sum_{g,h \in \mathcal{G}} \Big|
  \Pr_{\pi, \R}(\hat{Y} {=} y \mid G {=} g) -
  \Pr_{\pi, \R}(\hat{Y} {=} y \mid G {=} h)
  \Big|
\end{equation}
\begin{equation}
  \disp_{\DP}(\pi, \R) \leq \disp_{\DP}(\pi, \S) + \sum_{y \in \mathcal{Y}} \sum_{g,h \in \mathcal{G}}
  \big(\beta_{g,y}(1 - \beta_{g,y}) B_g\big)^{\nicefrac{1}{2}}
\end{equation}
where \(\beta_{g,y} =\Pr(\hat{Y}{=}y \mid G{=}g)\), and assuming $\var_{\S}[w_g(\S, \R, X)] \leq B_g$.

\begin{proof}
We again adopt the shorthand \(\omega_g(x) = \omega_g(\R, \S, x)\).
We first generalize \cref{lem:cs-positive-prediction-rate-difference}.
For each group \(g\in \mathcal{G}\), under covariate shift, for all \(y \in \mathcal{Y}\),
\begin{equation}
  \begin{aligned}
\Pr_{\pi,\R}&\Big(\hat{Y} {=} y \mid G {=} g\Big) - \Pr_{\pi,\S}\Big(\hat{Y} {=} y \mid G {=} g \Big) = \cov_{\pi,\S}\Big[\omega_g(\S, \R, X), \Pr_{\pi(X, g)}(\hat{Y} {=} y)\Big]
\end{aligned}
\end{equation}
Retracing the logic of \cref{thm:dp-covariate}, for \(\var_{\S}[\omega_g(\R, \S, x)] \leq B_g\), it follows that
\begin{align}
  \dispdp(\pi, \R)
  &\define \sum_{y \in \mathcal{Y}} \sum_{g,h \in \mathcal{G}} \Big|
  \Pr_{\pi, \R}(\hat{Y} {=} y \mid G {=} g) -
  \Pr_{\pi, \R}(\hat{Y} {=} y \mid G {=} h)
  \Big|  \\
  &\leq \dispdp(\pi, \S) + \sum_{y \in \mathcal{Y}} \sum_{g,h \in \mathcal{G}}
  \sqrt{\big(\beta_{g,y}(1 - \beta_{g,y}) \var_{\S}[\omega_g(x)]\big)}\\
    &\leq \dispdp(\pi, \S) + \sum_{y \in \mathcal{Y}} \sum_{g,h \in \mathcal{G}}
  \sqrt{\big(\beta_{g,y}(1 - \beta_{g,y}) B_g\big)}\\
  & = \dispdp(\pi, \S) + \sum_{y \in \mathcal{Y}} \sum_{g,h \in \mathcal{G}}
  \big(\beta_{g,y}(1 - \beta_{g,y}) B_g\big)^{\nicefrac{1}{2}}
\end{align}
\end{proof}

\textbf{Proof of \cref{thm:cov-EO-superbound}}

\emph{Statement}:
 Subject to covariate shift and any given \(\mathbf{D}, \mathbf{B}\), assume extremal values for \(\beta_{g}^{+}\), \ie,
\begin{equation}
\forall g, ~~
D_{g}(\R \parallel \S) < B_{g} \implies
l_g \leq \beta^{+}_{g}(\pi, \R) < u_g
\end{equation}
then, for \(v\) corresponding to \(\disp_{\EOp}\),
\begin{equation}
v(\disp_{\EOp}, \mathbf{D}, \pi, \S, \mathbf{B}) \leq \max_{x_g \in \{l_g, u_g\}} \sum_{g, h}\Big|
  x_g - x_h
\Big|
\end{equation}

\begin{proof}
  Recall that, for this setting,
  \begin{equation}
    v(\disp_{\EOp}, \mathbf{D}, \pi, \S, \mathbf{B}) = \sup_{\mathbf{D}(\R \parallel \S) \preceq \mathbf{B}} \disp_{\EOp}
  \end{equation}
 and
 \begin{equation}
   \disp_{\EOp} = \sum_{g,h} |\beta_{g}^{+} - \beta_{h}^{+} |
 \end{equation}
  This latter expression is convex in each \(\beta_{g}^{+}\). Therefore, \(\disp_{\EOp}\)
  is maximized on the boundary of its domain, \ie \(\beta^{+}_{g} \in \{l_{g}, u_{g}\}\)
  for each \(g\), given the assumption of the theorem.
\end{proof}

\textbf{Proof of \cref{cor:max-cov-EO}}

\emph{Statement}:
\rr{The disparity measurement} \(\disp_\EOp\) cannot exceed \(\frac{|\mathcal{G}|^2}{4}\).

\begin{proof}
  We note that each \(\beta^{+}_{g}\) is ultimately confined to the interval
\([0, 1]\).  Building on our proof for \cref{thm:cov-EO-superbound}, to maximize
\(\disp_{\EOp}\), we must consider the boundary of this domain, where, for each
\(g\), \(\beta^{+}_{g} \in \{0, 1\}\). Because the only terms that contribute to
\(\disp_{\EOp}\) are those in which \(\beta_{g}^{+} = 1\) and
\(\beta_{h}^{+} = 0\) (as opposed to \(\beta_{g}^{+} = \beta_{h}^{+}\)), we seek to
maximize the number of such terms. This occurs when as close to half of the
groups as possible have one extremal true positive rate (\eg, without loss of
generality, \(\beta^{+}_{g} = 1\)) and the remaining groups have the other (\eg,
\(\beta^{+}_{g} = 0\)). In such cases, \(\disp_{\EOp}\) is given by
\begin{equation}
\max \disp_{\EOp} = \floor{\frac{\mathcal{G}}{2}} \ceil{\frac{\mathcal{G}}{2}} \leq \frac{|\mathcal{G}|^{2}}{4}
\end{equation}
\end{proof}

\textbf{Proof of \cref{thm:yangsthm}}:

\emph{Statement}:
A Lipshitz condition bounds \(\nabla_\mathbf{b} v(\disp_\DP, \mathbf{D}, \pi, \S, \mathbf{b})\) when
\begin{equation}
    D_{g}(\R \parallel \S) \define \Big| Q_{g}(\mathcal{S}) - Q_{g}(\mathcal{\R}) \Big| \leq B_{g}
\end{equation}
Specifically,
\begin{equation}
  \pdv{}{b_{g}} v(\disp_\DP, \mathbf{D}, \pi, \S, \mathbf{b}) \leq
  (|\mathcal{G}| - 1)
  \Big| \beta^{+}_{g} - \beta^{-}_{g} \Big|
\end{equation}
for true positive rates \(\beta^{+}_{g}\) and false positive rates \(\beta^{-}_{g}\):
\begin{equation}
    \beta^{+}_{g}
\define \Pr_{\pi}( \hat{Y} {=} 1 {\mid} Y {=} 1, G {=} g ) \quad ; \quad \beta^{-}_{g} \define \Pr_{\pi}( \hat{Y} {=} 1 {\mid} Y {=} 0, G {=} g )
\end{equation} 
\begin{proof}
We first establish that \(D_g^{(\DP)}(\R \parallel \S) = |Q_{g}(\S) - Q_{g}(\R)|\), where
\begin{equation}
Q_{g}(\R) \define \Pr_{\mathcal{T}}(Y {=} 1 \mid G {=} g)
\end{equation}
is an appropriate measure of
group-conditioned distribution shift (\cref{def:vector-divergence}).
That \(\mathbf{D}\) satisfies
the axioms of a divergence on group-conditioned distributions subject to the label shift assumption  (\(\Pr_\R(X \mid Y, G) = \Pr_\S(X \mid Y, G)\)) and unchanging group sizes is easily verified:
\begin{align}
  \forall \S, \R, \quad D_g(\R \parallel \S) = |Q_{g}(\S) - Q_{g}(\R)| &\geq 0 \\
 D_g(\R\parallel \S) = |Q_{g}(\R) - Q_{g}(\R)| &= 0
 \end{align}
 and
 \begin{align}
 \forall g, \quad D_g(\R \parallel \S) = 0 &\implies \Pr_\R(Y\mid  G) = \Pr_\S(Y \mid G) \\
 & \implies \Pr_\R(Y, X \mid G) = \Pr_\S(Y, X \mid G)
\end{align}
We next show that
\((|\mathcal{G}| - 1)|\beta_{g}^{+} - \beta_{g}^{-}|\) is the corresponding Lipshitz bound for the slope of
\(v\) with respect to \(B_g\), where we recall
\begin{align}
  \forall g, \quad  \beta^{+}_{g} \define \Pr_{\pi, \R}\Big( \hat{Y} {=} 1 \mid Y {=} 1, G {=} g \Big)  \\
  \forall g, \quad  \beta^{-}_{g} \define \Pr_{\pi, \R}\Big( \hat{Y} {=} 1 \mid Y {=} {-}1, G {=} g \Big)
\end{align}
That is, we wish to show
\begin{equation}
  \pdv{}{b_{g}} v(\dispdp, \mathbf{D}, \pi, \S, \mathbf{b}) \leq (|\mathcal{G}| - 1)| \beta_{g}^{+} - \beta_{g}^{-} |
\end{equation}
This follows directly from recognition that \(\dispdp\) is locally always affine in the
acceptance rate for each group, with slope bounded by one less than the number of groups.
\begin{equation}
  \dispdp = \sum_{g,h \in \mathcal{G}} | \beta_{g} - \beta_{h} | \implies
  \pdv{}{\beta_{g}} \dispdp \leq |\mathcal{G}| - 1
\end{equation}
By the definition of conditional probability,
\begin{align}
  \beta_{g} &\define \Pr(\hat{Y} {=} 1) = \beta_{g}^{+} Q_{g} + \beta_{g}^{-}(1 - Q_{g}) \\
  \pdv{}{Q_{g}} \beta_{g} &= \beta_{g}^{+} - \beta_{g}^{-}
\end{align}
It follows by the chain rule that, for all \(\R\) mutated from \(\S\) subject to label shift,
\begin{equation}
  \pdv{}{Q_{g}(\R)} \dispdp(\pi, \R) \leq (|\mathcal{G}| - 1) |\beta_{g}^{+} - \beta_{g}^{-}|
\end{equation}
By the linearity of derivatives, for fixed \(\S\), this implies that for all \(\R\) attainable via label shift,
\begin{align}
    \pdv{}{|Q_{g}(\R) - Q_{g}(\S)|}
  \dispdp(\pi, \R)
    \leq
  (|\mathcal{G}| - 1) |\beta_{g}^{+} - \beta_{g}^{-}|
\end{align}
Since this equation holds for all \(\R\), it must also hold when evaluated at \(v\), the supremum of \(\disp\). It follows that
\begin{equation} \label{eqA:dp-label-target}
  \pdv{}{B_{g}} v(\dispdp, \mathbf{D}^{(\DP)}, \pi, \S, \mathbf{B}) \leq (|\mathcal{G}| - 1) |\beta_{g}^{+} - \beta_{g}^{-}|
\end{equation}

\end{proof}

\textbf{Proof of \cref{thm:dp-label}}:

\emph{Statement}:
For \DP under the bounded label-shift assumption $\forall g,  | Q_{g}(\S) - Q_{g}(\R) | \leq B_{g}$,
\begin{equation}
\begin{aligned}
 \dispdp(\pi, \R) \leq \dispdp(\pi, \S) + (|\mathcal{G}| - 1)
\sum_g B_{g} \ \Big|\beta^{+}_g - \beta^{-}_g \Big|
\end{aligned}
\end{equation}
\begin{proof}

This follows from the Lipshitz property implied by \cref{thm:yangsthm} (\cref{eqA:dp-label-target}) and
\cref{thm:lipshitz}.
\end{proof}

\subsection{Omitted details for \Cref{sec:strategic-response}}
\label{sec:proof-strategic-response}

\begin{lemma}
  \label{lemma:sr-w_g(x)}
    Recall the covariate shift reweighting coefficient \(\omega_{g}(x)\), defined in
    \cref{sec:covariateDP}.
    \begin{equation}
         \omega_g(x) \define \frac{\Pr_{\R}(X {=} x \mid G {=} g)}{\Pr_{\S}(X {=} x \mid  G {=} g)}
    \end{equation}
    For our assumed setting,
    \begin{align}\label{eq:strategic-omega}
   \omega_g(x) = 
    \begin{cases}
    1, & x\in [0, \tau_g - {m_{g}}) \\
    \frac{\tau_g - x}{m_{g}}, & x\in [\tau_g - {m_{g}} ,\tau_g)\\
    \frac{1}{m_{g}}(- x + \tau_g + 2m_{g}), & x\in [\tau_g, \tau_g + m_{g})\\
    1, & x \in [\tau_g + m_{g}, 1]
    \end{cases}
\end{align}
\end{lemma}

Proof for \Cref{lemma:sr-w_g(x)}:
\begin{proof}
We discuss the target distribution by cases:
\begin{itemize}
    \item For the target distribution between $[0, \tau_g - M_g]$: since we assume the agents are rational, under assumption \ref{assumption:cost-function}, agents with feature that is smaller than $[0, \tau_g - M_g]$ will not perform any kinds of adaptations, and no other agents will adapt their features to this range of features either, so the distribution between $[0, \tau_g - M_g]$ will remain the same as before. 
    \item For target distribution between $[\tau_g - M_g, \tau_g]$, it can be directly calculated from assumption \ref{assumption:nondeterministic-feature-update}.
    \item For distribution between $[\tau_g, \tau_g + M_g]$, 
consider a particular feature $x^\star\in [\tau_g, \tau_g + M_g]$, under \Cref{assumption:new-feature-uniform-distribution}, we know its new distribution becomes:
\begin{align*}
    \Pr_\R (x = x^\star) &= 1 + \int_{x^\star - M_g}^{\tau_g} \frac{1 - \frac{\tau_g - z}{M_g}}{M_g - \tau_g + z} dz\\
    &= 1 +  \int_{x^\star - M_g}^{\tau_g} \frac{1}{M_g}d z\\
    & = \frac{1}{M_g} (-x^\star + \tau_g + 2M_g)
\end{align*}
\item For the target distribution between $[\tau_g + M_g, 1]$: under assumption \ref{assumption:cost-function} and \ref{assumption:new-feature-uniform-distribution}, we know that no agents will change their feature to this feature region. So the distribution between $[\tau_g + M_g, 1]$ remains the same as the source distribution.
\end{itemize}

Thus, the new feature distribution of $x^{(M_g)}_{\tau_g}$ after agents from group $g$ strategic responding becomes:
\begin{align}
   \Pr_\R(x) = \Pr(x^{(M_g)}_{\tau_g}) = 
    \begin{cases}
    1, \ \ & x\in [0, \tau_g - {M_g}) ~\text{and}~  x\in [\tau_g + M_g, 1]\\
    \frac{\tau_g - x}{M_g}, & x\in [\tau_g - {M_g} ,\tau_g)\\
    \frac{1}{M_g}(- x + \tau_g + 2M_g), & x\in [\tau_g, \tau_g + M_g)\\
     0,  & \text{otherwise}
    \end{cases}
\end{align}
\end{proof}

\textbf{Proof of \Cref{proposition:bound-strategic-response}}:

\emph{Statement}:
  For our assumed setting of strategic response involving \DP for two groups \(\{g, h\}\),
\cref{thm:dp-covariate} implies
\vspace{-0.1in}
\begin{equation}
    \dispdp(\pi, \R) 
    \leq \dispdp(\pi, \S) + \tau_g(1-\tau_g)\frac{2}{3}m_{g}  +
    \tau_h(1 - \tau_h)\frac{2}{3}m_{h}
\end{equation}

\begin{proof}
According to \Cref{lemma:sr-w_g(x)}, we can compute the variance of $w_g(x)$: $\var(w_g(x)) = \E\Big[\big(w_g(x)- \E[w_g(x)]\big)^2\Big] = \frac{2}{3}M_g$. Then by plugging it to the general bound for \Cref{thm:dp-covariate} gives us the result.
\end{proof}

\textbf{Proof of \cref{thm:label-replicator}}:

\emph{Statement}:
  For \DP subject to label replicator dynamics,
  \begin{equation}
  \begin{aligned}
\dispdp(\pi, \R) \leq \dispdp(\pi, \S) + \sum_g \Big|Q_{g}[t + 1] - Q_{g}[t]\Big| \frac{| \rho^{\TP}_{g} - \rho^{\FP}_{g} |}{\rho^{\TP}_{g} + \rho^{\FP}_{g}}
  \end{aligned}
\end{equation}

\begin{proof}
  We may directly substitute
  \begin{align*}
    &|\mathcal{G}| = 2 \\
    &B_{g} = \Big|Q_{g}[t + 1] - Q_{g}[t]\Big| \\
 &\Big|\beta^{+}_g - \beta^{-}_g \Big|
  =
\frac{| \rho^{\TP}_{g} - \rho^{\FP}_{g} |}{\rho^{\TP}_{g} + \rho^{\FP}_{g}}
  \end{align*}
  into \cref{thm:dp-label}.
\end{proof}

\textbf{Proof of \cref{thm:geometric-beta}}:

\emph{Statement}:
The true positive rate \(\beta^{+}_{g}\) is bounded over the domain of covariate
shift \(\DD_{\text{cov}}[\mathbf{B}]\), which we define by the bound
\(\mathbf{D}(\R \parallel \S) \preceq \mathbf{B}\), and the invariance of
\(\Pr(Y {=} 1 \mid X {=} x, G {=} g)\) for all \(x, g\), as
\begin{equation}
\frac{\cos(\phi_g^u)}{\cos(\xi_g - \phi_g^u)} \leq \beta^{+}_{g}(\pi, \R) \leq \frac{\cos (\phi_g^l)}{\cos(\xi_g - \phi_g^l)}
\end{equation}
where
\begin{equation}
\phi_g^l \define \min_{\D \in \DD_\text{cov}[\mathbf{B}]} \phi_g[\D]; \quad
\phi_g^u \define \max_{\D \in \DD_\text{cov}[\mathbf{B}]} \phi_g[\D]
\end{equation}

\begin{proof}
To be rigorous, we may give
an explicit expression for \(\mathsf{r}_{g}^{\perp}\) by implicitly forming a basis in the
\((\mathsf{1}, \mathsf{t}_{g})\)-plane via the Gram-Schmidt process.
\begin{align}
  \mathsf{r}_{g}^{\perp}
  &\define
    \langle \mathsf{r}_{g}, \mathsf{t}_{g} \rangle_{g}
    \frac{\mathsf{t}_{g}}{\|\mathsf{t}_{g}\|^{2}}
    +
    \langle \mathsf{r}_{g}, \mathsf{u}_{g} \rangle_{g}
    \frac{\mathsf{u}_{g}}{\|\mathsf{u}_{g}\|^{2}} \\
  \mathsf{u}_{g}
  &\define
    \mathsf{1} -
 \langle \mathsf{1}, \mathsf{t}_{g} \rangle \frac{\mathsf{t}_{g}}{\|\mathsf{t}_{g}\|^{2}} \\
\end{align}
From which we may verify that
\begin{align}
 \langle \mathsf{u}_{g}, \mathsf{t}_{g} \rangle  &= 0 \\
  \langle \mathsf{r}_{g}^{\perp}, \mathsf{t}_{g} \rangle_{g}
  &=\langle \mathsf{r}_{g}, \mathsf{t}_{g} \rangle_{g} \\
  \langle \mathsf{r}_{g}^{\perp}, \mathsf{u}_{g} \rangle_{g}
  &=\langle \mathsf{r}_{g}, \mathsf{u}_{g} \rangle_{g} \\
  \langle \mathsf{r}_{g}^{\perp}, \mathsf{1} \rangle_{g} &=
\langle \mathsf{r}_{g}, \mathsf{1} \rangle_{g}
\end{align}
Recalling the relationship between the cosine of an angle between two vectors and
inner products:
\begin{align}
\cos(\angle(a, b)) = \frac{\langle a, b \rangle}{\|a\| \|b\|}
\end{align}
It follows from \cref{eq:inner-product} that, defining \(\xi_{g} \define \angle(\mathsf{t}_{g}, \mathsf{1})\),
\begin{equation}\label{eq:monotonic92}
  \beta^{+}_{g}
\frac{\|\mathsf{1}\|}{\|\mathsf{t}_g\|}
    =
  \frac{\cos(\angle(\mathsf{r}_{g}, \mathsf{t}_{g}))}
       {\cos(\angle(\mathsf{r}_{g}, \mathsf{1}))}
=
  \frac{\cos(\angle(\mathsf{r}_{g}^{\perp}, \mathsf{t}_{g}))}
  {\cos(\angle(\mathsf{r}_{g}^{\perp}, \mathsf{1}))}
  =
    \frac{\cos(\phi_{g})}{\cos(\xi_{g} - \phi_{g})}
\end{equation}
By the monotonicity of the final expression above with respect to \(\phi_{g}\), for fixed \(\xi_{g}\):
\begin{equation}\label{eq:monotonic93}
\dv{}{x} \frac{\cos(x)}{\cos(\xi - x)} = -\frac{\sin(x)\cos(\xi - x) + \cos(x)\sin(\xi - x)}{\cos^{2}(\xi - x)} = -\frac{\sin(\xi)}{\cos^{2}(\xi - x)}
\end{equation}
We note that \cref{eq:monotonic93} is strictly negative, thus the expression in
\cref{eq:monotonic92} must be monotonic for fixed \(\xi\). We may conclude that
\(\beta^{+}_{g}\) is extremized with extremal values of \(\phi_{g}\), denoted as
\(\phi_{g}^{u}\) and \(\phi_{g}^{l}\).
\end{proof}

%% file: main.bbl
\begin{thebibliography}{43}
\providecommand{\natexlab}[1]{#1}
\providecommand{\url}[1]{\texttt{#1}}
\expandafter\ifx\csname urlstyle\endcsname\relax
  \providecommand{\doi}[1]{doi: #1}\else
  \providecommand{\doi}{doi: \begingroup \urlstyle{rm}\Url}\fi

\bibitem[An et~al.(2022)An, Che, Ding, and Huang]{an2022transferring}
Bang An, Zora Che, Mucong Ding, and Furong Huang.
\newblock Transferring fairness under distribution shifts via fair consistency
  regularization.
\newblock \emph{arXiv preprint arXiv:2206.12796}, 2022.

\bibitem[Ben-David et~al.(2010)Ben-David, Blitzer, Crammer, Kulesza, Pereira,
  and Vaughan]{ben-david2010domain}
Shai Ben-David, John Blitzer, Koby Crammer, Alex Kulesza, Fernando Pereira, and
  Jennifer Vaughan.
\newblock A theory of learning from different domains.
\newblock \emph{Machine Learning}, 79:\penalty0 151--175, 2010.

\bibitem[Biswas and Mukherjee(2021)]{biswas2021ensuring}
Arpita Biswas and Suvam Mukherjee.
\newblock Ensuring fairness under prior probability shifts.
\newblock In \emph{Proceedings of the 2021 AAAI/ACM Conference on AI, Ethics,
  and Society}, pages 414--424, 2021.

\bibitem[Blake(1998)]{Blake1998UCIRO}
Catherine Blake.
\newblock Uci repository of machine learning databases.
\newblock 1998.

\bibitem[Chen et~al.(2021)Chen, Wang, and Liu]{chen2021linear}
Yatong Chen, Jialu Wang, and Yang Liu.
\newblock Linear classifiers that encourage constructive adaptation.
\newblock In \emph{Algorithmic Recourse workshop at ICML'21}, 2021.

\bibitem[Chouldechova(2017)]{chouldechova2017fair}
Alexandra Chouldechova.
\newblock {Fair prediction with disparate impact: A study of bias in recidivism
  prediction instruments}.
\newblock \emph{Big data}, 5\penalty0 (2):\penalty0 153--163, 2017.

\bibitem[Coate and Loury(1993)]{coate1993will}
Stephen Coate and Glenn~C Loury.
\newblock {Will affirmative-action policies eliminate negative stereotypes?}
\newblock \emph{The American Economic Review}, pages 1220--1240, 1993.

\bibitem[Corbett-Davies et~al.(2017)Corbett-Davies, Pierson, Feller, Goel, and
  Huq]{corbett2017algorithmic}
Sam Corbett-Davies, Emma Pierson, Avi Feller, Sharad Goel, and Aziz Huq.
\newblock {Algorithmic decision making and the cost of fairness}.
\newblock In \emph{Proceedings of the 23rd acm sigkdd international conference
  on knowledge discovery and data mining}, pages 797--806, 2017.

\bibitem[Coston et~al.(2019)Coston, Ramamurthy, Wei, Varshney, Speakman,
  Mustahsan, and Chakraborty]{Coston2019fair}
Amanda Coston, Karthikeyan~Natesan Ramamurthy, Dennis Wei, Kush~R. Varshney,
  Skyler Speakman, Zairah Mustahsan, and Supriyo Chakraborty.
\newblock Fair transfer learning with missing protected attributes.
\newblock In \emph{Proceedings of the 2019 AAAI/ACM Conference on AI, Ethics,
  and Society}, AIES '19, page 91–98, New York, NY, USA, 2019. Association
  for Computing Machinery.
\newblock ISBN 9781450363242.
\newblock \doi{10.1145/3306618.3314236}.
\newblock URL \url{https://doi.org/10.1145/3306618.3314236}.

\bibitem[Cotter et~al.(2019)Cotter, Gupta, and
  Narasimhan]{cotter2019stochastic}
Andrew Cotter, Maya Gupta, and Harikrishna Narasimhan.
\newblock On making stochastic classifiers deterministic.
\newblock In \emph{Advances in Neural Information Processing Systems},
  volume~32. Curran Associates, Inc., 2019.

\bibitem[Creager et~al.(2020)Creager, Madras, Pitassi, and
  Zemel]{creager2020causal}
Elliot Creager, David Madras, Toniann Pitassi, and Richard Zemel.
\newblock Causal modeling for fairness in dynamical systems.
\newblock In \emph{Proceedings of the 37th International Conference on Machine
  Learning}, ICML'20. JMLR.org, 2020.

\bibitem[D'Amour et~al.(2020)D'Amour, Srinivasan, Atwood, Baljekar, Sculley,
  and Halpern]{d2020fairness}
Alexander D'Amour, Hansa Srinivasan, James Atwood, Pallavi Baljekar, D~Sculley,
  and Yoni Halpern.
\newblock {Fairness is not static: deeper understanding of long term fairness
  via simulation studies}.
\newblock In \emph{Proceedings of the 2020 Conference on Fairness,
  Accountability, and Transparency}, pages 525--534, 2020.

\bibitem[Ding et~al.(2021)Ding, Hardt, Miller, and Schmidt]{ding2021retiring}
Frances Ding, Moritz Hardt, John Miller, and Ludwig Schmidt.
\newblock Retiring adult: New datasets for fair machine learning.
\newblock In \emph{Thirty-Fifth Conference on Neural Information Processing
  Systems}, 2021.
\newblock URL \url{https://openreview.net/forum?id=bYi_2708mKK}.

\bibitem[Dwork et~al.(2012)Dwork, Hardt, Pitassi, Reingold, and
  Zemel]{dwork2012fairness}
Cynthia Dwork, Moritz Hardt, Toniann Pitassi, Omer Reingold, and Richard Zemel.
\newblock {Fairness through awareness}.
\newblock In \emph{Proceedings of the 3rd innovations in theoretical computer
  science conference}, pages 214--226, 2012.

\bibitem[Feldman et~al.(2015)Feldman, Friedler, Moeller, Scheidegger, and
  Venkatasubramanian]{feldman2015certifying}
Michael Feldman, Sorelle~A Friedler, John Moeller, Carlos Scheidegger, and
  Suresh Venkatasubramanian.
\newblock {Certifying and removing disparate impact}.
\newblock In \emph{proceedings of the 21th ACM SIGKDD international conference
  on knowledge discovery and data mining}, pages 259--268, 2015.

\bibitem[Flood et~al.(2020)Flood, King, Ruggles, and Warren]{IPUMS}
Sarah Flood, Miriam King, Renae Rodgers~Steven Ruggles, and J.~Robert Warren.
\newblock Integrated public use microdata series, current population survey:
  Version 8.0 [dataset], 2020.
\newblock URL \url{https://www.ipums.org/projects/ipums-cps/d030.v8.0}.

\bibitem[Grgić-Hlača et~al.(2017)Grgić-Hlača, Zafar, Gummadi, and
  Weller]{grgichlaca2017fairness}
Nina Grgić-Hlača, Muhammad~Bilal Zafar, Krishna~P. Gummadi, and Adrian
  Weller.
\newblock On fairness, diversity and randomness in algorithmic decision making,
  2017.

\bibitem[Hardt et~al.(2016{\natexlab{a}})Hardt, Megiddo, Papadimitriou, and
  Wootters]{hardt2016strategic}
Moritz Hardt, Nimrod Megiddo, Christos Papadimitriou, and Mary Wootters.
\newblock Strategic classification.
\newblock In \emph{Proceedings of the 2016 ACM Conference on Innovations in
  Theoretical Computer Science}, page 111–122, New York, NY, USA,
  2016{\natexlab{a}}. Association for Computing Machinery.

\bibitem[Hardt et~al.(2016{\natexlab{b}})Hardt, Price, and
  Srebro]{hardt2016equality}
Moritz Hardt, Eric Price, and Nati Srebro.
\newblock {Equality of opportunity in supervised learning}.
\newblock In \emph{Advances in neural information processing systems}, pages
  3315--3323, 2016{\natexlab{b}}.

\bibitem[Hu and Chen(2018)]{hu2018short}
Lily Hu and Yiling Chen.
\newblock {A short-term intervention for long-term fairness in the labor
  market}.
\newblock In \emph{Proceedings of the 2018 World Wide Web Conference on World
  Wide Web}, pages 1389--1398. International World Wide Web Conferences
  Steering Committee, 2018.

\bibitem[Kang et~al.(2022)Kang, Li, Weber, Liu, Zhang, and
  Li]{kang2022certifying}
Mintong Kang, Linyi Li, Maurice Weber, Yang Liu, Ce~Zhang, and Bo~Li.
\newblock Certifying some distributional fairness with subpopulation
  decomposition.
\newblock \emph{arXiv preprint arXiv:2205.15494}, 2022.

\bibitem[Kulinski et~al.(2020)Kulinski, Bagchi, and
  Inouye]{NEURIPS2020_e2d52448}
Sean Kulinski, Saurabh Bagchi, and David~I Inouye.
\newblock Feature shift detection: Localizing which features have shifted via
  conditional distribution tests.
\newblock In H.~Larochelle, M.~Ranzato, R.~Hadsell, M.F. Balcan, and H.~Lin,
  editors, \emph{Advances in Neural Information Processing Systems}, volume~33,
  pages 19523--19533. Curran Associates, Inc., 2020.
\newblock URL
  \url{https://proceedings.neurips.cc/paper/2020/file/e2d52448d36918c575fa79d88647ba66-Paper.pdf}.

\bibitem[Liu et~al.(2018)Liu, Dean, Rolf, Simchowitz, and
  Hardt]{liu2018delayed}
Lydia~T Liu, Sarah Dean, Esther Rolf, Max Simchowitz, and Moritz Hardt.
\newblock Delayed impact of fair machine learning.
\newblock In \emph{International Conference on Machine Learning}, pages
  3150--3158. PMLR, 2018.

\bibitem[Liu et~al.(2020)Liu, Wilson, Haghtalab, Kalai, Borgs, and
  Chayes]{liu2019disparate}
Lydia~T Liu, Ashia Wilson, Nika Haghtalab, Adam~Tauman Kalai, Christian Borgs,
  and Jennifer Chayes.
\newblock {The disparate equilibria of algorithmic decision making when
  individuals invest rationally}.
\newblock In \emph{Proceedings of the 2020 Conference on Fairness,
  Accountability, and Transparency}, pages 381--391, 2020.

\bibitem[Liu et~al.(2021)Liu, Chen, Tang, and Zhang]{liu2021model}
Yang Liu, Yatong Chen, Zeyu Tang, and Kun Zhang.
\newblock Model transferability with responsive decision subjects, 2021.

\bibitem[Mandal et~al.(2020)Mandal, Deng, Jana, Wing, and
  Hsu]{mandal2020ensuring}
Debmalya Mandal, Samuel Deng, Suman Jana, Jeannette Wing, and Daniel~J Hsu.
\newblock Ensuring fairness beyond the training data.
\newblock In H.~Larochelle, M.~Ranzato, R.~Hadsell, M.F. Balcan, and H.~Lin,
  editors, \emph{Advances in Neural Information Processing Systems}, volume~33,
  pages 18445--18456. Curran Associates, Inc., 2020.

\bibitem[Mansour et~al.(2009)Mansour, Mohri, and
  Rostamizadeh]{mansour2009domain}
Yishay Mansour, Mehryar Mohri, and Afshin Rostamizadeh.
\newblock Domain adaptation: Learning bounds and algorithms, 2009.

\bibitem[Mouzannar et~al.(2019)Mouzannar, Ohannessian, and
  Srebro]{mouzannar2019fair}
Hussein Mouzannar, Mesrob~I Ohannessian, and Nathan Srebro.
\newblock {From fair decision making to social equality}.
\newblock In \emph{Proceedings of the Conference on Fairness, Accountability,
  and Transparency}, pages 359--368. ACM, 2019.

\bibitem[Narasimhan(2018)]{narasimhan2018complex}
Harikrishna Narasimhan.
\newblock Learning with complex loss functions and constraints.
\newblock In \emph{Proceedings of the Twenty-First International Conference on
  Artificial Intelligence and Statistics}, volume~84 of \emph{Proceedings of
  Machine Learning Research}, pages 1646--1654. PMLR, 09--11 Apr 2018.

\bibitem[Raab and Liu(2021)]{raab2021unintended}
Reilly Raab and Yang Liu.
\newblock Unintended selection: Persistent qualification rate disparities and
  interventions.
\newblock \emph{Advances in Neural Information Processing Systems}, 34, 2021.

\bibitem[Rezaei et~al.(2021)Rezaei, Liu, Memarrast, and
  Ziebart]{rezaei2021robust}
Ashkan Rezaei, Anqi Liu, Omid Memarrast, and Brian~D. Ziebart.
\newblock Robust fairness under covariate shift.
\newblock In \emph{AAAI}, 2021.

\bibitem[Roh et~al.(2021)Roh, Lee, Whang, and Suh]{Roh2021sample}
Yuji Roh, Kangwook Lee, Steven Whang, and Changho Suh.
\newblock Sample selection for fair and robust training.
\newblock In M.~Ranzato, A.~Beygelzimer, Y.~Dauphin, P.S. Liang, and J.~Wortman
  Vaughan, editors, \emph{Advances in Neural Information Processing Systems},
  volume~34, pages 815--827. Curran Associates, Inc., 2021.

\bibitem[Schrouff et~al.(2022)Schrouff, Harris, Koyejo, Alabdulmohsin,
  Schnider, Opsahl-Ong, Brown, Roy, Mincu, Chen,
  et~al.]{schrouff2022maintaining}
Jessica Schrouff, Natalie Harris, Oluwasanmi Koyejo, Ibrahim Alabdulmohsin, Eva
  Schnider, Krista Opsahl-Ong, Alex Brown, Subhrajit Roy, Diana Mincu,
  Christina Chen, et~al.
\newblock Maintaining fairness across distribution shift: do we have viable
  solutions for real-world applications?
\newblock \emph{arXiv preprint arXiv:2202.01034}, 2022.

\bibitem[Schumann et~al.(2019)Schumann, Wang, Beutel, Chen, Qian, and
  Chi]{schumann2019transfer}
Candice Schumann, Xuezhi Wang, Alex Beutel, Jilin Chen, Hai Qian, and Ed~H.
  Chi.
\newblock Transfer of machine learning fairness across domains, 2019.

\bibitem[Shimodaira(2000)]{shimodaira2000improving}
Hidetoshi Shimodaira.
\newblock Improving predictive inference under covariate shift by weighting the
  log-likelihood function.
\newblock \emph{Journal of statistical planning and inference}, 90\penalty0
  (2):\penalty0 227--244, 2000.

\bibitem[Singh et~al.(2021)Singh, Singh, Mhasawade, and
  Chunara]{Singh2021fairness}
Harvineet Singh, Rina Singh, Vishwali Mhasawade, and Rumi Chunara.
\newblock Fairness violations and mitigation under covariate shift.
\newblock In \emph{Proceedings of the 2021 ACM Conference on Fairness,
  Accountability, and Transparency}, FAccT '21, New York, NY, USA, 2021.
  Association for Computing Machinery.

\bibitem[Sugiyama et~al.(2008)Sugiyama, Suzuki, Nakajima, Kashima, von
  B{\"u}nau, and Kawanabe]{sugiyama2008direct}
Masashi Sugiyama, Taiji Suzuki, Shinichi Nakajima, Hisashi Kashima, Paul von
  B{\"u}nau, and Motoaki Kawanabe.
\newblock Direct importance estimation for covariate shift adaptation.
\newblock \emph{Annals of the Institute of Statistical Mathematics},
  60\penalty0 (4):\penalty0 699--746, 2008.

\bibitem[Ustun et~al.(2019)Ustun, Spangher, and Liu]{ustun2019actionable}
Berk Ustun, Alexander Spangher, and Yang Liu.
\newblock Actionable recourse in linear classification.
\newblock In \emph{Proceedings of the Conference on Fairness, Accountability,
  and Transparency}, pages 10--19, 2019.

\bibitem[Wen et~al.(2019)Wen, Bastani, and Topcu]{wen2019fairness}
Min Wen, Osbert Bastani, and Ufuk Topcu.
\newblock {Fairness with Dynamics}.
\newblock \emph{arXiv preprint arXiv:1901.08568}, 2019.

\bibitem[Wu et~al.(2022)Wu, Chen, and Liu]{wu2022metric}
Jimmy Wu, Yatong Chen, and Yang Liu.
\newblock Metric-fair classifier derandomization.
\newblock In \emph{Proceedings of the 39th International Conference on Machine
  Learning}, volume 162 of \emph{Proceedings of Machine Learning Research}.
  PMLR, 17--23 Jul 2022.

\bibitem[Zemel et~al.(2013)Zemel, Wu, Swersky, Pitassi, and
  Dwork]{zemel2013learning}
Rich Zemel, Yu~Wu, Kevin Swersky, Toni Pitassi, and Cynthia Dwork.
\newblock {Learning fair representations}.
\newblock In \emph{International conference on machine learning}, pages
  325--333. PMLR, 2013.

\bibitem[Zhang et~al.(2013)Zhang, Sch{\"o}lkopf, Muandet, and
  Wang]{zhang2013domain}
Kun Zhang, Bernhard Sch{\"o}lkopf, Krikamol Muandet, and Zhikun Wang.
\newblock Domain adaptation under target and conditional shift.
\newblock In \emph{International Conference on Machine Learning}, pages
  819--827. PMLR, 2013.

\bibitem[Zhang et~al.(2020)Zhang, Tu, Liu, Liu, Kjellström, Zhang, and
  Zhang]{zhang2020fair}
Xueru Zhang, Ruibo Tu, Yang Liu, Mingyan Liu, Hedvig Kjellström, Kun Zhang,
  and Cheng Zhang.
\newblock How do fair decisions fare in long-term qualification?
\newblock In \emph{NeurIPS}, 2020.

\end{thebibliography}
